\documentclass[11pt]{article}

\usepackage[preprint]{acl}
\usepackage{times}
\usepackage{latexsym}
\usepackage{amsmath,amssymb,amsfonts}
\usepackage{bbm}
\usepackage{algorithm}
\usepackage{algpseudocode}
\usepackage{amsmath,amssymb,amsthm}
\usepackage{graphicx}
\usepackage{subcaption}
\newtheorem{theorem}{Theorem}
\newtheorem{lemma}{Lemma}
\DeclareMathOperator*{\argmax}{arg\,max}

\newtheorem{remark}{Remark}
\usepackage{times}
\usepackage{latexsym}
\usepackage{booktabs}
\usepackage{multirow}
\usepackage{makecell}
\usepackage{tabularx}
\usepackage{threeparttable}
\usepackage{adjustbox}
\usepackage{booktabs}
\usepackage[table]{xcolor}
\usepackage{booktabs}
\usepackage{tabularx}
\newcolumntype{C}[1]{>{\centering\arraybackslash}p{#1}}
\newcolumntype{C}[1]{>{\centering\arraybackslash}p{#1}}

\usepackage{xcolor}
\newcolumntype{C}[1]{>{\centering\arraybackslash}p{#1}}
\newcommand{\cmark}{\textcolor{teal!80!green}{\ding{51}}}

\usepackage{pifont}
\usepackage{array}
\usepackage{mathtools}

\newcommand{\MethodTag}[2]{%
  {\setlength{\fboxsep}{2pt}\colorbox{#1}{\strut #2}}%
}
\newcommand{\Only}{\MethodTag{gray!12}{LLM-only}}
\newcommand{\RAG}{\MethodTag{blue!10}{LLM-RAG}}
\newcommand{\UCB}{\MethodTag{orange!15}{UCB-Cascaded}}
\newcommand{\Balance}{\MethodTag{red!12}{BalanceRAG}}
\usepackage{microtype}
\usepackage{amsmath}
\usepackage{amssymb}
\usepackage{amsthm}
\usepackage{algorithm}
\usepackage{bm}
\usepackage[most]{tcolorbox}
\usepackage{xcolor}
\usepackage{graphicx}

\usepackage[most]{tcolorbox}
\usepackage{xcolor}
\usepackage{graphicx}
\usepackage{pifont}

\definecolor{ExampleBorder}{RGB}{80,80,80}
\definecolor{OnlyGreen}{RGB}{226,245,232}
\definecolor{RAGBlue}{RGB}{225,238,255}
\definecolor{BothOrange}{RGB}{255,240,220}
\definecolor{CheckGreen}{RGB}{30,140,70}
\usepackage{enumitem}
\renewcommand{\cmark}{\textcolor{CheckGreen}{\ding{51}}}

\tcbset{
  qualitativebox/.style={
    enhanced,
    width=\linewidth,
    colback=white,
    colframe=ExampleBorder,
    boxrule=0.8pt,
    arc=2pt,
    left=7pt,
    right=7pt,
    top=5pt,
    bottom=5pt,
    before skip=5pt,
    after skip=5pt,
  },
  onlyhelpbox/.style={
    qualitativebox,
    colback=OnlyGreen!35,
  },
  raghelpbox/.style={
    qualitativebox,
    colback=RAGBlue!35,
  },
  bothcorrectbox/.style={
    qualitativebox,
    colback=BothOrange!35,
  }
}

\newcommand{\exampletitle}[1]{%
  \vspace{-2pt}
  \noindent{\small\textbf{#1}}
  \vspace{2pt}
}

\newcommand{\examplesep}{%
  \vspace{2pt}
  \noindent\textcolor{gray}{\dotfill}
  \vspace{2pt}
}
\usepackage{booktabs}
\usepackage{multirow}
\usepackage{graphicx}
\usepackage[table]{xcolor}
\usepackage{colortbl}
\definecolor{HeaderGray}{HTML}{F2F2F2}
\definecolor{GroupBlue}{HTML}{EAF2FB}
\definecolor{GainBlue}{HTML}{DCEBFA}
\providecommand{\Only}{\textsc{LLM-only}}
\providecommand{\RAG}{\textsc{LLM-RAG}}
\providecommand{\UCB}{\textsc{UCB-Cascaded}}
\providecommand{\Balance}{\textsc{BalanceRAG}}

\definecolor{BalanceShade}{HTML}{F3F6FB}
\newcommand{\balcell}[1]{\cellcolor{BalanceShade}{#1}}
\usepackage[T1]{fontenc}

\usepackage[utf8]{inputenc}

\usepackage{microtype}

\usepackage{inconsolata}

\usepackage{graphicx}

\title{BalanceRAG: Joint Risk Calibration for Cascaded Retrieval-Augmented Generation}

\author{
 \textbf{Zijun Jia\textsuperscript{1}},
 \textbf{Yuanchang Ye\textsuperscript{3}},
 \textbf{Sen Jia\textsuperscript{2}},
 \textbf{Yiyao Qian\textsuperscript{4}},
 \textbf{Haoning Wang\textsuperscript{1}},\\
 \textbf{Baojie Chen\textsuperscript{1}},
 \textbf{Diyin Tang\textsuperscript{1}},
 \textbf{Jinsong Yu\textsuperscript{1,*}},
 \textbf{Zhiyuan Wang\textsuperscript{4,*}}
\\
\\
 \textsuperscript{1}Beihang University,\\
  \textsuperscript{2}Shenzhen Institute of Advanced Technology,
\\
 \textsuperscript{3}Zhejiang University of Finance $\&$ Economics,\\
 \textsuperscript{4}University of Electronic Science and Technology of China\\
 \small{
   \textbf{Correspondence:} 
   \texttt{yujs@buaa.edu.cn, yhzywang@gmail.com}
 }
}

\begin{document}
\maketitle

\begin{abstract}
Large language models (LLMs) can enhance factuality via retrieval-augmented generation (RAG), but applying RAG to every query is unnecessary when the model-only answer is reliable. 
This motivates cascaded RAG: each query is first handled by an LLM-only branch, escalated to a RAG fallback only if the primary branch is uncertain, and abstained from when neither branch is sufficiently trustworthy. 
However, calibrating such cascades stage by stage may be conservative, since the final utility depends on joint uncertainty thresholding of LLM-only and RAG. 
In this work, we develop \textbf{BalanceRAG} to certify threshold pairs at a target risk level. 
Given uncertainty scores from the two branches, BalanceRAG frames each threshold pair as an operating point on a two-dimensional lattice and identifies safe operating points using sequential graphical testing. 
This enables risk-adaptive threshold calibration, controlling the system-level error rate among accepted points, while retaining more examples. 
Furthermore, BalanceRAG extends to multi-risk calibration, allowing retrieval usage to be bounded together with the selection-conditioned risk. 
Experiments on three open-domain question answering (QA) benchmarks across multiple LLM backbones demonstrate that BalanceRAG meets prescribed risk levels, preserves higher coverage and more accepted correct examples, and reduces unnecessary retrieval calls compared with always-on RAG.
\end{abstract}

\section{Introduction}
\label{sec:Introduction}

Large language models (LLMs) have been widely deployed into real-world open-ended question answering (QA) scenarios~\citep{pandit2025medhallu,duan-etal-2025-guidellm}, but they still often produce fluent yet factually incorrect responses~\citep{wang2025word,pmlr-v235-huang24x}. 
Retrieval-augmented generation (RAG) can mitigate this issue by grounding model output in external evidence, improving factuality on knowledge-intensive tasks~\citep{lewis2020retrieval,karpukhin2020dense,asai2023self,jiang2023active}. 
However, always-on retrieval can be unnecessary and costly, since retrieval and long-context generation increase latency, memory use, and token consumption relative to direct LLM inference~\citep{wang2024searching,li2024retrieval}. 

\begin{figure}
\vspace{2mm}
    \centering
    \includegraphics[width=1\linewidth]{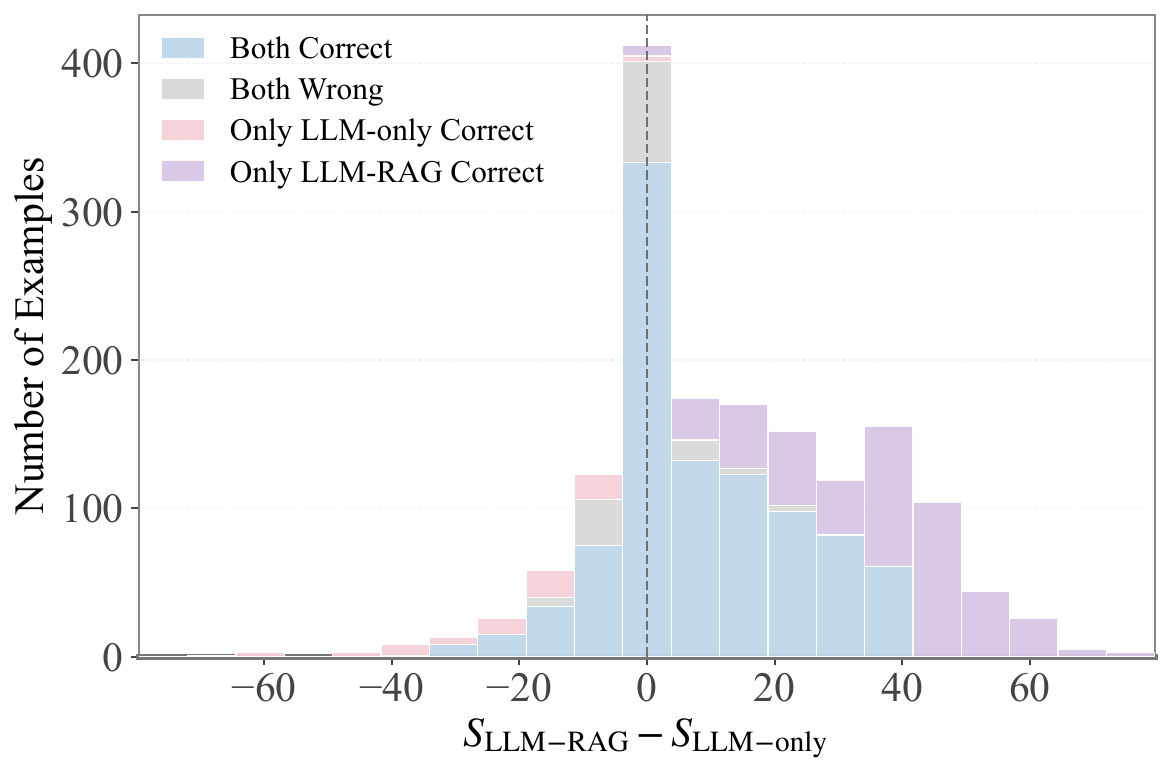}
    \vspace{-8mm}
    \caption{Distribution of the per-example score differences between RAG and LLM-only. \(S_{\mathrm{LLM\text{-}RAG}}\) and \(S_{\mathrm{LLM\text{-}only}}\) are the similarity scores between each path's prediction and the ground-truth answer. The x-axis reports \(S_{\mathrm{LLM\text{-}RAG}} - S_{\mathrm{LLM\text{-}only}}\), with positive values favoring RAG and negative values favoring LLM-only, while the y-axis reports the number of examples. Colors distinguish whether both branches are correct, both are wrong, or only one branch is correct.}
    \label{fig:intro}
    \vspace{-5mm}
\end{figure}

\begin{figure*}[t]
    \centering
    \includegraphics[width=1\linewidth]{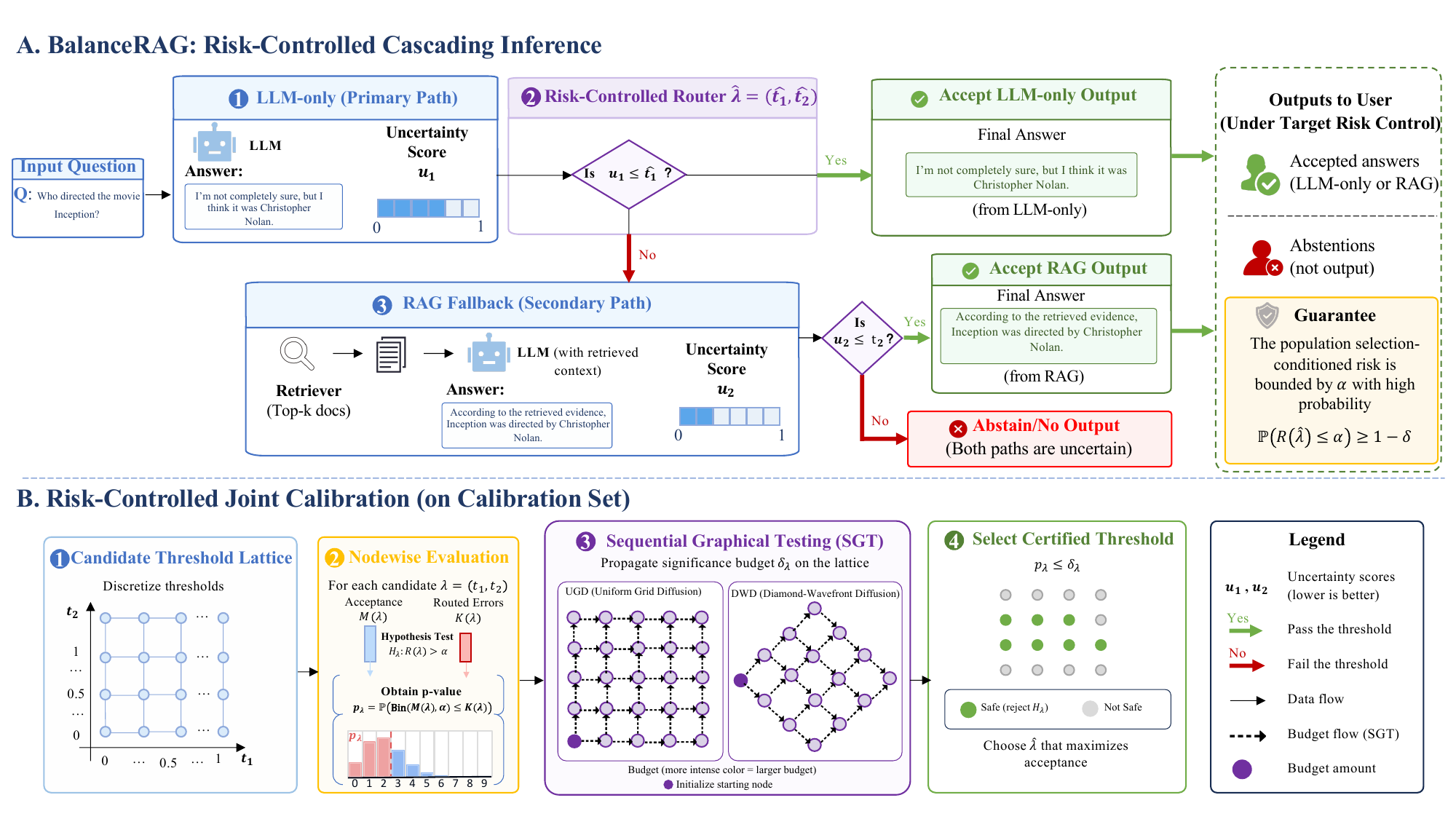}
    \vspace{-8mm}
    \caption{Overview of BalanceRAG: risk-controlled cascading inference (A) and joint threshold calibration via Sequential Graphical Testing (B). $\alpha$ is the user-specified risk level (upper bound).}
    \label{fig:overview}
    \vspace{-5mm}
\end{figure*}



Figure~\ref{fig:intro} shows that LLM-only and RAG have similar correctness scores for many queries, suggesting that retrieval mainly adds overhead when the model-only answer is already reliable. 
At the same time, a non-negligible fraction of examples are answered correctly only with retrieval. 
Thus, retrieval is neither uniformly necessary nor safely removable~\citep{mallen2023not}. 
A desirable system should keep easy examples on the LLM-only path, escalate uncertain queries to RAG, and abstain when neither path is reliable enough. 
Prior adaptive RAG frameworks follow this principle by routing queries with uncertainty or complexity signals~\citep{wang2025skewroute,jeong2024adaptive}, but they are mostly heuristic policies optimized for average accuracy without finite-sample guarantees at a user-specified level. 
This leaves a deployment-critical question: \emph{after adaptively combining LLM-only and RAG outputs, how can we control the error rate of the final accepted answers?}


To address this gap, we propose \textbf{BalanceRAG}, a training-free framework for risk-controlled cascaded RAG. 
As shown in Figure~\ref{fig:overview}, given uncertainty scores from the LLM-only and RAG paths, BalanceRAG first attempts the LLM-only answer, invokes the RAG fallback only when the LLM-only is uncertain, and abstains if the fallback is also unreliable. 
The router is governed by two uncertainty thresholds, and BalanceRAG calibrates their joint operating point rather than tuning them stage by stage. 
Stage-wise calibration can be valid in cascaded systems~\citep{jung2024trust}, but it may be overly conservative as the reliability and utility of the final routing cascade depend on the system-level threshold choices for both branches.

Specifically, BalanceRAG formulates joint calibration as a multiple testing problem over a two-dimensional threshold lattice, where each threshold pair defines a candidate operating point, for which we test the null hypothesis that the system-level selection-conditioned risk exceeds $\alpha$. 
Rejecting this null certifies the corresponding pair as safe. 
Although a Bonferroni correction provides valid error control~\citep{bauer1991multiple}, it is overly conservative for this structured lattice: neighboring threshold pairs typically induce highly correlated routing decisions, and safe operating points often form contiguous regions. 
To exploit this structure, BalanceRAG performs Sequential Graphical Testing (SGT)~\citep{https://doi.org/10.1002/sim.3495}, which sequentially propagates significance budget across the lattice and reallocates unused budget from certified nodes to their successors. 
The final router then selects, among all certified threshold pairs, the one with the largest acceptance rate, while preserving high-probability control of the system-level selection-conditioned error rate~\citep{wang2025coin,wang2025lec}. 
Furthermore, we extend BalanceRAG to multi-risk calibration, enabling simultaneous control of answer error and retrieval usage.

Our contributions are threefold. First, to the best of our knowledge, BalanceRAG is the first framework to study finite-sample risk control for cascaded LLM/RAG routing, shifting adaptive retrieval from heuristic confidence estimation to statistically certified reliability. 
Second, we formulate joint threshold calibration as multiple testing over a two-dimensional lattice and adopt SGT to certify high-utility operating points. 
Third, we establish high-probability control of the system-level selection-conditioned error rate and extend the framework to multi-risk calibration.

\section{Related Work}

\noindent \textbf{\textit{Adaptive Retrieval.}}
Adaptive RAG aims to avoid unnecessary retrieval by routing queries based on empirical signals. Previous studies explore when RAG is needed by comparing parametric and non-parametric memory~\citep{mallen2023not}, or route queries using complexity estimation and learned policies~\citep{jeong2024adaptive,tang2025mbaragbanditapproachadaptive}. 
Other methods trigger retrieval during generation based on token-level confidence, self-reflection, or information needs~\citep{jiang2023active,asai2023self,su2024dragindynamicretrievalaugmented}, while recent approaches further exploit retrieval-score statistics and agentic search~\citep{wang2025skewroute,li2025searcho1agenticsearchenhancedlarge}. These methods reduce retrieval usage while preserving average answer quality, but they are mainly heuristic routing policies without statistical guarantees. BalanceRAG is orthogonal to these signals: they can serve as uncertainty estimates, while our calibration layer certifies the final LLM-RAG cascade at a prescribed selection-conditioned risk level.

\noindent \textbf{\textit{Risk Control in Selective Prediction.}}
Conformal prediction~\citep{wang-etal-2024-conu,wang-etal-2025-sconu,wang2025sample,tan2025conformal,jia2025coverage} and risk control~\citep{angelopoulos2025conformalriskcontrol,li2026set} provide finite-sample guarantees for task-specific risk targets. 
Recent work extends risk control to selected foundation-model outputs~\citep{gui2024conformalalignmentknowingtrust}, AI-assigned labeling~\citep{huang2026modelagnosticselectivelabelingprovable}, cascaded LLM judges~\citep{jung2024trust}, and selective QA or routing systems~\citep{wang2025coin,wang2025lec}. These works provide the statistical foundation for reliable selection, but they do not investigate joint risk calibration for cascaded LLM-only/RAG routing. 
BalanceRAG targets this setting by certifying threshold pairs on a two-dimensional routing lattice, reducing conservativeness compared with stage-wise or Bonferroni-style calibration while preserving finite-sample risk control.
\section{Methodology}
\label{sec:method}

\subsection{Problem Formulation}
\label{subsec:problem_formulation}

Given a prompt \(x\), BalanceRAG considers two inference branches: an LLM-only branch and an RAG fallback branch. Branch \(b\in\{1,2\}\) returns an answer \(\hat y_b(x)\) together with an uncertainty score \(u_b(x)\), where smaller values indicate higher confidence. Given a threshold pair \(\lambda=(t_1,t_2)\), the router first accepts the LLM-only answer if \(u_1(x)\le t_1\). Otherwise, it invokes the RAG fallback and accepts \(\hat y_2(x)\) if \(u_2(x)\le t_2\). If both checks fail, the system abstains by returning \(\varnothing\). 
Thus, the final output \(\hat y_\lambda(x)\in\{\hat y_1(x),\hat y_2(x),\varnothing\}\) is determined by the two thresholds jointly.

Let \(A(y^\star,\hat{y})\in\{0,1\}\) indicate whether a generated answer \(\hat y\) is admissible with respect to the ground truth \(y^\star\). We define the acceptance indicator as
\vspace{-0.2cm}
\begin{equation}
\label{eq:selection_indicator}
S_\lambda(x)=\mathbbm{1}[\hat{y}_\lambda(x)\neq \emptyset].
\end{equation}
The accepted-error indicator is
\vspace{-0.2cm}
\begin{equation}
\label{eq:error_indicator}
Z_\lambda(x,y^\star)
=
S_\lambda(x)\left(1-A(y^\star,\hat{y}_\lambda(x))\right).
\end{equation}
Here, \(Z_\lambda(x,y^\star)=1\) means that the router outputs an incorrect answer rather than abstaining. The population risk of interest is the selection-conditioned error rate:
\vspace{-0.2cm}
\begin{equation}
\label{eq:risk_def}
R(\lambda)=\mathbb{P}\left(Z_\lambda(X,Y^\star)=1\mid S_\lambda(X)=1\right).
\end{equation}
Given a target risk level \(\alpha\) and error level \(\delta\), our goal is to select a threshold pair \(\hat{\lambda}\) with high acceptance while satisfying
\vspace{-0.2cm}
\begin{equation}
\label{eq:goal}
\mathbb{P}\!\left(R(\hat{\lambda})\le\alpha\right)\ge 1-\delta .
\end{equation}

\subsection{Joint Calibration with SGT}
\label{subsec:joint_threshold_calibration}

BalanceRAG calibrates the composed LLM-only/RAG router directly. Let \(\Lambda=\{\lambda_j\}_{j=1}^{N}\) denote the candidate threshold pairs on a two-dimensional lattice. For each candidate \(\lambda_j\), we test
\vspace{-0.2cm}
\begin{equation}
\label{eq:hypothesis}
H_j:\; R(\lambda_j)>\alpha .
\end{equation}
Rejecting \(H_j\) certifies \(\lambda_j\) as safe with respect to the target risk level. Since the final threshold pair will be selected adaptively from all certified candidates, we require family-wise error rate~(FWER)~\citep{holm1979simple} control at level \(\delta\). Letting \(\widehat{\Lambda}\) denote the certified set, FWER control ensures
\vspace{-0.2cm}
\begin{equation}
\label{eq:fwer_guarantee}
\mathbb{P}\!\left(
\sup_{\lambda\in\widehat{\Lambda}}R(\lambda)\le\alpha
\right)\ge 1-\delta .
\end{equation}
Therefore, any later choice from \(\widehat{\Lambda}\), including the acceptance-maximizing one, remains valid. A Bonferroni correction also provides this guarantee, but it is often overly conservative on dense threshold grids. BalanceRAG instead uses SGT~\citep{https://doi.org/10.1002/sim.3495}, which preserves FWER control while exploiting the lattice structure to recycle significance budget to more useful operating points.

\noindent \textbf{Step 1: \(p\)-value computation.}
On the calibration set
\(\mathcal{D}_{\mathrm{cal}}=\{(x_\ell,y^\star_\ell)\}_{\ell=1}^n\),
for each threshold pair
\(\lambda_{i,j}=(t_1^{(i)},t_2^{(j)})\), we compute the accepted count and accepted-error count:
\vspace{-0.2cm}
\begin{equation}
\label{eq:mk_def}
M_{i,j}=\sum_{\ell=1}^n S_{\lambda_{i,j}}(x_\ell),\
K_{i,j}=\sum_{\ell=1}^n Z_{\lambda_{i,j}}(x_\ell,y^\star_\ell).
\end{equation}
For a fixed candidate \(\lambda_{i,j}\), conditioned on \(M_{i,j}\),
\vspace{-0.2cm}
\begin{equation}
\label{eq:binomial_model}
K_{i,j}\mid M_{i,j}\sim
\mathrm{Bin}\!\left(M_{i,j},R(\lambda_{i,j})\right).
\end{equation}
We therefore use the one-sided exact binomial \(p\)-value
\vspace{-0.2cm}
\begin{equation}
\label{eq:pvalue}
p_{i,j}
=
\mathbb{P}\!\left(
\mathrm{Bin}(M_{i,j},\alpha)\le K_{i,j}
\right),
\end{equation}
which is super-uniform under \(H_{\lambda_{i,j}}\).

\noindent \textbf{Step 2: graph prior and budget diffusion.}
The candidate thresholds form a natural partial order: increasing either threshold makes the router more permissive and usually increases acceptance. We therefore direct edges from each node to its immediate relaxed neighbors:
\vspace{-0.2cm}
\begin{equation}
\label{eq:neighbors}
\mathcal{N}(i,j)\subseteq\{(i+1,j),(i,j+1)\}.
\end{equation}
As the safe region varies across datasets and models, we split off a small portion of the calibration set as an initialization split. On this split, we choose the highest-acceptance node among those satisfying \(p\le\delta\); if no such node exists, we choose the node with the smallest \(p\)-value. The selected node receives the initial budget \(\delta\), while all formal SGT \(p\)-values are computed on the remaining calibration data.

\begin{figure}
    \centering
    \includegraphics[width=1\linewidth]{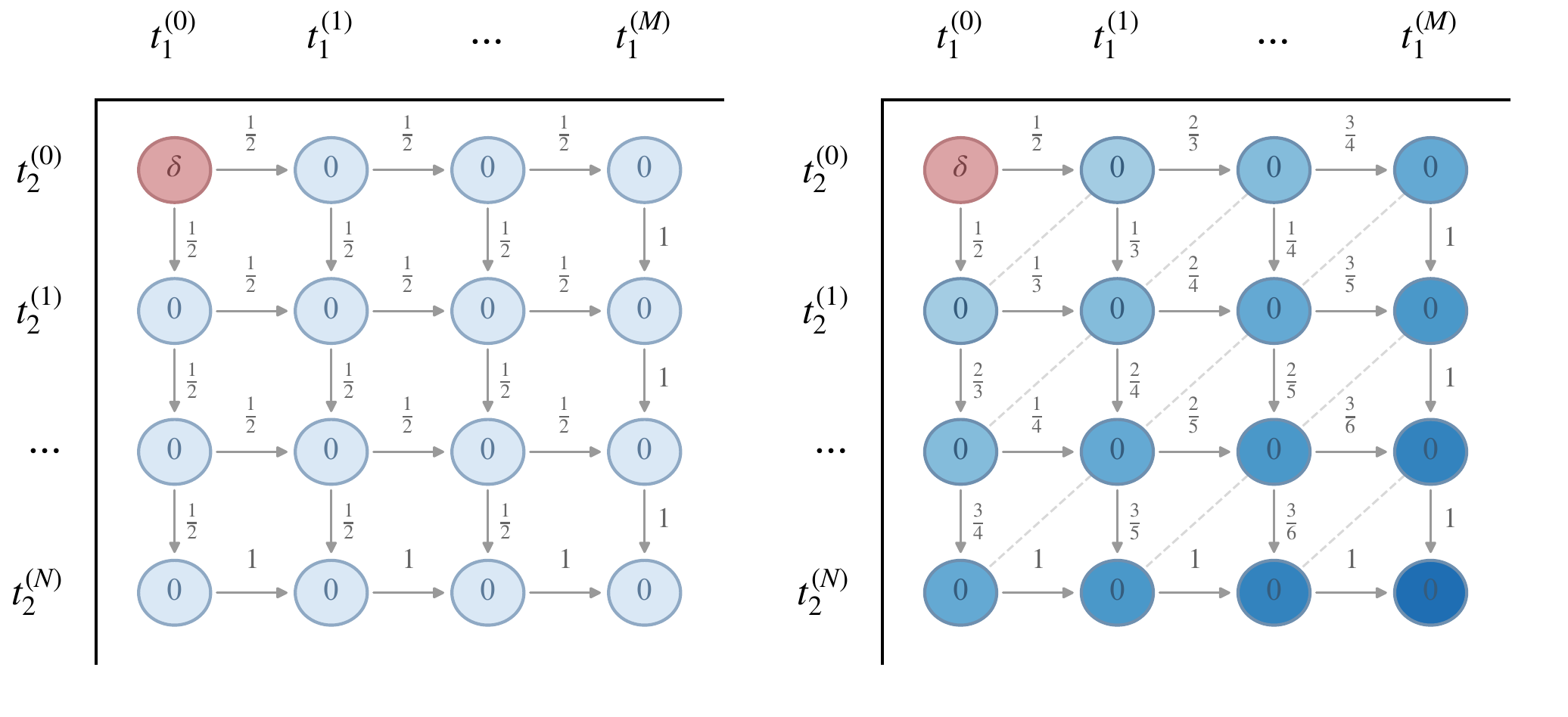}
    \caption{Budget diffusion on the threshold lattice. UGD splits budget evenly across successors, while DWD gives equal accumulated budget to nodes on the same diagonal.}
    \label{fig:budget_diffusion}
    \vspace{-5mm}
\end{figure}

As shown in Figure~\ref{fig:budget_diffusion}, we consider two diffusion rules over successors \(u\in\mathcal{N}(i,j)\). Uniform Grid Diffusion (UGD) assigns
\vspace{-0.2cm}
\begin{equation}
\label{eq:diffusion_weights}
g^{\mathrm{UGD}}_{(i,j)\to u}
=
\frac{1}{|\mathcal{N}(i,j)|}.
\end{equation}
Diamond-Wavefront Diffusion (DWD) assigns
\vspace{-0.2cm}
\begin{equation}
\label{eq:diamond_weights}
\begin{split}
g^{\mathrm{DWD}}_{(i,j)\to u}
&=
\frac{(i+1)\mathbbm{1}\!\left[u=(i+1,j)\right]}{i+j+2} \\
&\quad+
\frac{(j+1)\mathbbm{1}\!\left[u=(i,j+1)\right]}{i+j+2}.
\end{split}
\end{equation}
Under DWD, each node on diagonal \(i+j=d\) receives accumulated budget
\vspace{-0.2cm}
\begin{equation}
\label{eq:shell_equal_budget}
\delta^{\mathrm{DWD}}_{i,j}=\frac{\delta}{d+1},
\end{equation}
so threshold pairs with the same total relaxation receive the same testing opportunity. Appendix~\ref{app:proof_shell} proves this property.

\noindent \textbf{Step 3: SGT budget propagation.}
SGT maintains a local testing budget \(\delta_{i,j}\) at each active node and certifies the node whenever
\vspace{-0.2cm}
\begin{equation}
\label{eq:sig_condition}
p_{i,j}\le\delta_{i,j}.
\end{equation}
Once node \((i,j)\) is certified, its budget is transferred to uncertified successors:
\begin{equation}
\label{eq:budget_update}
\delta_u\leftarrow \delta_u+\delta_{i,j}g_{(i,j)\to u}.
\end{equation}
Because the graph is acyclic, SGT can apply the corresponding rewiring update
\vspace{-0.2cm}
\begin{equation}
\label{eq:sgt_dag_update}
g_{k,u}\leftarrow g_{k,u}+g_{k,(i,j)}g_{(i,j),u},
\end{equation}
where \(k\) is a predecessor and \(u\) is a successor. This budget recycling preserves validity while giving later, more permissive candidates greater power to be certified.

\noindent \textbf{Step 4: final threshold selection.}
After SGT terminates, BalanceRAG selects the certified threshold pair with the largest accepted count:
\vspace{-0.2cm}
\begin{equation}
\label{eq:select_final}
\hat{\lambda}
=
\argmax_{\lambda\in\widehat{\Lambda}}M(\lambda).
\end{equation}
Since this selection is restricted to the certified set, it inherits the simultaneous risk guarantee in Eq.~\eqref{eq:fwer_guarantee}.

\begin{theorem}[Risk control of calibrated routing]
\label{thm:main_guarantee}
Under i.i.d.\ calibration data, for any target risk \(\alpha\) and error level \(\delta\), any threshold pair \(\hat{\lambda}\) selected from the SGT-certified set satisfies
\vspace{-0.2cm}
\begin{equation}
\mathbb{P}\left(R(\hat{\lambda})\le\alpha\right)\ge 1-\delta.
\end{equation}
\end{theorem}

The proof is given in Appendix~\ref{app:proof_main}.

\subsection{Multi-Risk Control}
\label{subsec:multi_risk}

The same certification procedure can control multiple system-level risks. Besides the selection-conditioned answer error,
\vspace{-0.2cm}
\begin{equation}
\label{eq:answer_risk}
R^{(1)}(\lambda)=\mathbb{P}\left(Z_\lambda=1\mid S_\lambda=1\right),
\end{equation}
we also control the fallback invocation rate,
\vspace{-0.2cm}
\begin{equation}
\label{eq:fallback_risk}
R^{(2)}(\lambda)=\mathbb{P}(B_\lambda=1),
\end{equation}
where \(B_\lambda\) indicates whether the input is routed to the RAG fallback. Given targets \(\alpha_1\) and \(\alpha_2\), each node receives one \(p\)-value per risk, denoted by \(p^{(1)}_{i,j}\) and \(p^{(2)}_{i,j}\). The intersection-union rule~\citep{4d6e5e99-f202-3927-88ed-b212186d1c14} combines them as
\vspace{-0.2cm}
\begin{equation}
\label{eq:iut_pvalue}
p_{i,j}=\max\left(p^{(1)}_{i,j},p^{(2)}_{i,j}\right).
\end{equation}
Feeding this node-level \(p\)-value into the same SGT procedure yields
\vspace{-0.2cm}
\begin{equation}
\label{eq:multi_risk_guarantee}
\mathbb{P}\left(
R^{(1)}(\hat{\lambda})\le\alpha_1,\;
R^{(2)}(\hat{\lambda})\le\alpha_2
\right)
\ge 1-\delta.
\end{equation}
Thus, multi-risk BalanceRAG jointly certifies answer reliability and fallback usage, allowing the final router to trade off coverage and retrieval frequency under explicit risk constraints.


\section{Experiment}
\begin{figure*}[!t]
  \centering
  \captionsetup[subfigure]{skip=0pt}
  \setlength{\abovecaptionskip}{2pt}
  \setlength{\belowcaptionskip}{0pt}

  \begin{subfigure}[t]{0.5\textwidth}
    \centering
    \includegraphics[width=\linewidth]{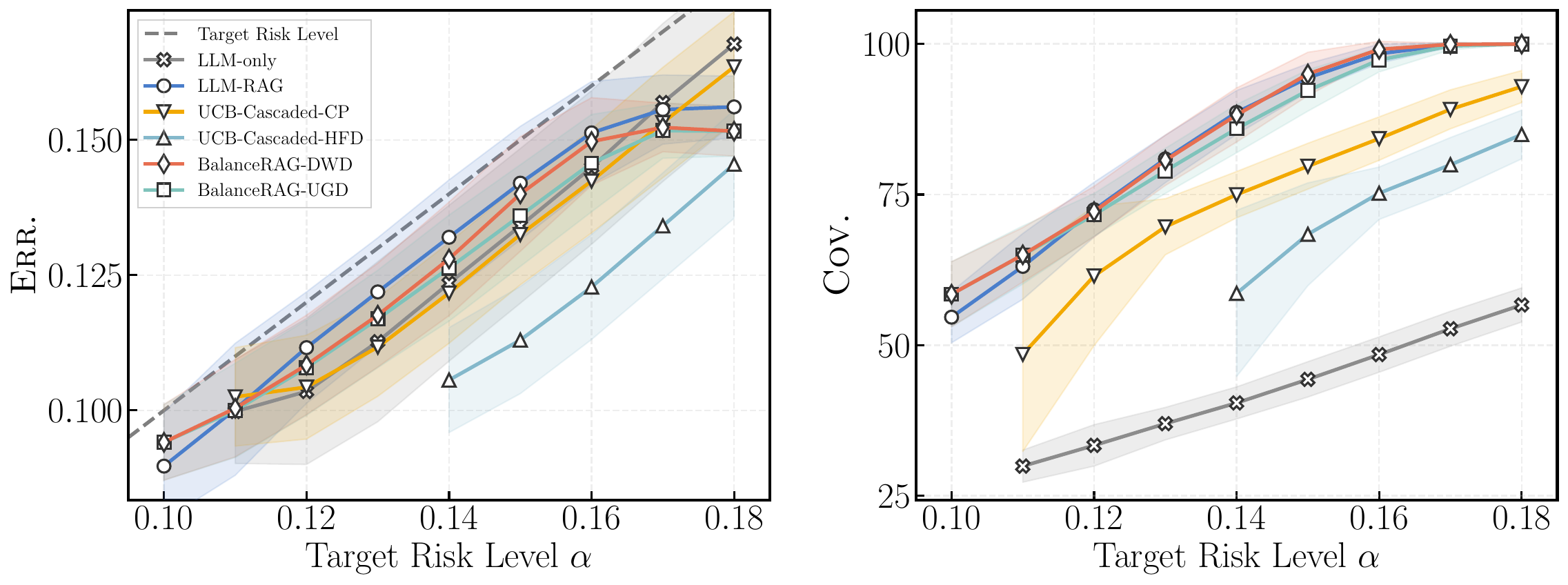}
    \caption{Qwen2.5-3B}
  \end{subfigure}%
  \begin{subfigure}[t]{0.5\textwidth}
    \centering
    \includegraphics[width=\linewidth]{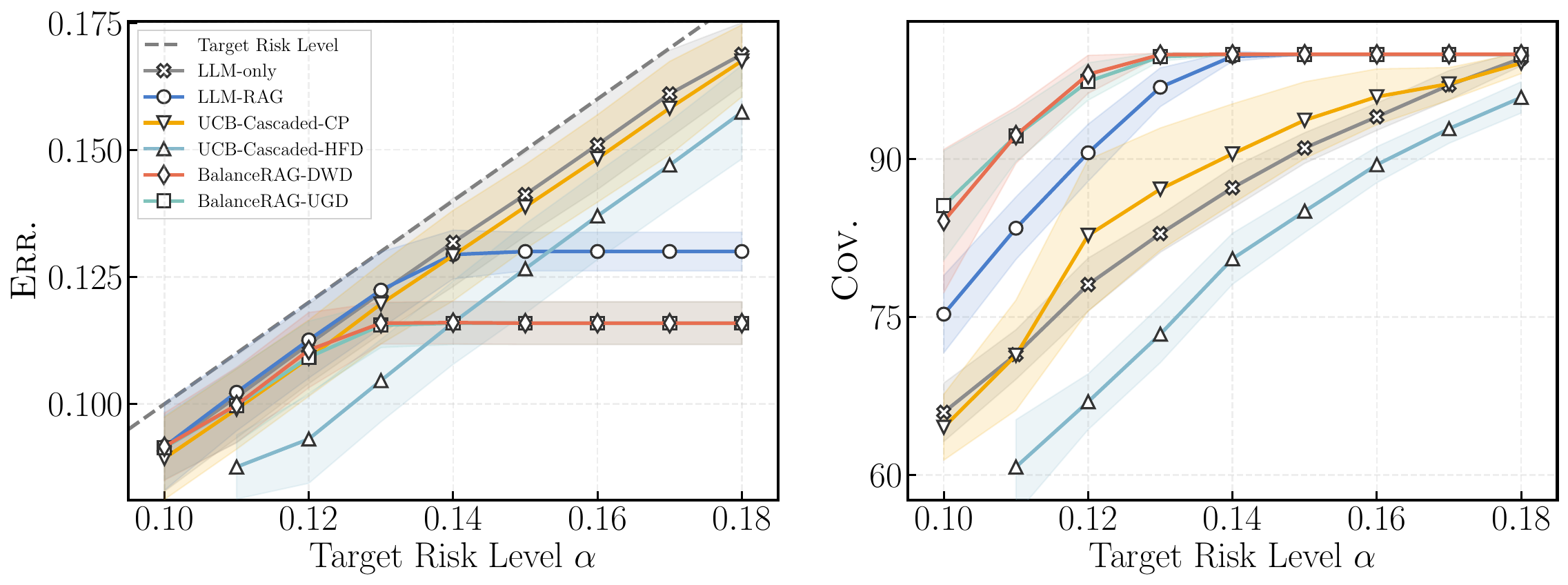}
    \caption{Qwen2.5-7B}
  \end{subfigure}

  \vspace{-0.1em}

  \begin{subfigure}[t]{0.5\textwidth}
    \centering
    \includegraphics[width=\linewidth]{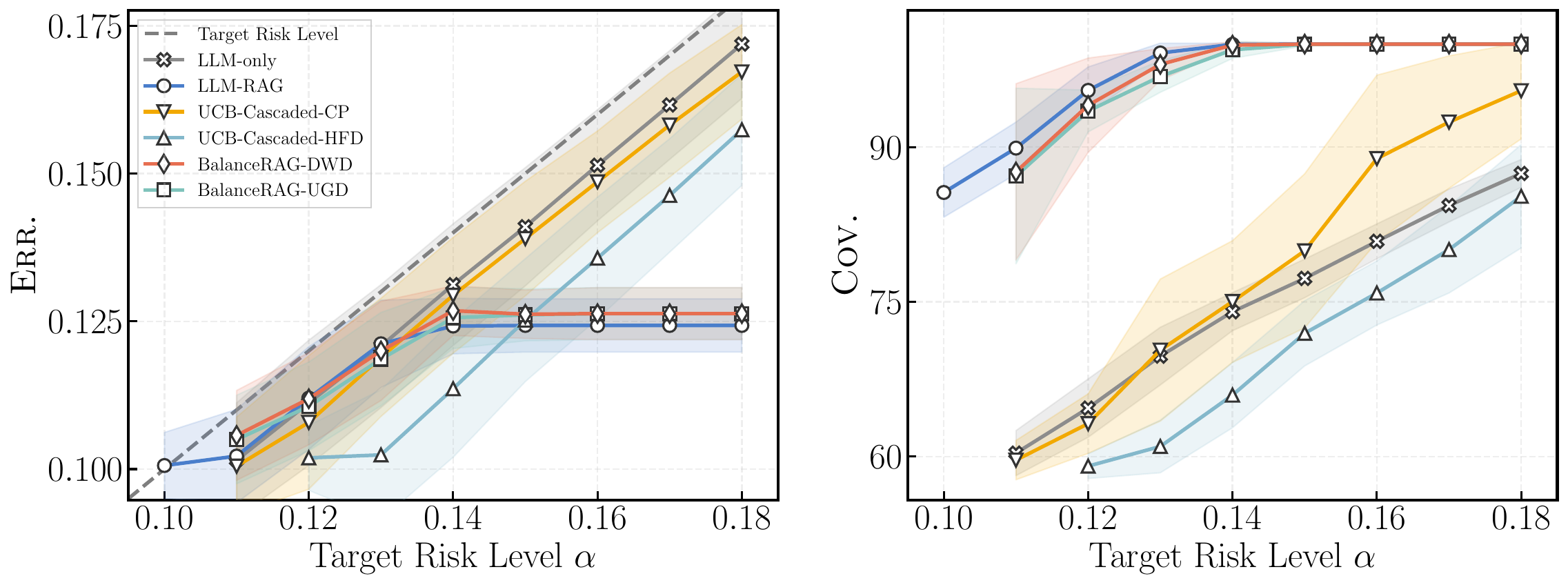}
    \caption{Qwen3-4B}
  \end{subfigure}%
  \begin{subfigure}[t]{0.5\textwidth}
    \centering
    \includegraphics[width=\linewidth]{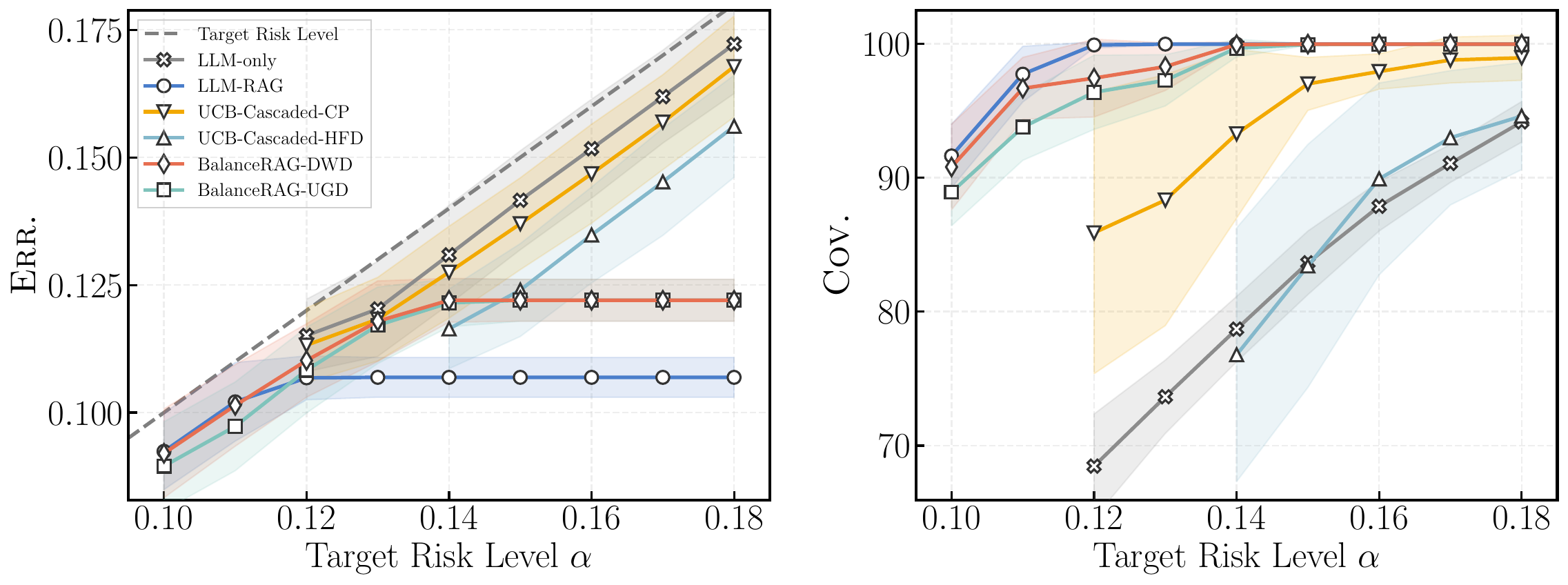}
    \caption{Qwen3-14B}
  \end{subfigure}

  \vspace{-0.1em}

  \begin{subfigure}[t]{0.5\textwidth}
    \centering
    \includegraphics[width=\linewidth]{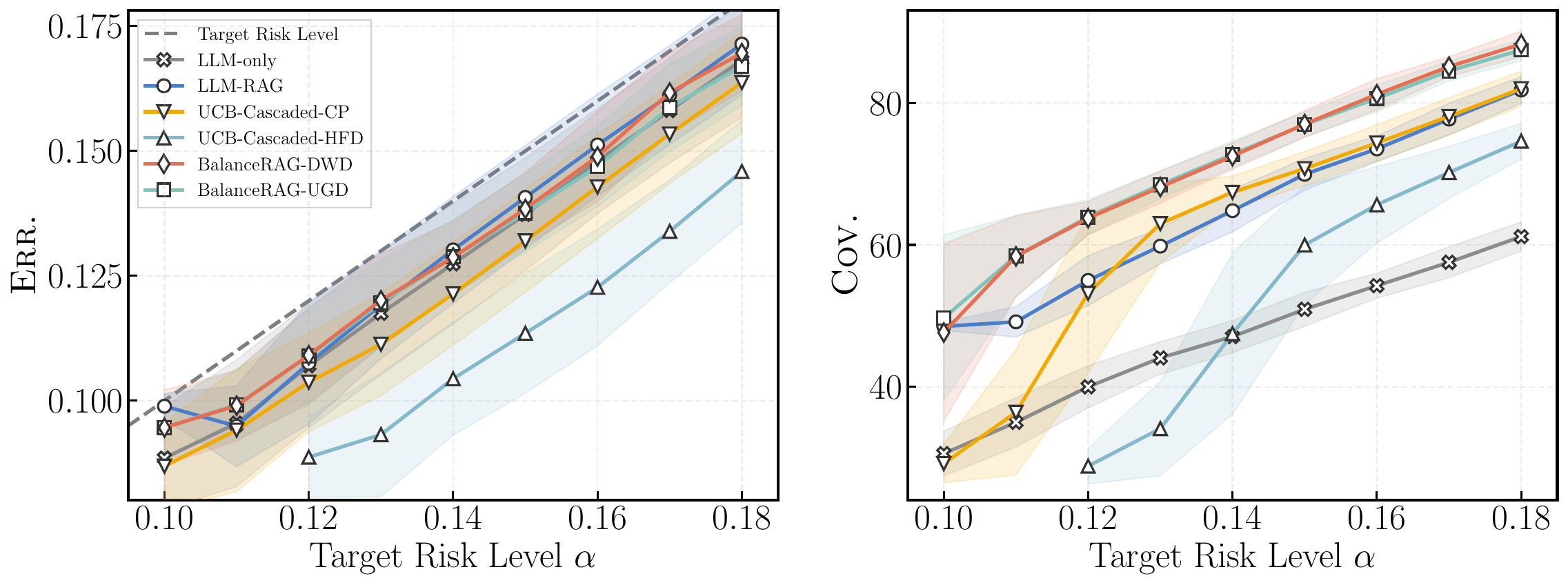}
    \caption{Llama3.2-3B}
  \end{subfigure}%
  \begin{subfigure}[t]{0.5\textwidth}
    \centering
    \includegraphics[width=\linewidth]{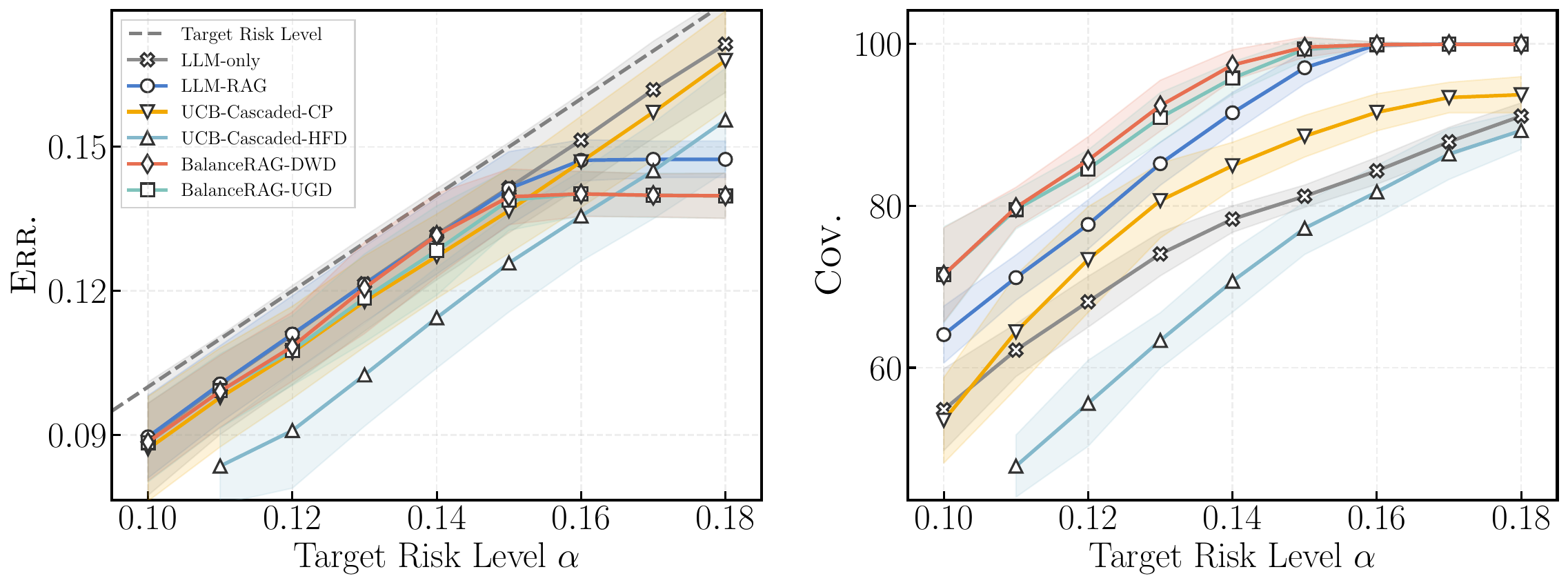}
    \caption{Llama3.1-8B}
  \end{subfigure}

  \vspace{-0.1em}

  \begin{subfigure}[t]{0.5\textwidth}
    \centering
    \includegraphics[width=\linewidth]{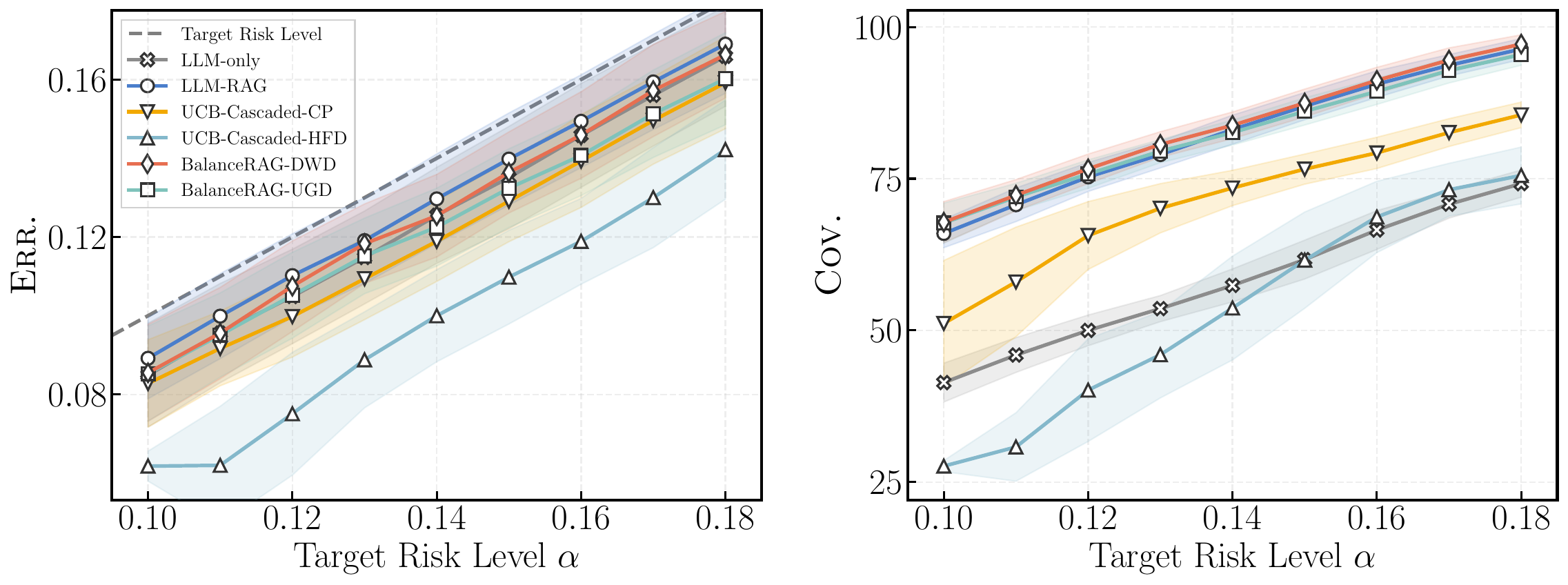}
    \caption{Vicuna-v1.5-7B}
  \end{subfigure}%
  \begin{subfigure}[t]{0.5\textwidth}
    \centering
    \includegraphics[width=\linewidth]{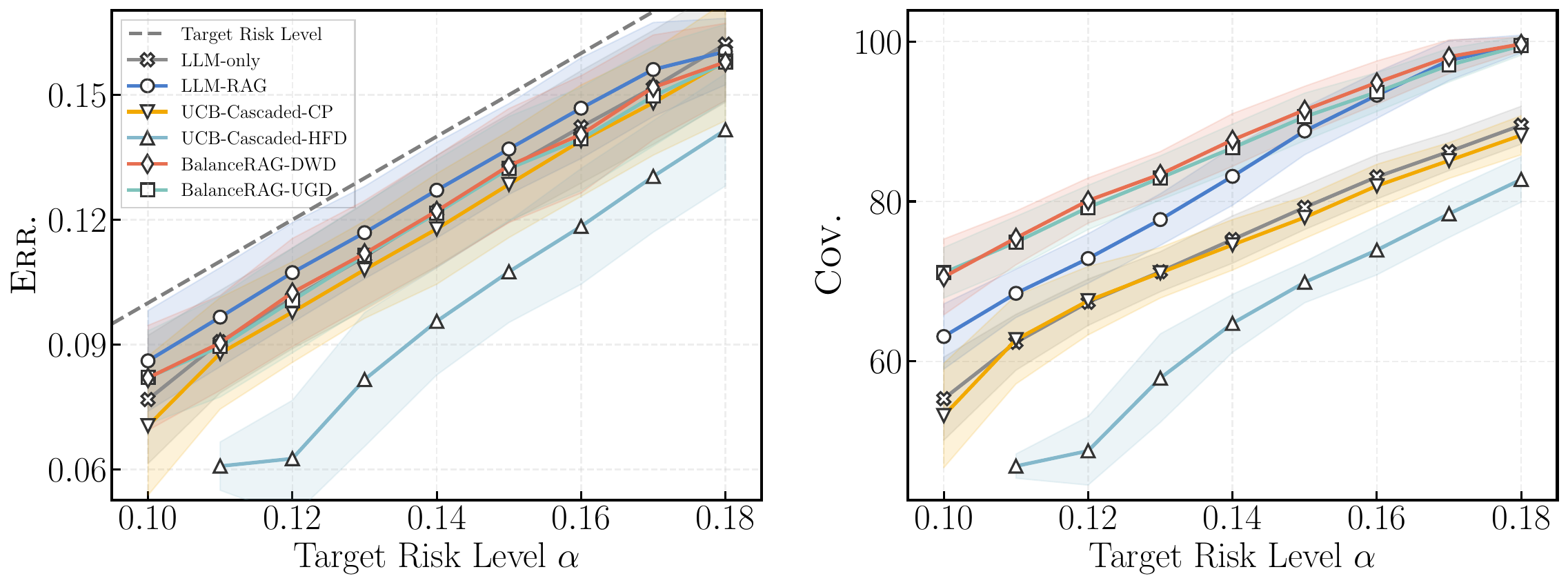}
    \caption{Vicuna-v1.5-13B}
  \end{subfigure}

  \caption{{Err.} (left in each panel pair) and {Cov.} (right) under
different target risk levels $\alpha$ across eight LLMs spanning four model
families on TriviaQA.}
  \label{fig:risk_level_triviaqa}
  \vspace{-5mm}
\end{figure*}

\subsection{Experimental Settings}
\noindent \textbf{Datasets and Base LLMs.}
We evaluate BalanceRAG on three open-ended QA datasets:
TriviaQA~\citep{joshi2017triviaqa}, SQuAD~v2~\citep{rajpurkar2018know},
and Natural Questions (NQ)~\citep{kwiatkowski2019natural}.
To assess robustness across model families and scales, we employ eight backbone LLMs
from Qwen~\citep{yang2025qwen3,qwen2025qwen25technicalreport},
LLaMA~\citep{grattafiori2024llama3herdmodels}, and
Vicuna~\citep{zheng2023judgingllmasajudgemtbenchchatbot}.

\begin{figure}[t]
    \centering
    \begin{subfigure}{0.235\textwidth}
        \centering
        \includegraphics[width=\linewidth]{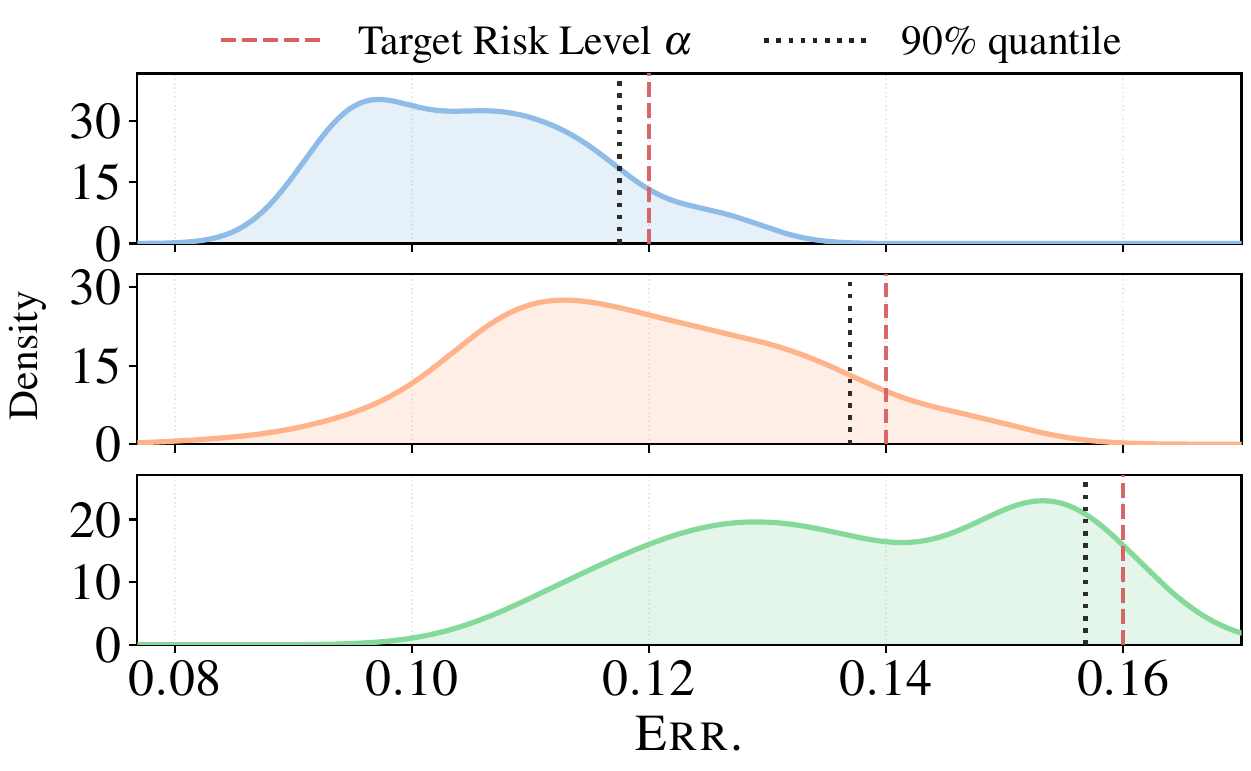} 
    \end{subfigure}
    \begin{subfigure}
        {0.235\textwidth}
        \centering
        \includegraphics[width=\linewidth]{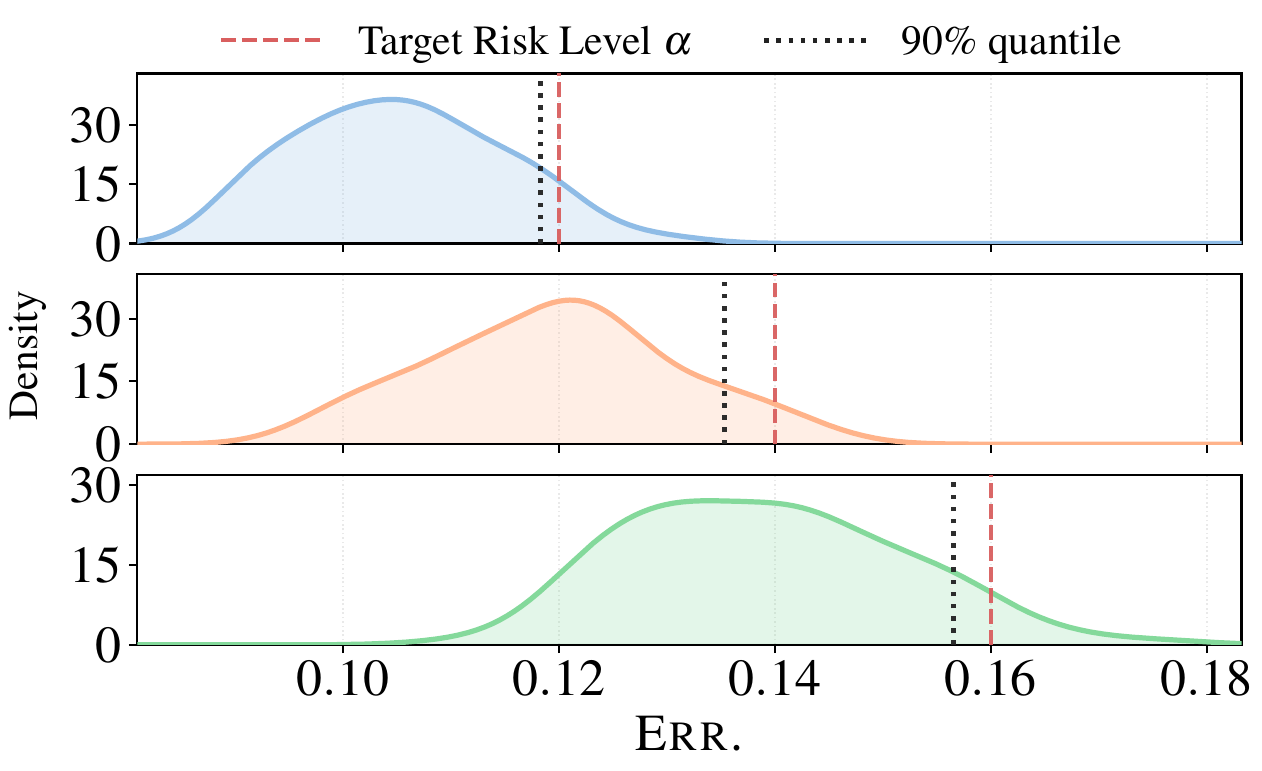} 
    \end{subfigure}
        \caption{{Err.} distributions for Qwen2.5-3B (Left) and Llama-3.2-3B (Right) on TriviaQA.}
    \label{fig:fdr_violin}
    \vspace{-5mm}
\end{figure}

\begin{table*}[t]
\centering
\footnotesize
\setlength{\tabcolsep}{4pt}
\renewcommand{\arraystretch}{1.02}
\caption{Results on TriviaQA, SQuAD~v2, and NQ under different target risk levels.
Bold numbers indicate the best Corr. under the same dataset, model, and risk
level. ``-'' denotes infeasibility, i.e., no non-empty certified operating point is
found under the target risk level.}
\label{tab:three_datasets_results_merged}

\resizebox{\textwidth}{!}{
\begin{tabular}{@{}c@{\hspace{2mm}}llccccc@{\hspace{4mm}}ccccc@{}}
\toprule
& Method & Metric
& \multicolumn{5}{c}{\textit{Llama3.2-3B}}
& \multicolumn{5}{c@{}}{\textit{Llama3.1-8B}} \\
\cmidrule(lr){1-8}\cmidrule(lr){9-13}

\multirow{11}{*}[-0.5mm]{\rotatebox[origin=c]{90}{\textbf{TriviaQA}}}
& \multicolumn{2}{l}{$\alpha$}
& 0.11 & 0.12 & 0.13 & 0.14 & 0.15
& 0.11 & 0.12 & 0.13 & 0.14 & 0.15 \\
\specialrule{0.06em}{0pt}{1pt}

& \multirow{2}{*}{\Only} & Corr.
& 988.0 & 1116.0 & 1214.0 & 1283.0 & 1371.0
& 1780.0 & 1928.0 & 2070.0 & 2165.0 & 2218.0 \\
&  & Err.
& 0.0955 & 0.1070 & 0.1175 & 0.1274 & 0.1373
& 0.1004 & 0.1108 & 0.1215 & 0.1318 & 0.1415 \\
\addlinespace[0.8pt]

& \multirow{2}{*}{\RAG} & Corr.
& \underline{1390.0} & \underline{1532.0} & 1646.0 & 1760.0 & 1876.0
& \underline{2036.0} & \underline{2198.0} & \underline{2382.0} & \underline{2529.0} & \underline{2652.0} \\
&  & Err.
& 0.0949 & 0.1075 & 0.1192 & 0.1302 & 0.1407
& 0.1006 & 0.1110 & 0.1214 & 0.1314 & 0.1413 \\
\addlinespace[0.8pt]

& \multirow{2}{*}{\UCB} & Corr.
& 1028.0 & 1488.0 & \underline{1749.0} & \underline{1850.0} & \underline{1917.0}
& 1849.0 & 2084.0 & 2264.0 & 2360.0 & 2434.0 \\
&  & Err.
& 0.0941 & 0.1037 & 0.1113 & 0.1214 & 0.1320
& 0.0978 & 0.1072 & 0.1178 & 0.1272 & 0.1367 \\
\addlinespace[0.8pt]

& \multirow{2}{*}{\Balance} & Corr.
& \balcell{\textbf{1643.0}} & \balcell{\textbf{1774.0}} & \balcell{\textbf{1873.0}} & \balcell{\textbf{1973.0}} & \balcell{\textbf{2073.0}}
& \balcell{\textbf{2288.0}} & \balcell{\textbf{2430.0}} & \balcell{\textbf{2584.0}} & \balcell{\textbf{2691.0}} & \balcell{\textbf{2727.0}} \\
&  & Err.
& \balcell{0.099} & \balcell{0.1091} & \balcell{0.1201} & \balcell{0.1287} & \balcell{0.1384}
& \balcell{0.0991} & \balcell{0.1084} & \balcell{0.1206} & \balcell{0.1316} & \balcell{0.1396} \\

\midrule

\multirow{11}{*}[-0.5mm]{\rotatebox[origin=c]{90}{\textbf{SQuAD~v2}}}
& \multicolumn{2}{l}{$\alpha$}
& 0.11 & 0.12 & 0.13 & 0.14 & 0.15
& 0.11 & 0.12 & 0.13 & 0.14 & 0.15 \\
\specialrule{0.06em}{0pt}{1pt}

& \multirow{2}{*}{\Only} & Corr.
& - & - & - & - & 137.0
& - & - & - & 194.0 & 211.0 \\
&  & Err.
& - & - & - & - & 0.1345
& - & - & - & 0.1319 & 0.1324 \\
\addlinespace[0.8pt]

& \multirow{2}{*}{\RAG} & Corr.
& \textbf{1546.0} & \textbf{1617.0} & \textbf{1666.0} & \textbf{1714.0} & \textbf{1755.0}
& \textbf{1803.0} & \textbf{1831.0} & \textbf{1850.0} & \textbf{1853.0} & \textbf{1853.0} \\
&  & Err.
& 0.1008 & 0.1112 & 0.1216 & 0.1316 & 0.1416
& 0.1000 & 0.1095 & 0.1175 & 0.1193 & 0.1193 \\
\addlinespace[0.8pt]

& \multirow{2}{*}{\UCB} & Corr.
& - & - & - & - & 1744.0
& - & - & - & 1835.0 & 1837.0 \\
&  & Err.
& - & - & - & - & 0.1354
& - & - & - & 0.1263 & 0.1269 \\
\addlinespace[0.8pt]

& \multirow{2}{*}{\Balance} & Corr.
& \balcell{\underline{1484.0}} & \balcell{\underline{1591.0}} & \balcell{\underline{1654.0}} & \balcell{\underline{1699.0}} & \balcell{\underline{1745.0}}
& \balcell{\underline{1784.0}} & \balcell{\underline{1819.0}} & \balcell{\underline{1836.0}} & \balcell{\underline{1842.0}} & \balcell{\underline{1843.0}} \\
&  & Err.
& \balcell{0.0975} & \balcell{0.1066} & \balcell{0.1164} & \balcell{0.1271} & \balcell{0.1375}
& \balcell{0.1010} & \balcell{0.1118} & \balcell{0.1210} & \balcell{0.1241} & \balcell{0.1240} \\

\midrule

\multirow{11}{*}[-0.5mm]{\rotatebox[origin=c]{90}{\textbf{NQ}}}
& \multicolumn{2}{l}{$\alpha$}
& 0.15 & 0.16 & 0.17 & 0.18 & 0.19
& 0.15 & 0.16 & 0.17 & 0.18 & 0.19 \\
\specialrule{0.06em}{0pt}{1pt}

& \multirow{2}{*}{\Only} & Corr.
& - & - & 119.0 & 125.0 & 134.0
& 158.0 & 191.0 & 240.0 & 283.0 & 320.0 \\
&  & Err.
& - & - & 0.1568 & 0.1552 & 0.1572
& 0.1055 & 0.1184 & 0.1362 & 0.1493 & 0.1630 \\
\addlinespace[0.8pt]

& \multirow{2}{*}{\RAG} & Corr.
& \textbf{489.0} & \textbf{532.0} & \textbf{562.0} & \textbf{590.0} & \textbf{615.0}
& \textbf{734.0} & \textbf{752.0} & \textbf{760.0} & \textbf{763.0} & \textbf{763.0} \\
&  & Err.
& 0.1283 & 0.1390 & 0.1497 & 0.1601 & 0.1697
& 0.1331 & 0.1422 & 0.1466 & 0.1478 & 0.1478 \\
\addlinespace[0.8pt]

& \multirow{2}{*}{\UCB} & Corr.
& - & - & - & 563.0 & 590.0
& 698.0 & 718.0 & 736.0 & 743.0 & 747.0 \\
&  & Err.
& - & - & - & 0.1520 & 0.1582
& 0.1214 & 0.1307 & 0.1408 & 0.1492 & 0.1581 \\
\addlinespace[0.8pt]

& \multirow{2}{*}{\Balance} & Corr.
& \balcell{\underline{447.0}} & \balcell{\underline{492.0}} & \balcell{\underline{531.0}} & \balcell{\underline{575.0}} & \balcell{\underline{615.0}}
& \balcell{\underline{709.0}} & \balcell{\underline{741.0}} & \balcell{\underline{753.0}} & \balcell{\underline{757.0}} & \balcell{\underline{759.0}} \\
&  & Err.
& \balcell{0.122} & \balcell{0.1312} & \balcell{0.1404} & \balcell{0.1506} & \balcell{0.1662}
& \balcell{0.1315} & \balcell{0.1420} & \balcell{0.1511} & \balcell{0.1523} & \balcell{0.1517} \\

\bottomrule
\end{tabular}
}
\vspace{-2mm}
\end{table*}

\noindent \textbf{Baselines.}
We compare BalanceRAG with three groups of baselines. First, we include fixed-path systems, LLM-only and LLM-RAG, both calibrated by fixed sequence testing~\citep{bauer1991multiple} as single-branch risk-controlled baselines. Second, we also evaluate two BalanceRAG variants: BalanceRAG-UGD, which uniformly diffuses the significance budget to neighboring nodes, and BalanceRAG-DWD, which uses diamond-wavefront diffusion to balance testing opportunities among threshold pairs with the same total relaxation. We employ DWD as the default variant. Third, we compare with UCB-based methods for selection-conditioned risk control. UCB-Cascaded-CP adapts the cascaded judge framework of~\citet{jung2024trust} to our LLM-only/RAG routing setting with the Clopper-Pearson bound, while UCB-Cascaded-HFD replaces it with a distribution-free Hoeffding bound. For completeness, Appendix~\ref{tab:llama31_8b_risk_control_comparison} reports additional baselines, including Adaptive-RAG~\citep{jeong2024adaptive}, Self-Route~\citep{li2024retrieval}, empirical risk selection, and Bonferroni correction, covering uncalibrated adaptive routing, simple empirical selection, and conservative multiple-testing correction under the same target risk levels.

\noindent \textbf{Evaluation Metrics.}
We consider four dimensions. Reliability is measured by the empirical selection-conditioned error rate ({Err.}) under different target risk levels $\alpha$. Coverage is measured by the acceptance rate ({Cov.}). Utility is measured by the number of accepted correct samples ({Corr.}), which captures how many useful answers the system produces. Cost is measured by the average number of tokens and the average memory increase per sample after uncertainty scores are obtained, so it reflects the execution cost induced by the calibrated routing policy rather than the overhead of any specific uncertainty estimator.

\noindent \textbf{Correctness Criterion and Uncertainty Estimator.}
By default, we use semantic similarity~\citep{reimers2019sentence} with a threshold of $0.7$ to determine whether a model answer aligns with the ground truth in the admission function. In the robustness analysis, we further consider bidirectional entailment~\citep{kuhn2023semantic} and LLM-as-a-Judge~\citep{zhang2024vluncertaintydetectinghallucinationlarge} as alternative correctness criteria. For uncertainty estimation, we consider five estimators: semantic entropy (SE)~\citep{farquhar2024detecting}, the sum of graph Laplacian eigenvalues (EIGV), eccentricity (ECC), the degree matrix statistic (DEG)~\citep{lin2023generating}, and length-normalized predictive entropy from the model outputs (SELF)~\citep{malinin2020uncertainty}. Unless otherwise stated, we use SE as the default uncertainty estimator.

\begin{figure*}[htbp]
  \centering
  \begin{subfigure}{0.245\textwidth}
    \centering
    \includegraphics[width=\linewidth]{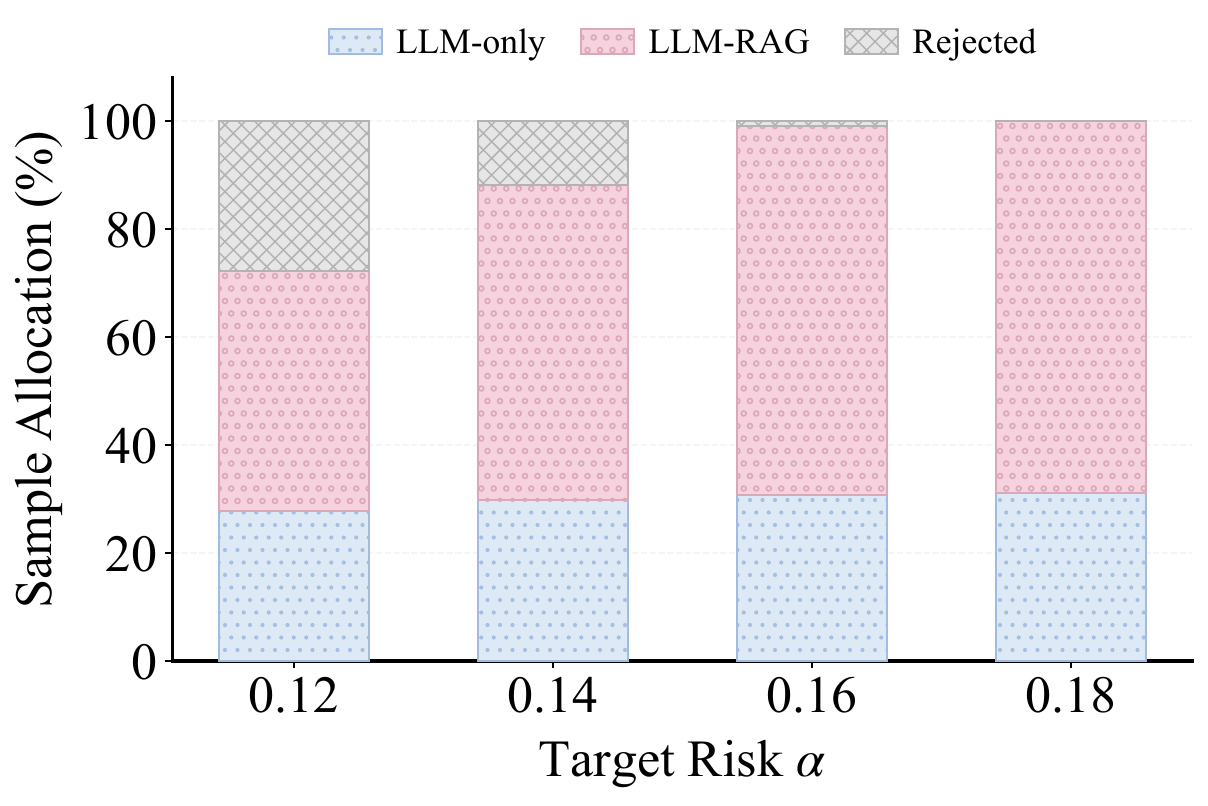}
    \caption{Qwen2.5-3B}
  \end{subfigure}
  \hfill
  \begin{subfigure}{0.245\textwidth}
    \centering
    \includegraphics[width=\linewidth]{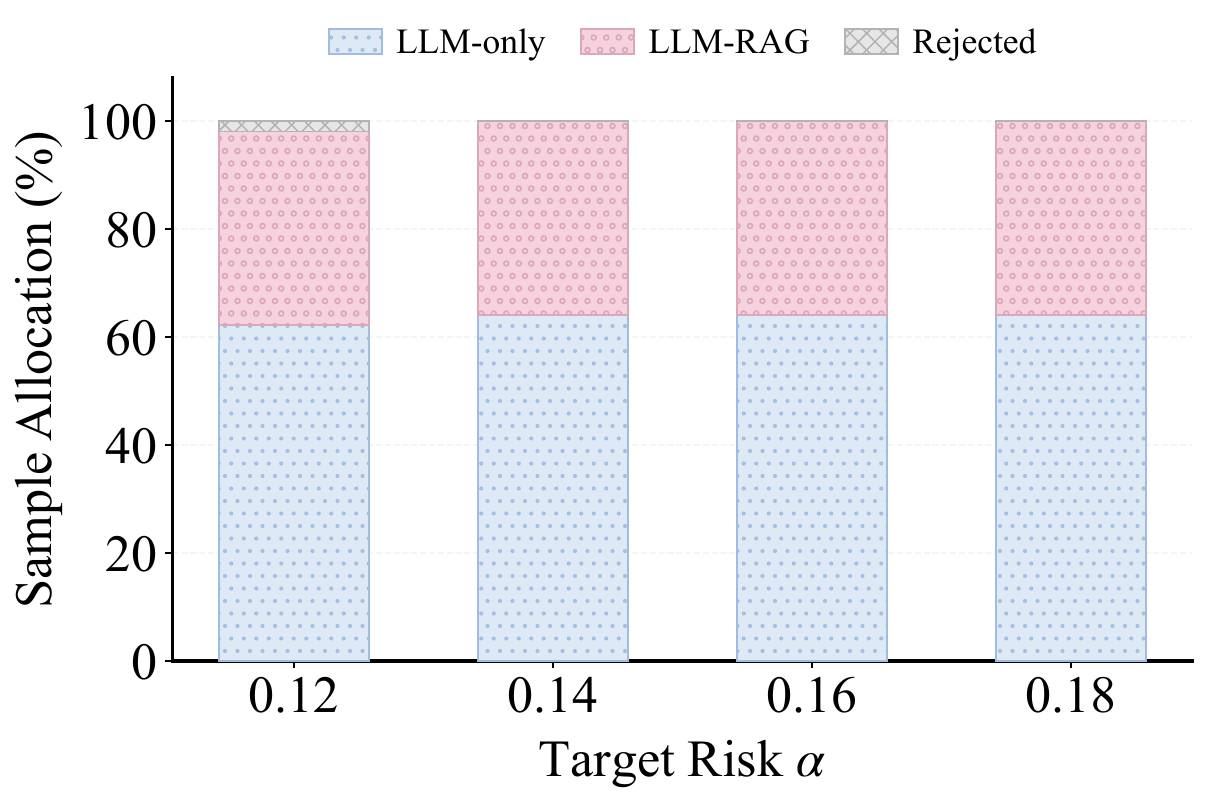}
    \caption{Qwen2.5-7B}
  \end{subfigure}
   \hfill
  \begin{subfigure}{0.245\textwidth}
    \centering
    \includegraphics[width=\linewidth]{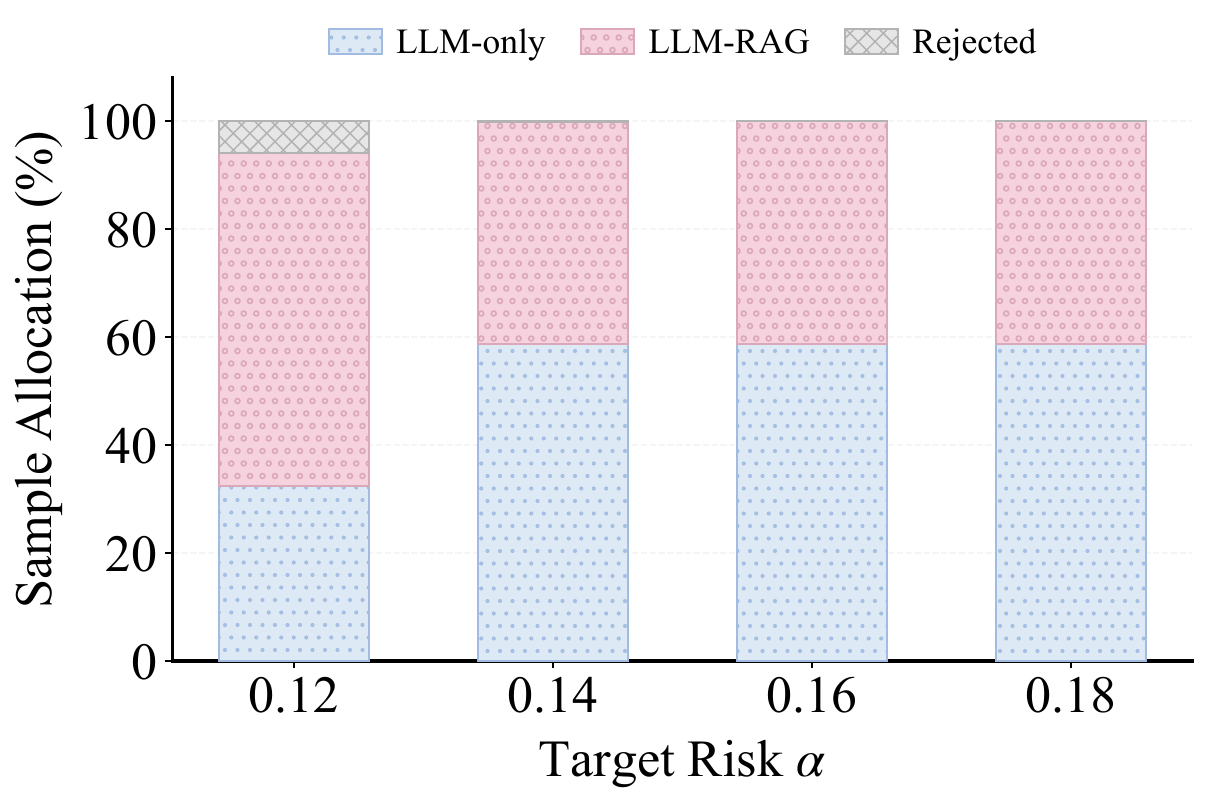}
    \caption{Qwen3-4B}
  \end{subfigure}
  \hfill
  \begin{subfigure}{0.245\textwidth}
    \centering
    \includegraphics[width=\linewidth]{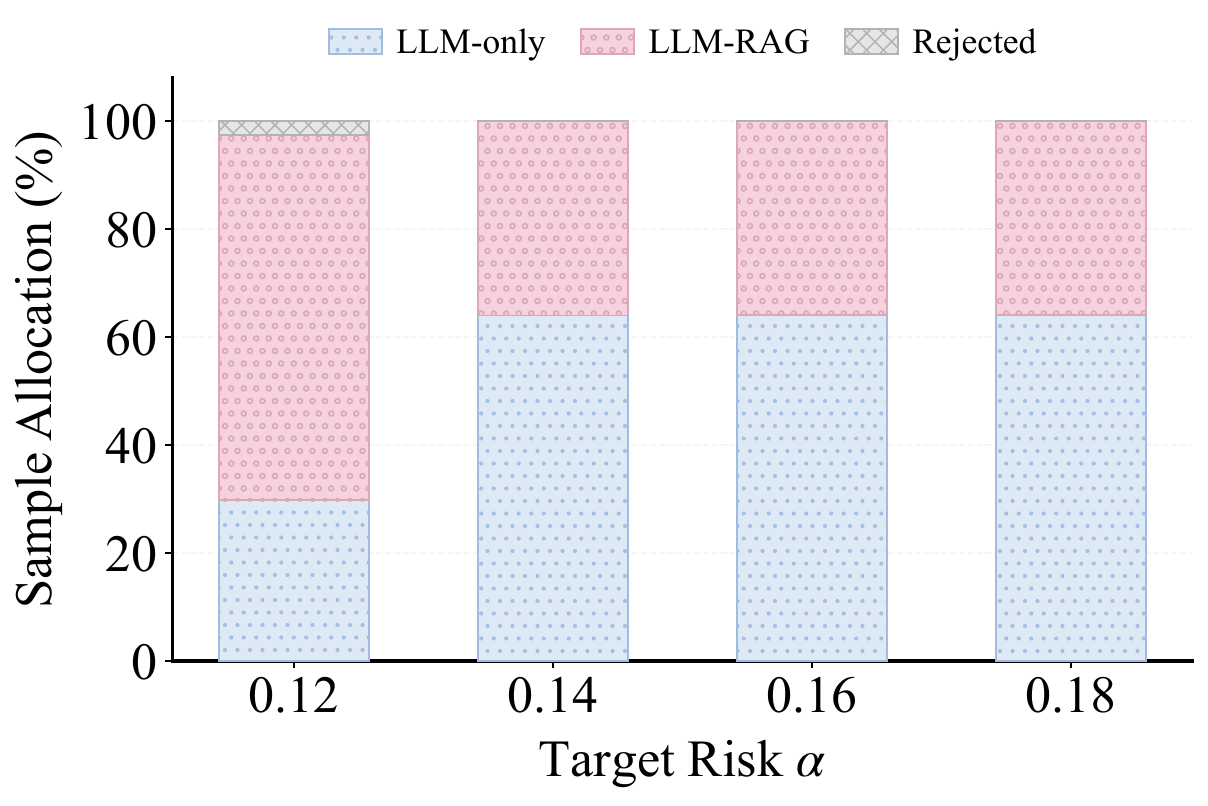}
    \caption{Qwen3-14B}
  \end{subfigure}
  \hfill
  \begin{subfigure}{0.245\textwidth}
    \centering
    \includegraphics[width=\linewidth]{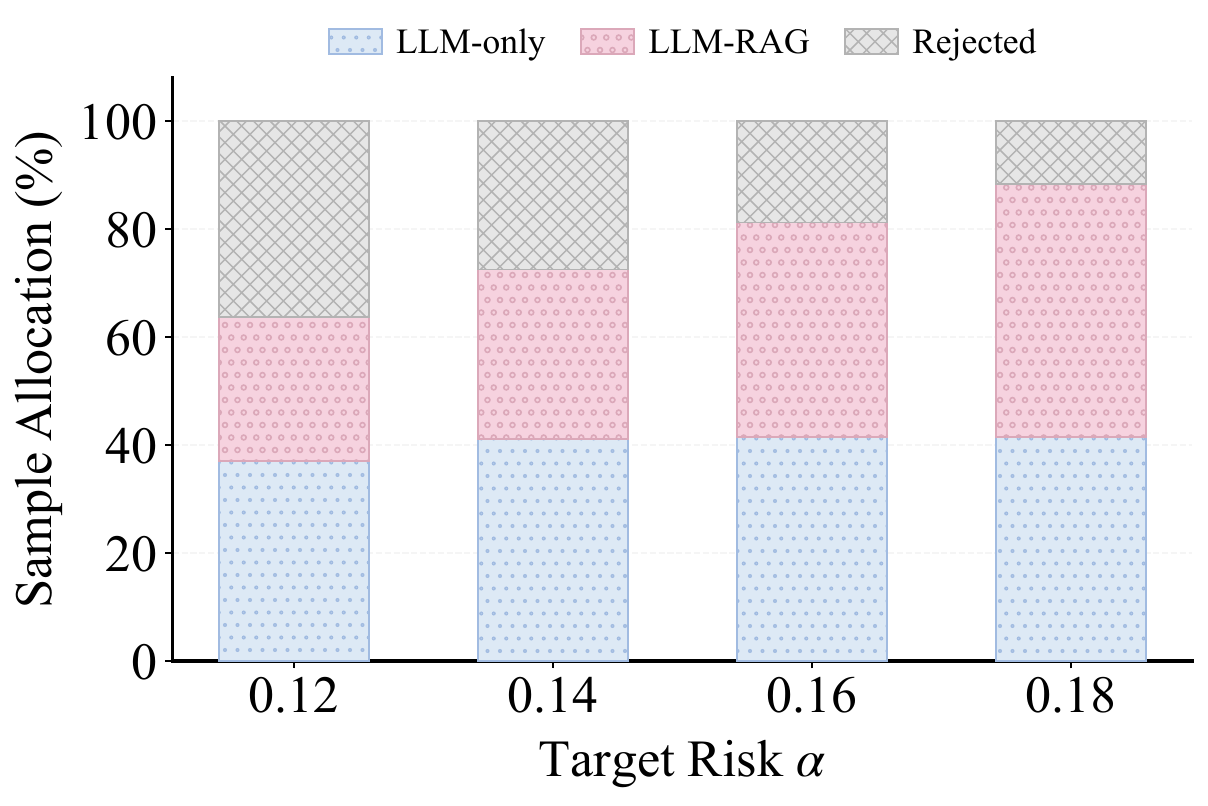}
    \caption{Llama3.2-3B}
  \end{subfigure}
  \hfill
  \begin{subfigure}{0.245\textwidth}
    \centering
    \includegraphics[width=\linewidth]{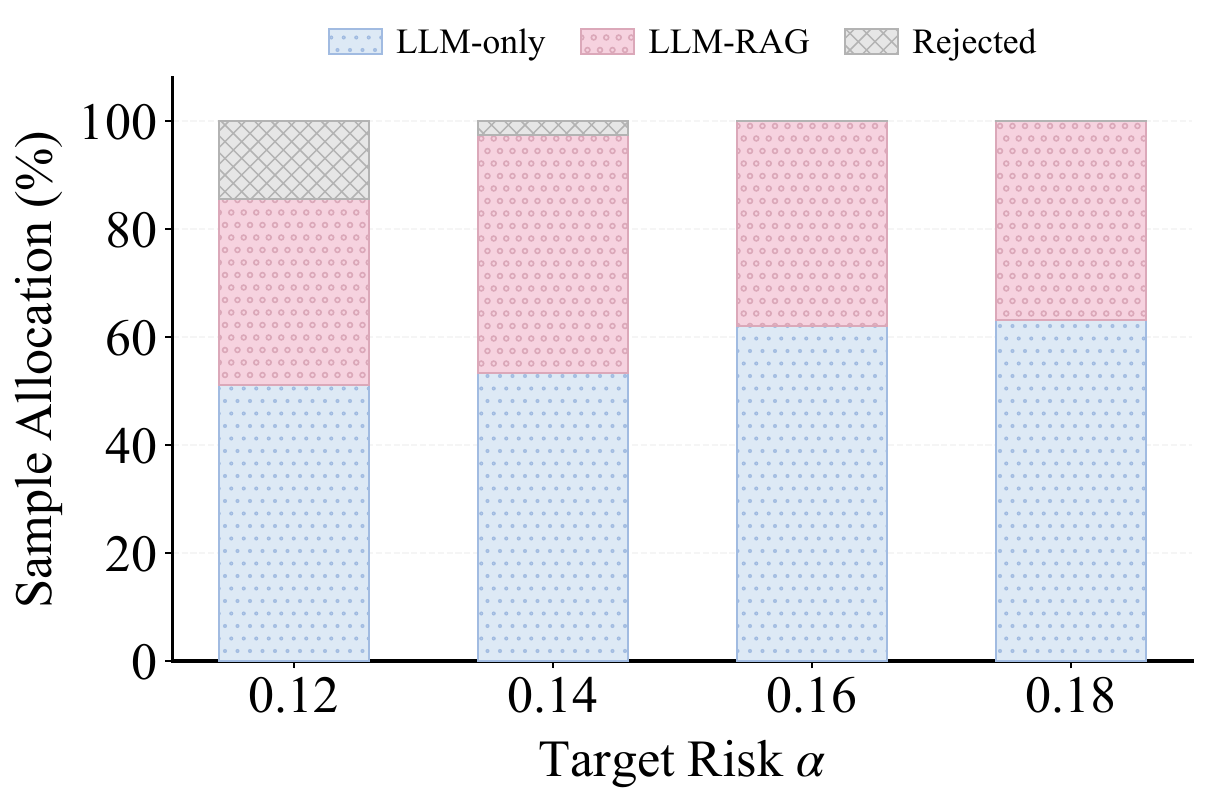}
    \caption{Llama3.1-8B}
  \end{subfigure}
  \hfill
  \begin{subfigure}{0.245\textwidth}
    \centering
    \includegraphics[width=\linewidth]{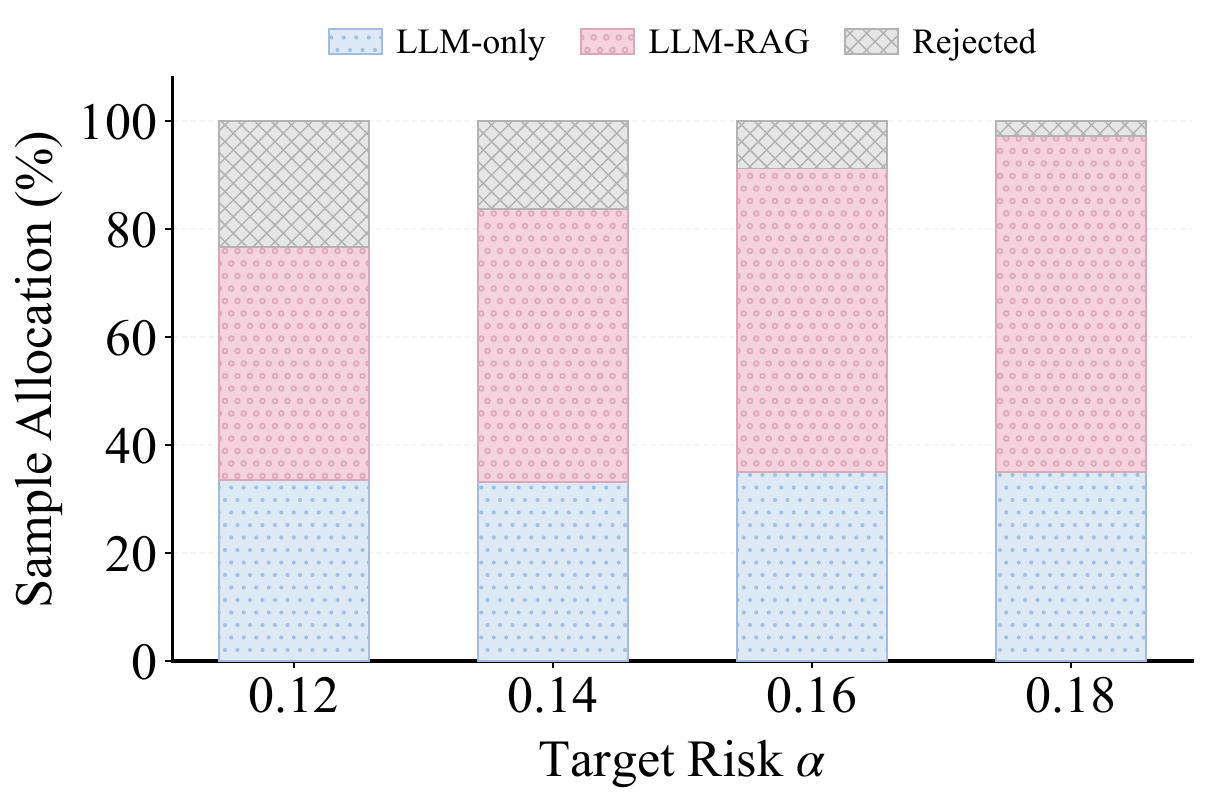}
    \caption{Vicuna-v1.5-7B}
  \end{subfigure}
  \hfill
  \begin{subfigure}{0.245\textwidth}   
    \centering
    \includegraphics[width=\linewidth]{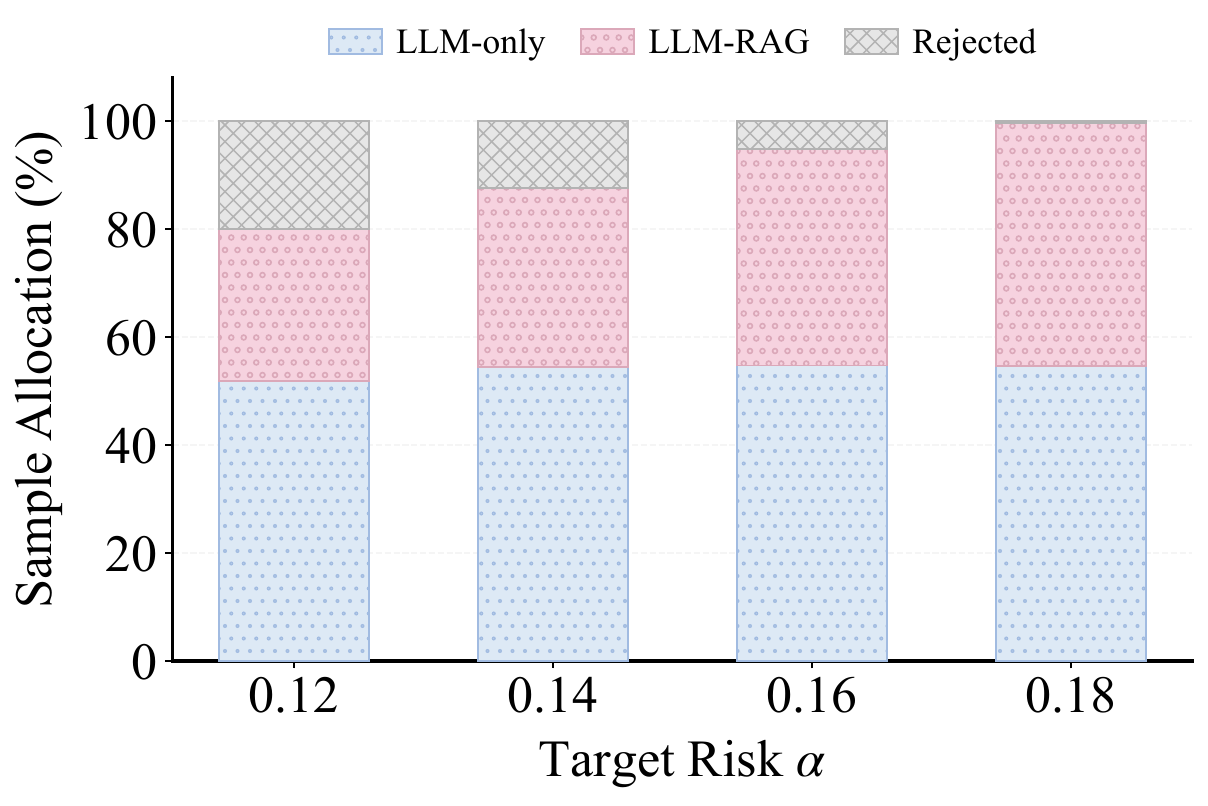}
    \caption{Vicuna-v1.5-13B}
  \end{subfigure}

  \caption{Routing allocation of test samples in BalanceRAG on TriviaQA (mean), including LLM-only acceptance, LLM-RAG fallback acceptance, and rejection.}
  \label{fig:proportion_triviaqa}
  \vspace{-5mm}
\end{figure*}

\noindent \textbf{Hyperparameters.}
Following standard split-calibration protocols~\citep{angelopoulos2025conformalriskcontrol}, we use a calibration-test split ratio of $0.5$ and repeat each experiment over $100$ random calibration-test splits. In each calibration set, we reserve $40\%$ of the examples for initial-node selection and use the remaining data for formal SGT calibration.

\subsection{Utility and Coverage under Risk Control}

\noindent \textbf{Reliability under Target Risk Control.}
We first examine whether BalanceRAG achieves system-level risk control at user-specified target levels. Figure~\ref{fig:risk_level_triviaqa} reports the empirical test-set selection-conditioned error rate across backbones and target risk levels, where solid curves denote the mean over repeated calibration/test splits and shaded regions show one standard deviation. BalanceRAG provides a finite-sample, high-probability guarantee that the population selection-conditioned risk of the calibrated router is at most the target level $\alpha$. The empirical test-set {Err.} closely tracks the target line, which is consistent with this guarantee. Figure~\ref{fig:fdr_violin} further shows the distribution of test-set {Err.} over repeated random splits. When $\delta=0.1$, the fraction of splits with test-set {Err.} exceeding $\alpha$ remains below $10\%$ at every target level. Moreover, joint calibration keeps BalanceRAG feasible under strict risk targets, whereas branch-wise baselines such as LLM-only and UCB-Cascaded often fail to find a feasible operating point due to branch-level constraints and stage-wise $\delta$ splitting.

\noindent \textbf{Coverage and Utility under Risk Control.}
We next evaluate how much useful output is retained at the same risk level $\alpha$. Figure~\ref{fig:risk_level_triviaqa} and Table~\ref{tab:three_datasets_results_merged} show that BalanceRAG consistently achieves the best or near-best coverage and utility among the baselines. It outperforms non-RAG baselines and matches or exceeds RAG in many settings. 
Notably, BalanceRAG accepts more test samples and retains more correctly accepted predictions ({Corr.}), while maintaining empirical {Err.} near or below the risk level. This advantage reflects the heterogeneous value of retrieval: retrieval can help knowledge-intensive cases, but may also introduce noisy or distracting context for easier queries. By jointly calibrating the LLM-only and fallback RAG branches, BalanceRAG keeps cases where the LLM-only answer is sufficient on the primary path and routes harder cases to LLM-RAG when the fallback is sufficiently reliable. Appendix~\ref{tab:llama31_8b_risk_control_comparison} further compares BalanceRAG with additional baselines under the same target risk levels. BalanceRAG is more reliable than Adaptive-RAG, Self-Route, and empirical risk selection, while achieving higher coverage than Bonferroni. Overall, these results show that joint calibration improves the reliability-coverage-utility trade-off under fixed risk control.

\begin{figure}[htbp]
  \centering

  \begin{subfigure}{0.23\textwidth}
    \centering
    \includegraphics[width=\linewidth]{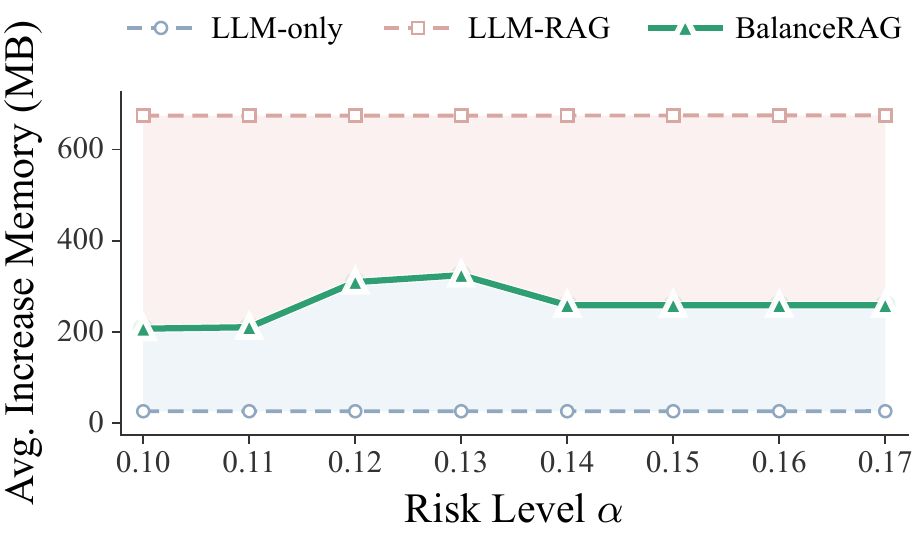}
    \caption{Qwen2.5-7B Memory}
  \end{subfigure}
  \hfill
  \begin{subfigure}{0.23\textwidth}
    \centering
    \includegraphics[width=\linewidth]{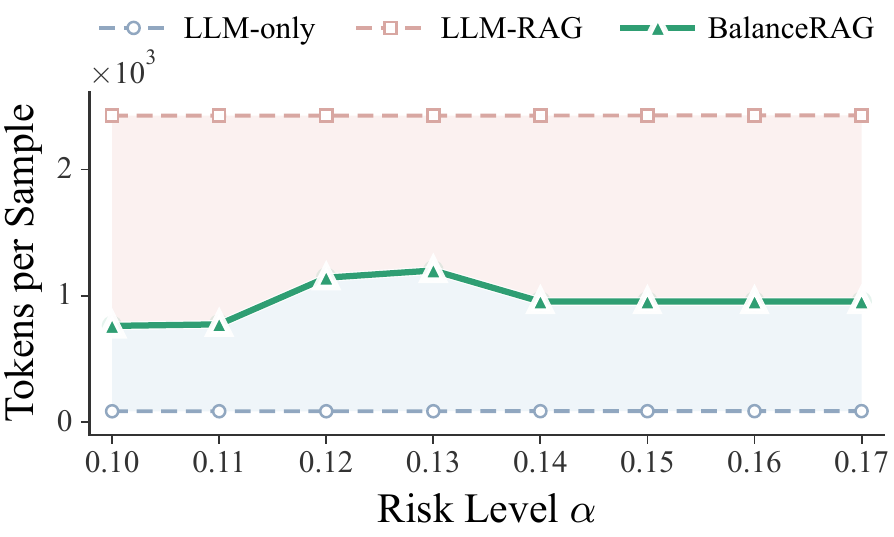}
    \caption{Qwen2.5-7B Tokens}
  \end{subfigure}

  \vspace{0.6em}

  \begin{subfigure}{0.23\textwidth}
    \centering
    \includegraphics[width=\linewidth]{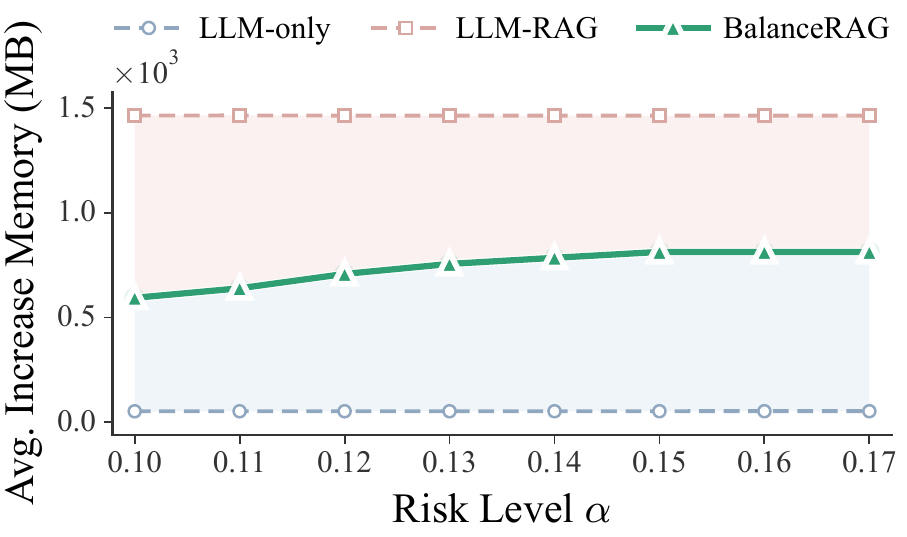}
    \caption{Llama-3.1-8B Memory}
  \end{subfigure}
  \hfill
  \begin{subfigure}{0.23\textwidth}
    \centering
    \includegraphics[width=\linewidth]{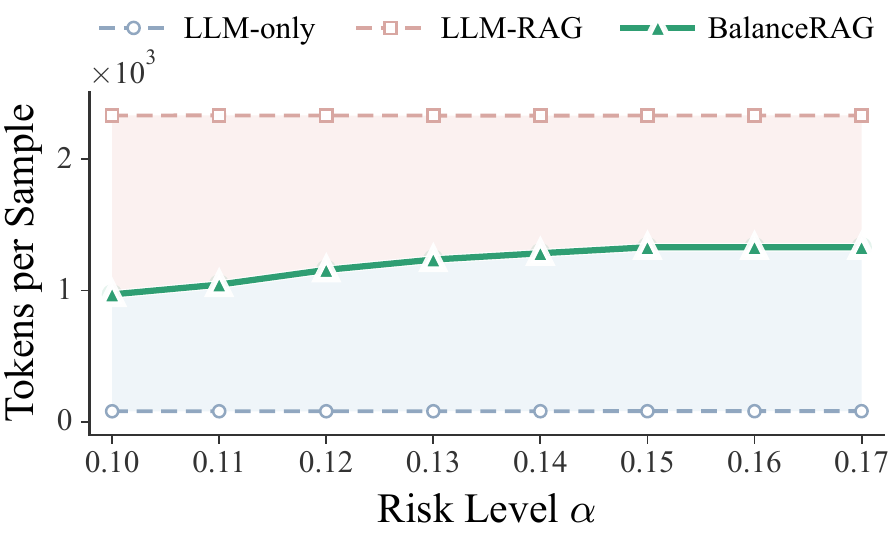}
    \caption{Llama-3.1-8B Tokens}
  \end{subfigure}

  \caption{Memory and token cost on TriviaQA across risk levels for Qwen2.5-7B and Llama-3.1-8B}
  \label{fig:cost_triviaqa}
  \vspace{-5mm}
\end{figure}

\subsection{Routing Allocation and Cost Efficiency}

\noindent \textbf{Routing Allocation across Models.}
Figure~\ref{fig:proportion_triviaqa} shows the routing allocation of BalanceRAG across eight backbones. BalanceRAG consistently keeps a substantial fraction of queries on the LLM-only branch, thereby reducing unnecessary fallback usage while maintaining the prescribed risk control. The LLM-only share also reflects the strength of the backbone: stronger models can safely answer more queries without retrieval, whereas weaker models require more frequent fallback to satisfy the same risk target. In our results, stronger backbones keep more than $60\%$ of accepted queries on the LLM-only branch, medium-strength backbones stay around $40$-$50\%$, and weaker backbones remain around $25$-$30\%$.

\noindent \textbf{Routed Execution Cost.}
We further evaluate the routed execution cost of BalanceRAG at calibrated operating points. Since BalanceRAG does not depend on a specific uncertainty estimator, we measure cost after uncertainty scores are obtained; this isolates the execution cost induced by the calibrated routing policy rather than the overhead of a particular uncertainty-estimation method. Under the same backbone and decoding setup, LLM-only and RAG perform fixed single-path inference, whereas BalanceRAG first attempts the LLM-only branch and invokes the RAG fallback only when the primary answer does not pass its calibrated threshold. We report the average memory increase and average token usage per sample. As shown in Figure~\ref{fig:cost_triviaqa}, BalanceRAG lies between the two fixed-path endpoints: it incurs more cost than LLM-only, but substantially lower cost than always-on RAG, because many queries can be accepted directly by the primary branch without invoking retrieval.

\subsection{Multi-Risk Control}
\label{sec:multi-risk}
\begin{table}[t]
\centering
\caption{Single-risk vs.\ multi-risk routing at selection-conditioned risk target
$\alpha_1=0.15$. $\alpha_2$ caps the fallback rate to LLM-RAG;
``Route'' shows the LLM-only/LLM-RAG split.}
\label{tab:single_vs_multi_routing}
\renewcommand{\arraystretch}{1.18}
\resizebox{\columnwidth}{!}{%
\begin{tabular}{llccccc}
\toprule
Model & Method & $\alpha_2$ & Route (Only/RAG) & Cov. & Corr. & Err. \\
\midrule

\multirow{4}{*}{Qwen2.5-3B}
& LLM-only & - & 1.0000 / 0.0000 & 0.4432 & 1137 & 0.1341 \\
& Single-risk & -  & 0.2995 / 0.6505 & 0.9499 & 2421 & 0.1400 \\
& Multi-risk  & 0.6 & 0.3464 / \textbf{0.5620} & 0.9084 & 2332 & 0.1338 \\
& Multi-risk  & 0.5 & 0.3859 / \textbf{0.4709} & 0.8568 & 2206 & 0.1313  \\
\midrule

\multirow{4}{*}{Qwen2.5-7B}
& LLM-only & - & 1.0000 / 0.0000 & 0.9105 & 2422 & 0.1412 \\
& Single-risk & -  & 0.6414 / 0.3582 & 0.9996 & 2738 & 0.1159 \\
& Multi-risk  & 0.3 & 0.7559 / \textbf{0.2438} & 0.9996 & 2719 & 0.1221 \\
& Multi-risk  & 0.2 & 0.8257 / \textbf{0.1729} & 0.9986 & 2692 & 0.1298 \\
\midrule

\multirow{4}{*}{Llama3.1-8B}
& LLM-only & - & 1.0000 / 0.0000 & 0.8118 & 2218 & 0.1415 \\
& Single-risk & -  & 0.5935 / 0.4024 & 0.9959 & 2727 & 0.1396 \\
& Multi-risk  & 0.3 & 0.7200 / \textbf{0.2593} & 0.9793 & 2689 & 0.1372 \\
& Multi-risk  & 0.2 & 0.7748 / \textbf{0.1702} & 0.9450 & 2599 & 0.1357 \\
\midrule

\multirow{4}{*}{Llama3.2-3B}
& LLM-only & - & 1.0000 / 0.0000 & 0.5092 & 1371 & 0.1373 \\
& Single-risk & -  & 0.4162 / 0.3539 & 0.7701 & 2073 & 0.1384 \\
& Multi-risk  & 0.3 & 0.4872 / \textbf{0.2504} & 0.7376 & 1989 & 0.1366 \\
& Multi-risk  & 0.2 & 0.5320 / \textbf{0.1279} & 0.6599 & 1781 & 0.1355 \\
\bottomrule
\end{tabular}%
}
\vspace{-2mm}
\end{table}

In deployment, practitioners may want to control not only the selection-conditioned answer risk but also the frequency of fallback retrieval. We therefore extend BalanceRAG to the multi-risk setting by adding a constraint $\alpha_2$ on the fallback invocation rate in addition to the answer-risk target $\alpha_1$. Table~\ref{tab:single_vs_multi_routing} shows that stricter $\alpha_2$ values lead BalanceRAG to certify operating points with lower fallback rates, shifting more queries to the cheaper LLM-only branch. This reduces retrieval usage, with only a moderate decrease in coverage and correctly accepted samples. At the same time, both risks remain controlled in all reported configurations: the empirical selection-conditioned risk stays below $\alpha_1=0.15$, and the fallback invocation rate stays below $\alpha_2$. These results show that multi-risk BalanceRAG provides an explicit mechanism for navigating the reliability-utility-retrieval trade-off.

\section{Conclusion}
\label{sec:conclusion}
In this paper, we develop BalanceRAG, a training-free framework for joint threshold calibration in risk-controlled cascaded LLM-RAG systems. By certifying candidate threshold pairs on a two-dimensional lattice with SGT, BalanceRAG selects high-utility operating points while providing finite-sample, high-probability control of the selection-conditioned error rate. The framework also supports multi-risk calibration to bound fallback retrieval usage together with answer risk. Across open-domain QA benchmarks and diverse backbones, BalanceRAG meets prescribed risk levels, preserves high coverage and accepted correct answers, and reduces unnecessary retrieval calls. These results suggest a principled path from heuristic routing to reliable cascaded LLM-RAG. 
\section*{Limitations}

BalanceRAG relies on the standard split-calibration assumption that the calibration and deployment examples are drawn from the same distribution. If the query distribution, retrieval corpus, or backbone model changes substantially, the calibrated thresholds may need to be updated. This is a common limitation of finite-sample calibration methods and can be addressed in practice through periodic recalibration.

The guarantee is also defined with respect to the chosen correctness criterion. Although we evaluate semantic similarity, bidirectional entailment, and LLM-as-a-Judge variants, open-domain QA correctness can still be ambiguous, especially for partially correct or underspecified answers. This reflects a broader challenge in evaluating free-form generation rather than a limitation specific to BalanceRAG.

Finally, our experiments focus on two-branch LLM-RAG cascades. Extending the calibration principle to multi-stage retrieval, tool-augmented systems, or agentic workflows is a promising direction, but may require more scalable search and calibration over larger routing spaces.
\bibliography{custom}

\clearpage\onecolumn
\appendix

\section{Proofs}
\label{app:proofs}

This appendix provides the proofs deferred from Section~\ref{sec:method}.
Throughout, $\{(x_\ell,y^\star_\ell\}_{\ell=1}^n$ is the i.i.d.\ calibration
set, and for any threshold pair $\lambda=(t_1,t_2)$ we use the notation
$M(\lambda)$, $K(\lambda)$, $R(\lambda)$ from
Section~\ref{subsec:problem_formulation}. We also write $F_{m,q}$ for the
CDF of the binomial distribution $\mathrm{Bin}(m,q)$.

\subsection{Super-uniformity of the binomial \texorpdfstring{$p$}{p}-value}
\label{app:proof_pvalue}

\begin{lemma}[Super-uniformity]
\label{lem:pvalue}
For each lattice node $\lambda_{i,j}$, let
\begin{equation}
\label{eq:pvalue_app}
p_{i,j} \;=\; F_{M(\lambda_{i,j}),\,\alpha}\!\bigl(K(\lambda_{i,j})\bigr).
\end{equation}
Under the null hypothesis $H_{\lambda_{i,j}}: R(\lambda_{i,j})\ge\alpha$,
$p_{i,j}$ is super-uniform, i.e.\ for every $t\in[0,1]$,
\begin{equation}
\mathbb{P}_{H_{\lambda_{i,j}}}\!\bigl(p_{i,j}\le t\bigr)\;\le\;t.
\end{equation}
\end{lemma}

\begin{proof}
Fix $\lambda=\lambda_{i,j}$ and abbreviate $M=M(\lambda)$, $K=K(\lambda)$,
$R=R(\lambda)$. Conditionally on $M$, the samples accepted by $\lambda$
are i.i.d.\ Bernoulli with error probability $R$, hence
\begin{equation}
\label{eq:bin_cond}
K \mid M \;\sim\; \mathrm{Bin}(M,R).
\end{equation}

Under $H_{\lambda}$ we have $R\ge\alpha$, which implies the stochastic
ordering $\mathrm{Bin}(M,R)\succeq_{\mathrm{st}}\mathrm{Bin}(M,\alpha)$;
equivalently, for every integer $k$,
\begin{equation}
\label{eq:stoch_dom}
F_{M,R}(k) \;\le\; F_{M,\alpha}(k).
\end{equation}

By the generalized probability integral transform applied to the
(possibly discrete) CDF $F_{M,R}$, for every $t\in[0,1]$,
\begin{equation}
\label{eq:pit}
\mathbb{P}\!\bigl(F_{M,R}(K)\le t\mid M\bigr) \;\le\; t.
\end{equation}

Combining Eq.~\ref{eq:stoch_dom} and Eq.~\ref{eq:pit} gives
\begin{align}
\mathbb{P}(p_{i,j}\le t \mid M)
&= \mathbb{P}\!\bigl(F_{M,\alpha}(K)\le t \mid M\bigr) \nonumber\\
&\le \mathbb{P}\!\bigl(F_{M,R}(K)\le t \mid M\bigr) \nonumber\\
&\le t.
\label{eq:pvalue_cond}
\end{align}
Marginalizing over $M$ yields
$\mathbb{P}_{H_{\lambda_{i,j}}}(p_{i,j}\le t)\le t$.
\end{proof}

\subsection{Shell-equal budget property of DWD}
\label{app:proof_shell}

We record here a structural property of DWD that motivates its design
(cf.\ Eq.~\ref{eq:shell_equal_budget} in the main text). The result is
self-contained and does not depend on any hypothesis-testing content;
while it is not used in the proof of Theorem~\ref{thm:main_guarantee},
it provides useful intuition for the behavior of DWD on the lattice.

\begin{lemma}[Shell-equal budget]
\label{lem:shell}
Let the lattice $\Lambda$ have root $(0,0)$ with initial budget
$\delta_{0,0}=\delta$ and all other budgets set to zero. If the edge
weights follow the DWD rule in Eq.~\ref{eq:diamond_weights} and each
node forwards its budget to neighbors once reached, then every node
$(i,j)$ with $i+j=d$ accumulates
\begin{equation}
\label{eq:shell_budget_app}
\delta^{\mathrm{DWD}}_{i,j} \;=\; \frac{\delta}{d+1}.
\end{equation}
\end{lemma}

\begin{proof}
We prove Eq.~\ref{eq:shell_budget_app} by induction on $d$.

\textbf{Base case ($d=0$).}
The only node is $(0,0)$ itself, which trivially receives
$\delta^{\mathrm{DWD}}_{0,0}=\delta=\delta/(0+1)$.

\textbf{Inductive step.}
Assume the claim holds on diagonal $d$. Let $(a,b)$ be any node on
diagonal $d+1$, so $a+b=d+1$. Budget reaches $(a,b)$ only from its
in-lattice parents $(a-1,b)$ and $(a,b-1)$. Using
Eq.~\ref{eq:diamond_weights},
\begin{align}
g^{\mathrm{DWD}}_{(a-1,b)\to(a,b)}
&= \frac{a}{(a-1)+b+2} = \frac{a}{d+2}, \label{eq:wt_left}\\
g^{\mathrm{DWD}}_{(a,b-1)\to(a,b)}
&= \frac{b}{a+(b-1)+2} = \frac{b}{d+2}. \label{eq:wt_down}
\end{align}
By the inductive hypothesis, both parents carry budget $\delta/(d+1)$.
Summing the incoming contributions,
\begin{align}
\delta^{\mathrm{DWD}}_{a,b}
&= \frac{\delta}{d+1}\cdot\frac{a}{d+2}
 + \frac{\delta}{d+1}\cdot\frac{b}{d+2}\nonumber\\
&= \frac{\delta}{d+1}\cdot\frac{a+b}{d+2}\nonumber\\
&= \frac{\delta}{d+1}\cdot\frac{d+1}{d+2}
 = \frac{\delta}{d+2},
\label{eq:shell_step}
\end{align}
where we used $a+b=d+1$ in the penultimate step.

For boundary nodes with only one in-lattice parent ($a=0$ or $b=0$),
the missing term in Eq.~\ref{eq:shell_step} has numerator zero, and
the ``sole neighbor receives weight $1$'' convention makes the lone
parent contribute $\delta/(d+1)\cdot 1=\delta/(d+2)$ after the same
algebra. Hence every node on diagonal $d+1$ accumulates $\delta/(d+2)$,
completing the induction.
\end{proof}

\begin{remark}
Lemma~\ref{lem:shell} implies that the total budget on each diagonal
is conserved: summing $\delta/(d+1)$ over the $d+1$ nodes of diagonal
$d$ yields $\delta$.
\end{remark}

\subsection{Proof of Theorem~\ref{thm:main_guarantee}}
\label{app:proof_main}

\begin{proof}[Proof of Theorem~\ref{thm:main_guarantee}]
The proof combines two ingredients: the validity of the binomial
$p$-value (Lemma~\ref{lem:pvalue}) and the FWER control of SGT.

\textbf{(i) Valid $p$-values.}
By Lemma~\ref{lem:pvalue}, the $p$-value $p_{i,j}$ defined in
Eq.~\ref{eq:pvalue} is super-uniform under the null
$H_{\lambda_{i,j}}$ for every lattice node.

\textbf{(ii) FWER control by SGT.}
Both the UGD weights (Eq.~\ref{eq:diffusion_weights}) and the DWD
weights (Eq.~\ref{eq:diamond_weights}) are non-negative, satisfy
$g_{(i,j)\to(i,j)}=0$, and obey
$\sum_{u\in\mathcal{N}(i,j)} g_{(i,j)\to u}\le 1$ at every node.
Hence, the proposed graph satisfies the regularity conditions of
\citet{https://doi.org/10.1002/sim.3495}. Combined with the initial budget allocation
$\sum_{(i,j)} \delta_{i,j}\le \delta$ and the super-uniform $p$-values
from (i), Algorithm~1 of \citet{https://doi.org/10.1002/sim.3495} yields a sequentially
rejective multiple testing procedure that strongly controls the
family-wise error rate at level $\delta$. Therefore, with the null
hypotheses defined as $H_\lambda: R(\lambda)>\alpha$, the certified set
$\widehat{\Lambda}$ returned by the procedure satisfies
\begin{equation}
\label{eq:fwer_in_proof}
\mathbb{P}\!\Bigl(\sup_{\lambda\in\widehat{\Lambda}} R(\lambda)\le\alpha\Bigr)
\;\ge\;1-\delta,
\end{equation}
which recovers Eq.~\ref{eq:fwer_guarantee}.

\textbf{(iii) From $\widehat{\Lambda}$ to $\hat{\lambda}$.}
Let
$\mathcal{E}:=\{\sup_{\lambda\in\widehat{\Lambda}} R(\lambda)\le\alpha\}$;
by (ii), $\mathbb{P}(\mathcal{E})\ge 1-\delta$. On $\mathcal{E}$, the
selection rule in Eq.~\ref{eq:select_final} picks
$\hat{\lambda}\in\widehat{\Lambda}$, so
\begin{equation}
R(\hat{\lambda})
\;\le\;
\sup_{\lambda\in\widehat{\Lambda}} R(\lambda)
\;\le\;
\alpha.
\end{equation}
Hence $\mathcal{E}\subseteq\{R(\hat{\lambda})\le\alpha\}$, and
\begin{equation}
\mathbb{P}\!\bigl(R(\hat{\lambda})\le\alpha\bigr)
\;\ge\;
\mathbb{P}(\mathcal{E})
\;\ge\;
1-\delta.\qedhere
\end{equation}
\end{proof}

\begin{remark}
The proof does \emph{not} rely on the specific form of the selection
rule in Eq.~\ref{eq:select_final}; any measurable rule that picks
$\hat{\lambda}\in\widehat{\Lambda}$ preserves the guarantee.
\end{remark}

\newpage
\section{Details of Experimental Setup}
\paragraph{Details of Datasets and Models}
We evaluate BalanceRAG on three open-domain QA datasets, TriviaQA, SQuAD~v2, and Natural Questions (NQ). All datasets are processed into a unified question-answer evaluation format, where each example contains a question, an answer, an id, and the corresponding prompt. The preprocessing stage filters out examples with empty questions, empty answers, non-ASCII questions or answers, or inputs that exceed the maximum context length of the model. The default maximum numbers of retained examples are 7000 for TriviaQA, 5000 for SQuAD~v2, and 2000 for NQ. Since the cleaning results after generation may differ across models, the actual number of examples used in the experiments can vary slightly. We therefore report the final number of evaluated examples as $N$ in Table~\ref{tab:branch_base_error_single}. We use eight backbone LLMs, including Qwen2.5-3B-Instruct, Qwen2.5-7B-Instruct, Qwen3-4B, Qwen3-14B, Llama3.2-3B-Instruct, Llama3.1-8B-Instruct, Vicuna-v1.5-7B, and Vicuna-v1.5-13B. Table~\ref{tab:branch_base_error_single} reports the base error rates of the two branches before calibration, where Only denotes the error rate of the LLM-only branch, RAG denotes the error rate of the LLM-RAG branch, and $\Delta=Only-RAG$ denotes the change in error rate brought by retrieval.

\begin{table}[!htbp]
\vspace{-2mm}
\centering
\small
\setlength{\tabcolsep}{6pt}
\renewcommand{\arraystretch}{1.08}
\caption{Branch-level base error rates.
$N$ denotes the number of examples.
Only and RAG are the pre-calibration error rates of the
LLM-only and LLM-RAG branches, and
$\Delta=Only-RAG$ indicates the retrieval gain.
Blue cells highlight the largest retrieval gain within each dataset.}
\label{tab:branch_base_error_single}
\begin{tabular}{lcccc@{\hspace{16pt}}lcccc}
\toprule
\rowcolor{HeaderGray}
Model & $N$ & Only $\downarrow$ & RAG $\downarrow$ & $\Delta \uparrow$
& Model & $N$ & Only $\downarrow$ & RAG $\downarrow$ & $\Delta \uparrow$ \\
\midrule

\rowcolor{GroupBlue}
\multicolumn{10}{l}{\textbf{TriviaQA}} \\
Qwen2.5-3B  & 5929 & 0.2739 & 0.1558 & \cellcolor{GainBlue}\textbf{0.1181}
& Llama3.2-3B & 6248 & 0.2769 & 0.2130 & 0.0639 \\
Qwen2.5-7B  & 6196 & 0.1699 & 0.1298 & 0.0401
& Llama3.1-8B & 6365 & 0.2016 & 0.1475 & 0.0541 \\
Qwen3-4B    & 6510 & 0.2167 & 0.1246 & 0.0921
& Vicuna-7B   & 4183 & 0.2417 & 0.1831 & 0.0586 \\
Qwen3-14B   & 6063 & 0.1921 & 0.1070 & 0.0851
& Vicuna-13B  & 3036 & 0.1989 & 0.1607 & 0.0382 \\

\addlinespace[3pt]

\rowcolor{GroupBlue}
\multicolumn{10}{l}{\textbf{SQuAD}} \\
Qwen2.5-3B  & 4475 & 0.2291 & 0.1321 & 0.0970
& Llama3.2-3B & 4482 & 0.3135 & 0.1772 & 0.1363 \\
Qwen2.5-7B  & 4375 & 0.2027 & 0.1191 & 0.0836
& Llama3.1-8B & 4210 & 0.3192 & 0.1202 & \cellcolor{GainBlue}\textbf{0.1990} \\
Qwen3-4B    & 4557 & 0.2203 & 0.0898 & 0.1305
& Vicuna-7B   & 3461 & 0.3389 & 0.1410 & 0.1979 \\
Qwen3-14B   & 4360 & 0.2122 & 0.0954 & 0.1168
& Vicuna-13B  & 3031 & 0.2748 & 0.1488 & 0.1260 \\

\addlinespace[3pt]

\rowcolor{GroupBlue}
\multicolumn{10}{l}{\textbf{NQ}} \\
Qwen2.5-3B  & 1732 & 0.3816 & 0.1819 & 0.1997
& Llama3.2-3B & 1808 & 0.3711 & 0.2163 & 0.1548 \\
Qwen2.5-7B  & 1755 & 0.3419 & 0.1744 & 0.1675
& Llama3.1-8B & 1792 & 0.3203 & 0.1479 & 0.1724 \\
Qwen3-4B    & 1858 & 0.3617 & 0.1335 & 0.2282
& Vicuna-7B   & 1599 & 0.4953 & 0.1932 & \cellcolor{GainBlue}\textbf{0.3021} \\
Qwen3-14B   & 1807 & 0.3298 & 0.1384 & 0.1914
& Vicuna-13B  & 1196 & 0.3972 & 0.1898 & 0.2074 \\

\bottomrule
\end{tabular}
\vspace{-2mm}
\end{table}
\paragraph{Details of Corpus Construction}
For each dataset, we build a dataset-level retrieval corpus from the raw evidence fields. Each document consists of a title and a passage. TriviaQA uses evidence from Wikipedia entity pages and web search results, SQuAD~v2 uses the article title and context passage, and Natural Questions uses the preprocessed long-answer passage or the majority-voted long-answer candidate. We remove documents with missing titles or passages, strip extra whitespace, and deduplicate exact title--passage pairs across the corpus. The resulting corpus is cached and reused for dense retrieval.

\paragraph{Details of Answer Generation and Retrieval}
All datasets use a one-shot prompt. The LLM-only branch provides only the question itself and one example question-answer pair. The LLM-RAG branch prepends retrieved evidence to the same question-answer format and asks the model to answer based on the evidence. The most likely answer is generated with beam search using a beam size of 5, without sampling. The sampled generations for uncertainty estimation are generated with sampling, where each input has 20 sampled answers, the temperature is 1.0, and top-$p$ is 0.9. The default maximum generation length is 30 tokens. The RAG branch uses dense retrieval with all-MiniLM-L6-v2 as the retriever encoder. During retrieval, the document title and passage are concatenated and encoded, the embeddings are L2-normalized, and the documents are ranked by dot product. For each question, we use the top-4 retrieved documents by default, and each passage is truncated to at most 1200 characters.

\paragraph{Details of Correctness Criteria}
By default, we use semantic similarity as the correctness criterion. Specifically, we compute the similarity between the ground-truth answer and the most likely generation using a sentence-transformers cross-encoder, and set 0.7 as the default threshold. An example is treated as correct when its similarity score is no lower than this threshold. This correctness label is used to construct the error indicator in the calibration risk. In addition to the default setting, we also use entailment and LLM-as-a-Judge in the robustness experiments. The entailment setting uses DeBERTa-v3-large-mnli-fever-anli-ling-wanli for NLI-based judgment. The LLM-as-a-Judge setting uses Qwen3-14B as the judge and formulates the judgment as a three-way classification problem, where both correct and partial are treated as acceptable correct answers.

\begingroup
\setlength{\floatsep}{0pt}
\setlength{\textfloatsep}{0pt}
\setlength{\intextsep}{5pt}
\setlength{\abovecaptionskip}{1pt}
\setlength{\belowcaptionskip}{0pt}
\begin{algorithm}[H]
\small
\caption{Single-branch calibration for LLM-only and LLM-RAG}
\label{alg:single_branch}
\begin{algorithmic}[1]
\Require 
$\mathcal{D}_{\text{cal}}=\{(x_\ell,u_b(x_\ell),c_b(x_\ell))\}_{\ell=1}^n$,
where $u_b$ is uncertainty and $c_b=1$ denotes correctness;
branch $b\in\{\text{LLM},\text{RAG}\}$; target risk $\alpha$; error level $\delta$
\Ensure  threshold $\hat{t}$
\State $\mathcal{T} \gets$ sorted unique values of $\{u_b(x_\ell)\}_{\ell=1}^{n}$
\State $\hat{t} \gets \varnothing$
\For{$t \in \mathcal{T}$ in ascending order}
    \State $m \gets |\{\ell : u_b(x_\ell) \le t\}|$;\;
           $w \gets |\{\ell : u_b(x_\ell) \le t,\; c_b(x_\ell)=0\}|$
    \If{$m = 0$} \textbf{continue} \EndIf
    \State $p(t) \gets \mathbb{P}\bigl(\mathrm{Bin}(m,\alpha)\le w\bigr)$
    \If{$p(t) \le \delta$}
        \State $\hat{t} \gets t$ \Comment{keep scanning for a looser threshold}
    \Else
        \State \textbf{break} \Comment{monotone stopping}
    \EndIf
\EndFor
\State \Return $\hat{t}$
\end{algorithmic}
\end{algorithm}
\begin{algorithm}[H]
\small
\caption{UCB-Cascaded calibration}
\label{alg:ucb_cascaded}
\begin{algorithmic}[1]
\Require 
$\mathcal{D}_{\text{cal}}=\{(x_\ell,u_1(x_\ell),u_2(x_\ell),c_1(x_\ell),c_2(x_\ell))\}_{\ell=1}^n$;
target risk $\alpha$; error level $\delta$; bound type
$B \in \{\text{CP}, \text{HFD}\}$
\Ensure  thresholds $(\hat{t}_1, \hat{t}_2)$
\State Split the confidence budget: $\delta_{\text{stage}} \gets \delta / 2$
\Statex
\Statex \textit{\# Stage 1: calibrate $t_1$ on the full calibration set}
\State $\mathcal{T}_1 \gets$ sorted unique values of $\{u_1(x_\ell)\}_{\ell=1}^n$
\State $\hat{t}_1 \gets -\infty$
\For{$t \in \mathcal{T}_1$ in ascending order}
    \State $m \gets |\{\ell : u_1(x_\ell) \le t\}|$;\;
           $w \gets |\{\ell : u_1(x_\ell) \le t,\; c_1(x_\ell)=0\}|$
    \State $\bar{R}(t) \gets \textsc{UpperBound}_B(m, w;\, \delta_{\text{stage}})$
    \If{$\bar{R}(t) \le \alpha$}
        \State $\hat{t}_1 \gets t$ \Comment{largest feasible so far}
    \Else
        \State \textbf{break}
    \EndIf
\EndFor
\Statex
\Statex \textit{\# Stage 2: calibrate $t_2$ on the residual set}
\State $\mathcal{R} \gets \{\ell : u_1(x_\ell) > \hat{t}_1\}$ \Comment{deferred samples}
\State $\mathcal{T}_2 \gets$ sorted unique values of $\{u_2(x_\ell) : \ell \in \mathcal{R}\}$
\State $\hat{t}_2 \gets -\infty$
\For{$t \in \mathcal{T}_2$ in ascending order}
    \State $m \gets |\{\ell \in \mathcal{R} : u_2(x_\ell) \le t\}|$;\;
           $w \gets |\{\ell \in \mathcal{R} : u_2(x_\ell) \le t,\; c_2(x_\ell)=0\}|$
    \State $\bar{R}(t) \gets \textsc{UpperBound}_B(m, w;\, \delta_{\text{stage}})$
    \If{$\bar{R}(t) \le \alpha$}
        \State $\hat{t}_2 \gets t$
    \Else
        \State \textbf{break}
    \EndIf
\EndFor
\State \Return $(\hat{t}_1, \hat{t}_2)$
\end{algorithmic}
\end{algorithm}

\endgroup
\paragraph{Details of Uncertainty Estimators}
The default uncertainty estimator is semantic entropy based on sampled answer clusters, denoted as SE. For each input, we sample 20 answers to compute the uncertainty score. In addition to SE, we also evaluate EigV, Deg, Ecc, and SELF. SE computes the entropy of the semantic-cluster distribution over sampled generations. EigV, Deg, and Ecc construct a graph from the pairwise similarity matrix of sampled generations and derive uncertainty scores from graph-structure statistics. SELF uses the length-normalized sentence entropy of the most likely generation.

\paragraph{Details of Baselines}
\label{app:baseline_details}
We provide three algorithms to clarify the calibration procedures of different baselines. Algorithm~\ref{alg:single_branch} describes single-branch calibration, which is used for LLM-only and LLM-RAG. These methods use the uncertainty score from a single branch and search over a one-dimensional threshold space for a feasible point that satisfies the target risk. Algorithm~\ref{alg:ucb_cascaded} describes UCB-Cascaded. This method first calibrates the LLM-only threshold on the full calibration set, then calibrates the RAG threshold on the residual subset not accepted by the LLM-only branch. It splits $\delta$ equally between the two stages, with each stage using $\delta/2$. UCB-Cascaded includes two variants based on the Clopper--Pearson upper bound and the Hoeffding upper bound. Algorithm~\ref{alg:balancerag} describes BalanceRAG, which jointly calibrates the LLM-only and RAG branches on a two-dimensional threshold lattice and uses sequential graphical testing for high-probability risk control. We report two graph diffusion priors, UGD and DWD, where DWD is the default main method. Additional baselines in the appendix include Adaptive RAG, Self-Route, Empirical Risk $<\alpha$, and Bonferroni correction. Adaptive RAG routes queries using question-complexity labels. Self-Route routes queries using route-judge labels. Empirical Risk $<\alpha$ selects the threshold pair with the largest number of accepted examples within the empirically feasible region. Bonferroni correction applies a conservative adjustment by dividing $\delta$ by the number of tested nodes.

\begingroup
\setlength{\floatsep}{0pt}
\setlength{\textfloatsep}{0pt}
\setlength{\intextsep}{5pt}
\setlength{\abovecaptionskip}{1pt}
\setlength{\belowcaptionskip}{0pt}

\begin{algorithm}[htbp]
\small

\caption{BalanceRAG calibration via SGT}
\label{alg:balancerag}
\begin{algorithmic}[1]
\Require $\mathcal{D}_{\text{cal}}$; target risk $\alpha$;
        error level $\delta$; graph prior $\mathcal{P} \in \{\textsc{UGD},\textsc{DWD}\}$
\Ensure  thresholds $(\hat{t}_1, \hat{t}_2)$
\Statex
\Statex \textit{\# Step 1: choose the initial node}
\State Discretize $u_1, u_2$ into grids of size $M \times N$, yielding
       lattice $\Lambda = \{\lambda_{i,j}\}$
\State Split $\mathcal{D}_{\text{cal}}$ into a seed split
       $\mathcal{D}_{\text{seed}}$ and an SGT split $\mathcal{D}_{\text{sgt}}$
\For{each $\lambda_{i,j} \in \Lambda$}
    \State $M^{\mathrm{seed}}_{ij} \gets M_{\mathcal{D}_{\text{seed}}}(\lambda_{i,j})$;\;
           $K^{\mathrm{seed}}_{ij} \gets K_{\mathcal{D}_{\text{seed}}}(\lambda_{i,j})$
\EndFor
\State $\mathcal{S}_{\text{seed}} \gets
       \{\lambda_{i,j}: M^{\mathrm{seed}}_{ij}>0,\;
       K^{\mathrm{seed}}_{ij}/M^{\mathrm{seed}}_{ij}\le \alpha\}$
\If{$\mathcal{S}_{\text{seed}}\neq\varnothing$}
    \State $s \gets \argmax_{\lambda_{i,j}\in\mathcal{S}_{\text{seed}}}
           M^{\mathrm{seed}}_{ij}$
\Else
    \State $s \gets \operatorname*{arg\,min}_{\lambda_{i,j}:M^{\mathrm{seed}}_{ij}>0}
           \mathbb{P}\bigl(\mathrm{Bin}(M^{\mathrm{seed}}_{ij},\alpha)
           \le K^{\mathrm{seed}}_{ij}\bigr)$
\EndIf
\Statex
\Statex \textit{\# Step 2: compute node-wise $p$-values}
\For{each $\lambda_{i,j} \in \Lambda$}
    \State $M_{ij} \gets M_{\mathcal{D}_{\text{sgt}}}(\lambda_{i,j})$;\;
           $K_{ij} \gets K_{\mathcal{D}_{\text{sgt}}}(\lambda_{i,j})$
    \State $p_{ij} \gets \mathbb{P}\bigl(\mathrm{Bin}(M_{ij},\alpha)\le K_{ij}\bigr)$
           \Comment{Eq.~\ref{eq:pvalue}}
\EndFor
\Statex
\Statex \textit{\# Step 3: initialize budget and graph prior}
\State $\delta_s \gets \delta$;\quad $\delta_{i,j} \gets 0$ for all $\lambda_{i,j}\neq s$
\State $\widehat{\Lambda} \gets \varnothing$
       \Comment{set of certified thresholds}
\State Set edge weights $g_{(i,j)\to u}$ from $\mathcal{P}$ with root node $s$
       (Eq.~\ref{eq:diffusion_weights} or~\ref{eq:diamond_weights})
\Statex
\Statex \textit{\# Step 4: sequential graphical testing}
\While{$\exists\,\lambda_{i,j} \in \Lambda\setminus\widehat{\Lambda}$
        such that $p_{ij}\le\delta_{i,j}$}
    \State Choose any such $\lambda_{i,j}$ maximizing $M_{ij}$
    \State $\widehat{\Lambda} \gets \widehat{\Lambda} \cup \{\lambda_{i,j}\}$
           \Comment{certify $\lambda_{i,j}$}
    \For{each successor $u\in\mathcal{N}(i,j)$ with $\lambda_u\notin\widehat{\Lambda}$}
        \State $\delta_u \gets \delta_u + \delta_{i,j}\cdot g_{(i,j)\to u}$
           \Comment{propagate budget}
    \EndFor
    \State $\delta_{i,j} \gets 0$
\EndWhile
\Statex
\Statex \textit{\# Step 5: select the best certified threshold pair}
\State $\hat{\lambda} \gets
       \argmax_{\lambda_{i,j}\in\widehat{\Lambda}} M_{ij}$
       \Comment{Eq.~\ref{eq:select_final}}
\State \Return $\hat{\lambda}=(\hat{t}_1,\hat{t}_2)$
\end{algorithmic}
\end{algorithm}
\endgroup

\newtcolorbox{promptbox}[2][]{
    enhanced,
    breakable,
    width=0.92\linewidth,
    colback=gray!3,
    colframe=gray!55,
    coltitle=white,
    colbacktitle=gray!75,
    title={\small\bfseries #2},
    fonttitle=\small\bfseries,
    boxrule=0.6pt,
    arc=2pt,
    left=10pt,
    right=10pt,
    top=7pt,
    bottom=7pt,
    titlerule=0pt,
    boxed title style={
        sharp corners,
        boxrule=0pt,
        colback=gray!75,
        left=8pt,
        right=8pt,
        top=3pt,
        bottom=3pt
    },
    #1
}

\newcommand{\promptrole}[1]{%
    \vspace{4pt}
    \noindent\textbf{\texttt{\#\#\# #1:}}\par
}

\newcommand{\promptslot}[1]{\texttt{\{#1\}}}
\clearpage
\newpage
\begin{figure*}[!htbp]
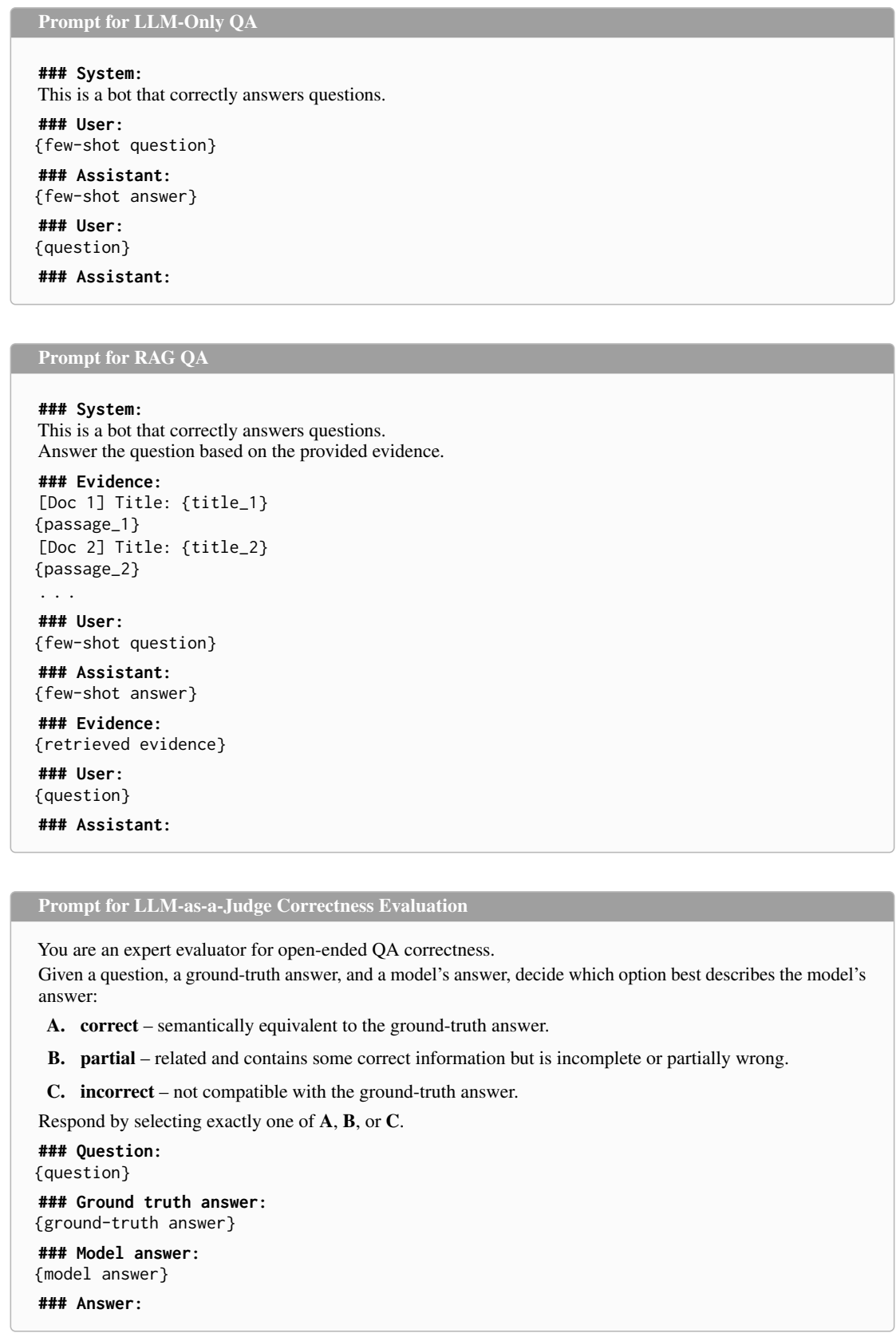

\centering
\captionsetup[subfigure]{skip=2pt}
\setlength{\abovecaptionskip}{3pt}
\setlength{\belowcaptionskip}{0pt}

\begin{promptbox}{Prompt for LLM-Only QA}
\footnotesize
\promptrole{System}
This is a bot that correctly answers questions.

\promptrole{User}
\promptslot{few-shot question}

\promptrole{Assistant}
\promptslot{few-shot answer}

\promptrole{User}
\promptslot{question}

\promptrole{Assistant}
\end{promptbox}

\vspace{0.4em}

\begin{promptbox}{Prompt for RAG QA}
\footnotesize
\promptrole{System}
This is a bot that correctly answers questions.\\
Answer the question based on the provided evidence.

\promptrole{Evidence}
\texttt{[Doc 1] Title: \promptslot{title\_1}}\\
\promptslot{passage\_1}

\vspace{1pt}
\texttt{[Doc 2] Title: \promptslot{title\_2}}\\
\promptslot{passage\_2}

\vspace{1pt}
\texttt{\ldots}

\promptrole{User}
\promptslot{few-shot question}

\promptrole{Assistant}
\promptslot{few-shot answer}

\promptrole{Evidence}
\promptslot{retrieved evidence}

\promptrole{User}
\promptslot{question}

\promptrole{Assistant}
\end{promptbox}

\vspace{0.4em}

\begin{promptbox}{Prompt for LLM-as-a-Judge Correctness Evaluation}
\footnotesize
You are an expert evaluator for open-ended QA correctness.

\vspace{2pt}
Given a question, a ground-truth answer, and a model's answer, decide which option best describes the model's answer:

\vspace{2pt}
\begin{enumerate}[
    label=\textbf{\Alph*.},
    leftmargin=2.2em,
    labelsep=0.8em,
    itemsep=1pt,
    topsep=1pt
]
    \item \textbf{correct} -- semantically equivalent to the ground-truth answer.
    \item \textbf{partial} -- related and contains some correct information but is incomplete or partially wrong.
    \item \textbf{incorrect} -- not compatible with the ground-truth answer.
\end{enumerate}

\vspace{2pt}
Respond by selecting exactly one of \textbf{A}, \textbf{B}, or \textbf{C}.

\promptrole{Question}
\promptslot{question}

\promptrole{Ground truth answer}
\promptslot{ground-truth answer}

\promptrole{Model answer}
\promptslot{model answer}

\promptrole{Answer}
\end{promptbox}

\caption{Prompt templates used for LLM-only QA, RAG QA, and LLM-as-a-Judge correctness evaluation.}
\label{fig:prompt_templates}
\end{figure*}

\newpage
\section{Additional Experimental Results}
\paragraph{Risk Control Holds Across Datasets}
Figures~\ref{fig:risk_level_squad_appendix} and~\ref{fig:risk_level_nq_appendix} show the results of BalanceRAG on SQuAD~v2 and NQ. Consistent with the TriviaQA results in the main text, BalanceRAG achieves stable risk control across different models and target risk levels, while its coverage is usually close to or even higher than that of LLM-RAG. More importantly, under stricter values of $\alpha$, the UCB-Cascaded baselines often fail to find a feasible point, which suggests that calibrating the two branches separately can be overly restrictive under a tight risk budget. In contrast, BalanceRAG still returns feasible solutions, showing that system-level calibration preserves reliability across datasets with different retrieval difficulty.

\begin{figure*}[!t]
  \centering
  \captionsetup[subfigure]{skip=0pt}
  \setlength{\abovecaptionskip}{2pt}
  \setlength{\belowcaptionskip}{0pt}

  \begin{subfigure}[t]{0.5\textwidth}
    \centering
    \includegraphics[width=\linewidth]{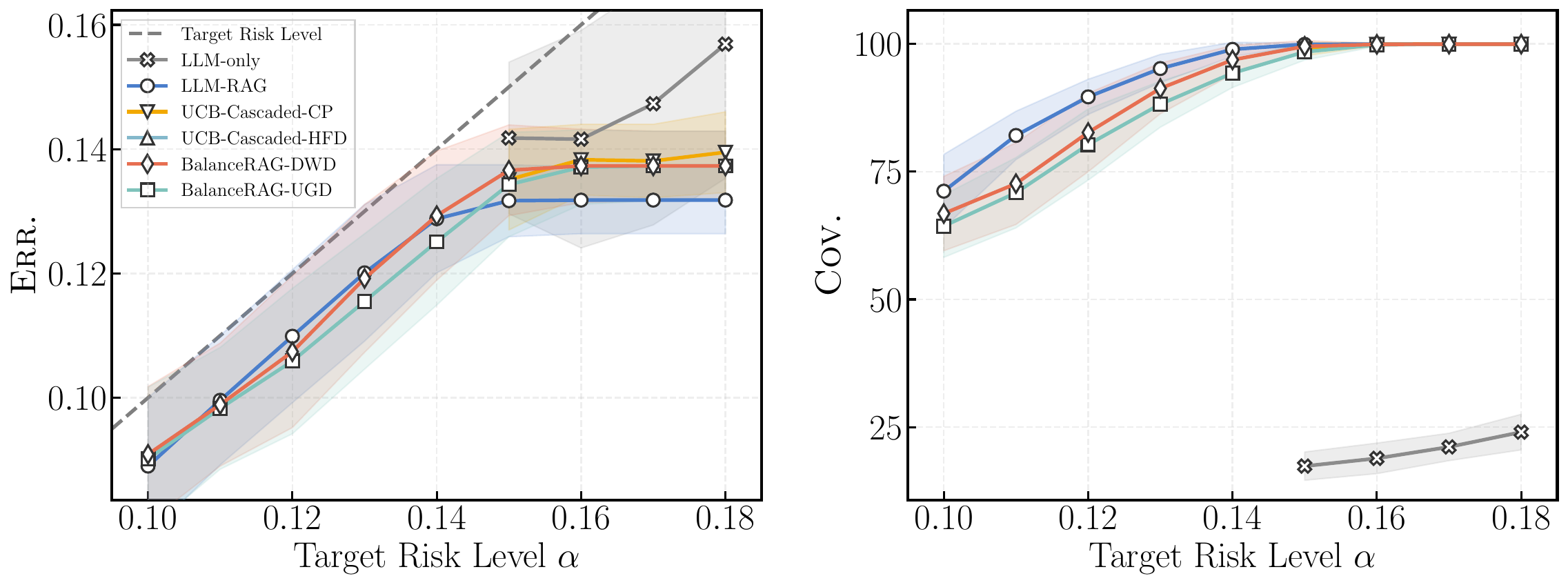}
    \caption{Qwen2.5-3B}
  \end{subfigure}%
  \begin{subfigure}[t]{0.5\textwidth}
    \centering
    \includegraphics[width=\linewidth]{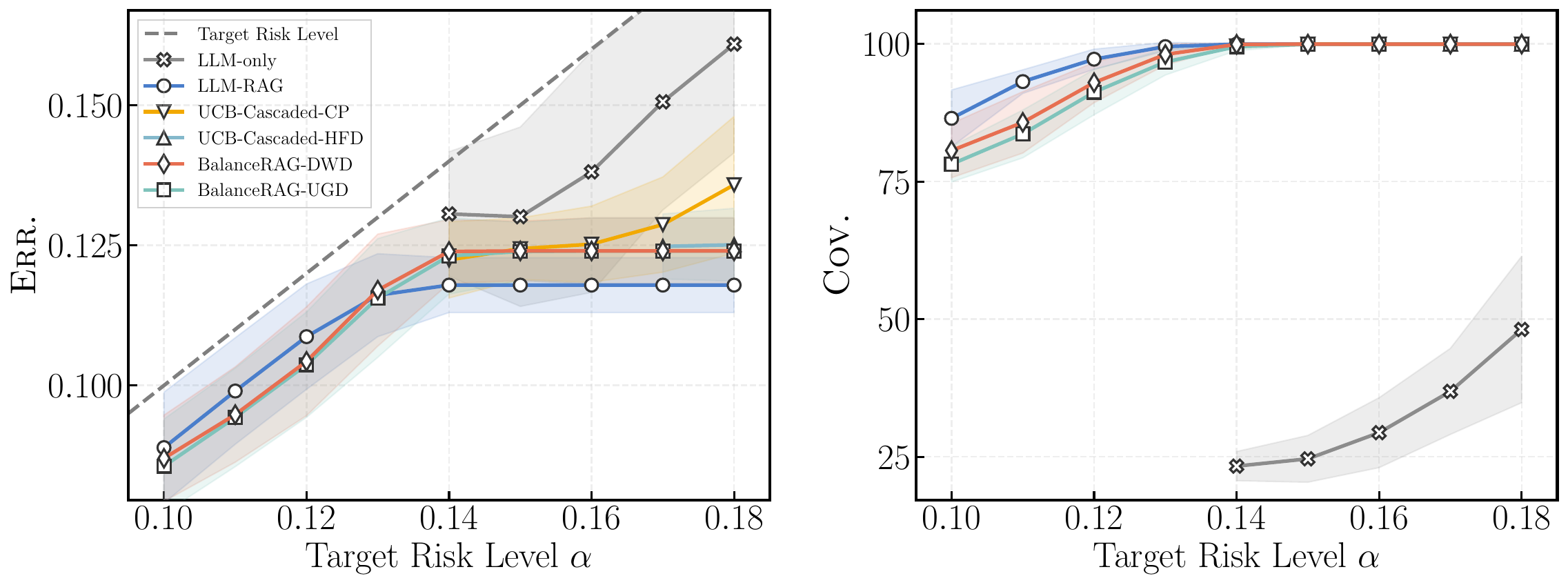}
    \caption{Qwen2.5-7B}
  \end{subfigure}

  \vspace{-0.1em}

  \begin{subfigure}[t]{0.5\textwidth}
    \centering
    \includegraphics[width=\linewidth]{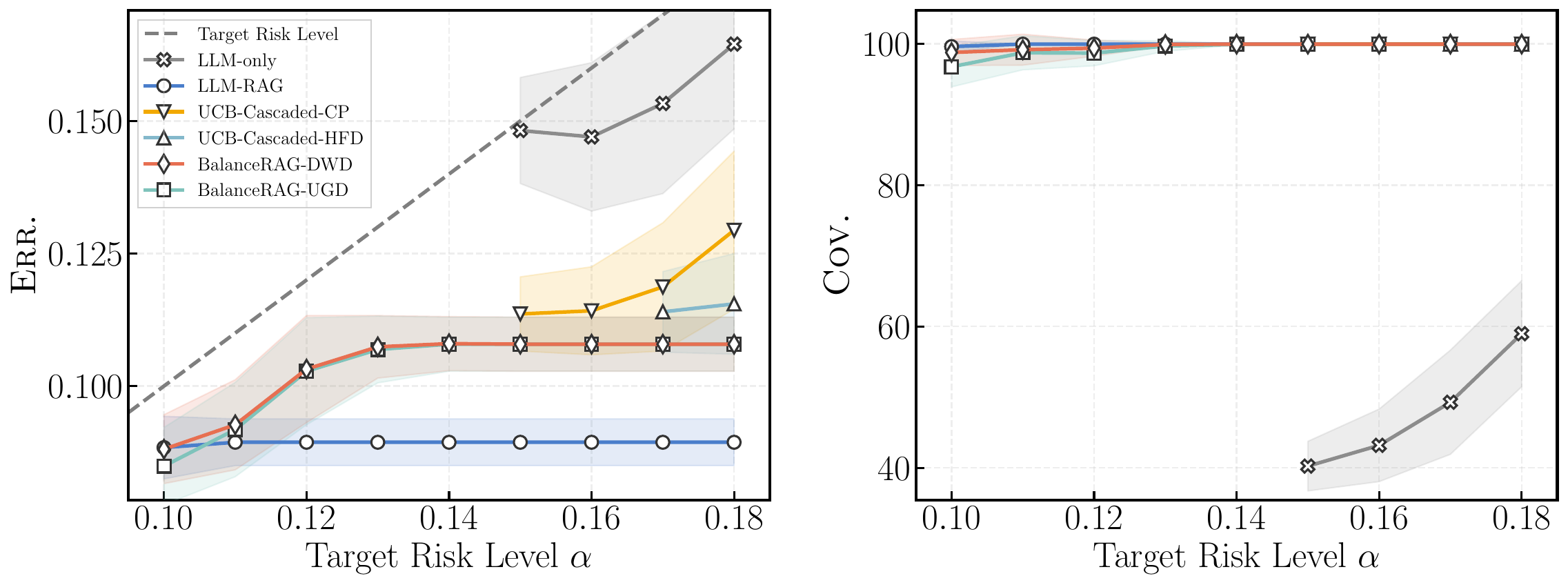}
    \caption{Qwen3-4B}
  \end{subfigure}%
  \begin{subfigure}[t]{0.5\textwidth}
    \centering
    \includegraphics[width=\linewidth]{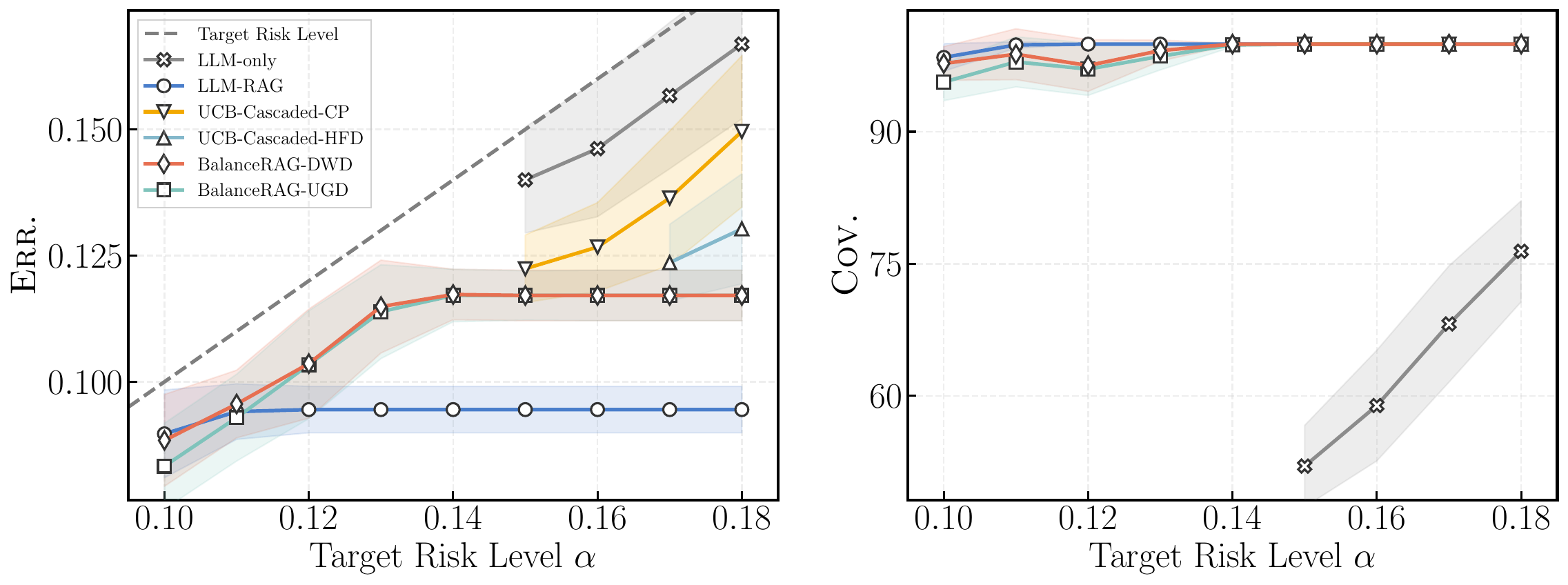}
    \caption{Qwen3-14B}
  \end{subfigure}

  \vspace{-0.1em}

  \begin{subfigure}[t]{0.5\textwidth}
    \centering
    \includegraphics[width=\linewidth]{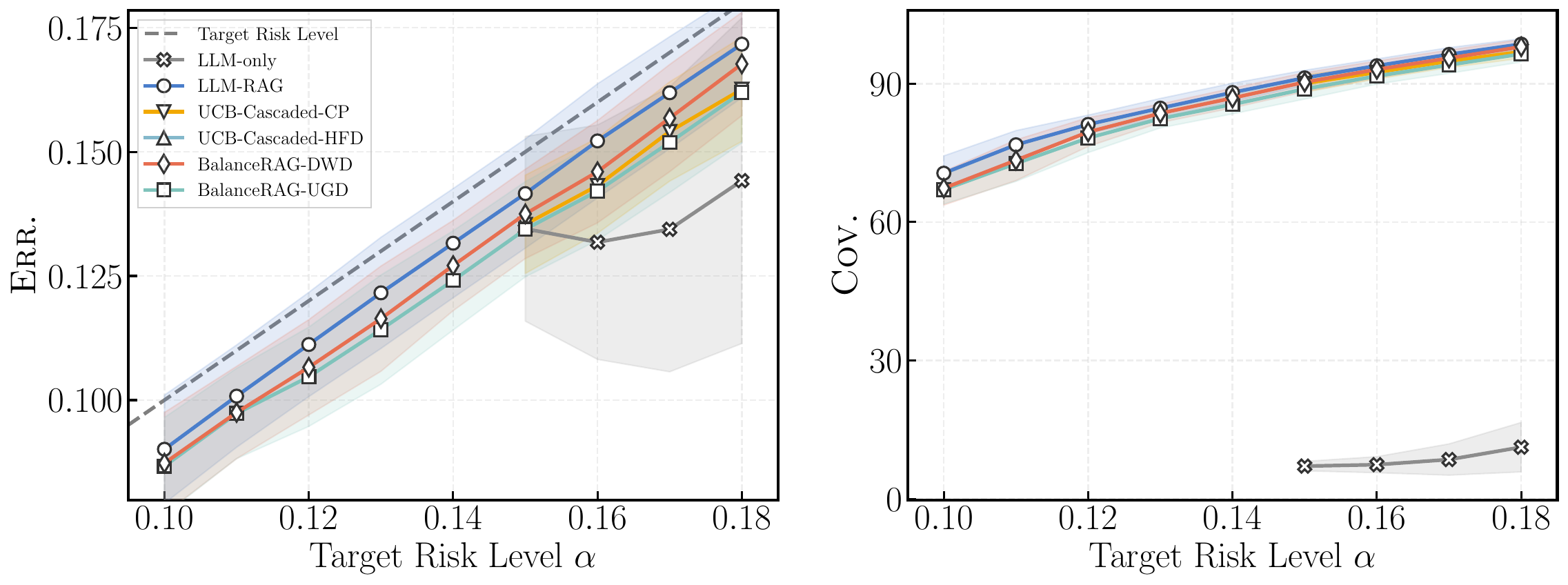}
    \caption{Llama3.2-3B}
  \end{subfigure}%
  \begin{subfigure}[t]{0.5\textwidth}
    \centering
    \includegraphics[width=\linewidth]{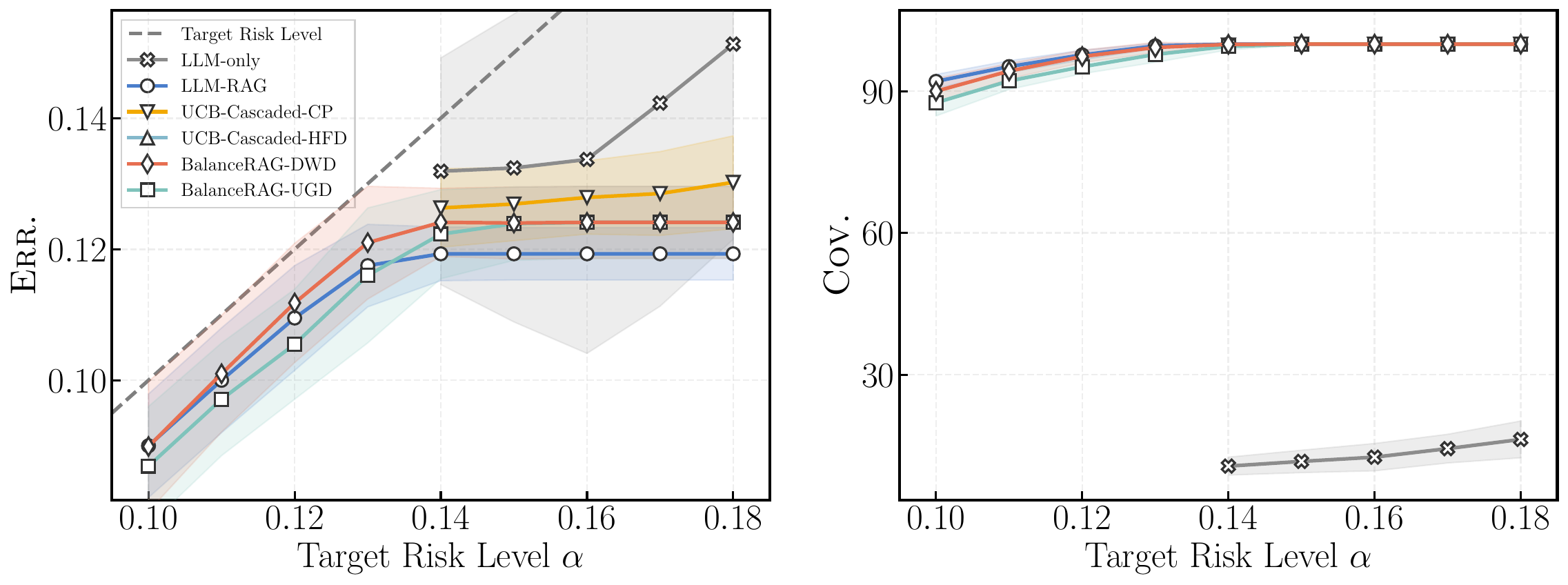}
    \caption{Llama3.1-8B}
  \end{subfigure}

  \vspace{-0.1em}

  \begin{subfigure}[t]{0.5\textwidth}
    \centering
    \includegraphics[width=\linewidth]{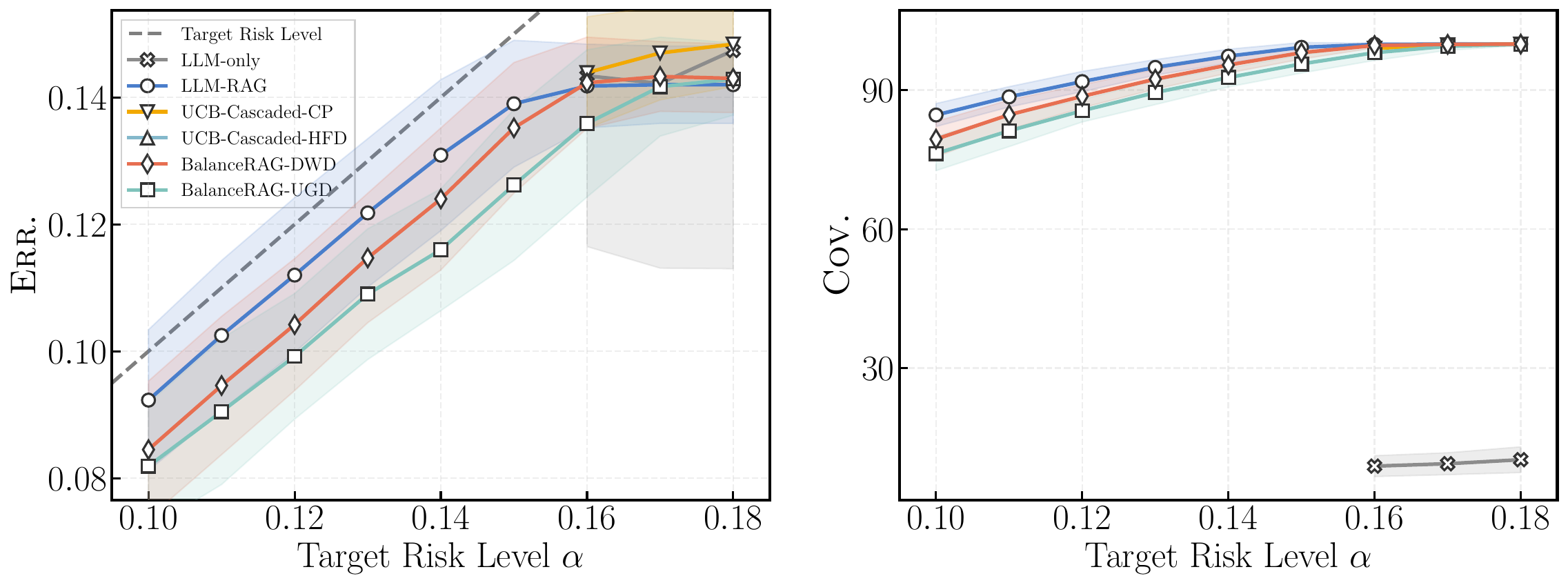}
    \caption{Vicuna-7B}
  \end{subfigure}%
  \begin{subfigure}[t]{0.5\textwidth}
    \centering
    \includegraphics[width=\linewidth]{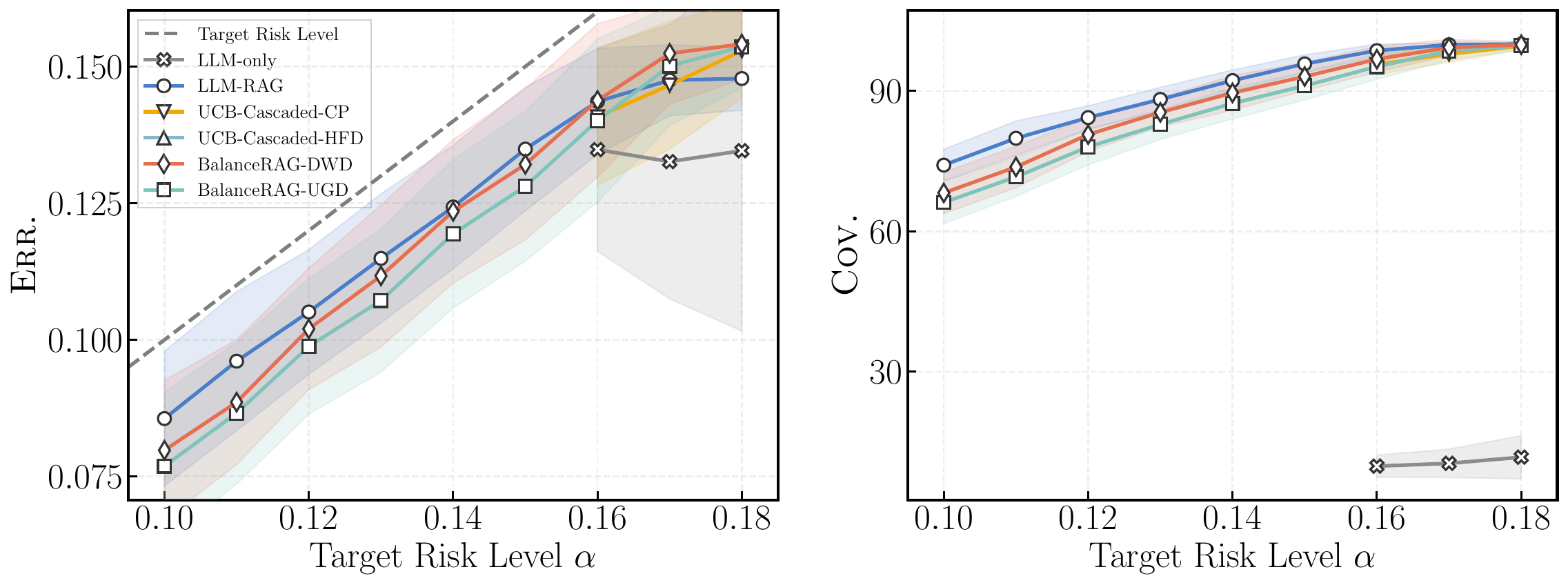}
    \caption{Vicuna-13B}
  \end{subfigure}

  \caption{Err. (left in each panel pair) and Cov. (right) under
different target risk levels $\alpha$ across eight LLMs spanning four model
families on SQuAD~v2.}
  \label{fig:risk_level_squad_appendix}
  \vspace{-5mm}
\end{figure*}

\begin{figure*}[!t]
  \centering
  \captionsetup[subfigure]{skip=0pt}
  \setlength{\abovecaptionskip}{2pt}
  \setlength{\belowcaptionskip}{0pt}

  \begin{subfigure}[t]{0.5\textwidth}
    \centering
    \includegraphics[width=\linewidth]{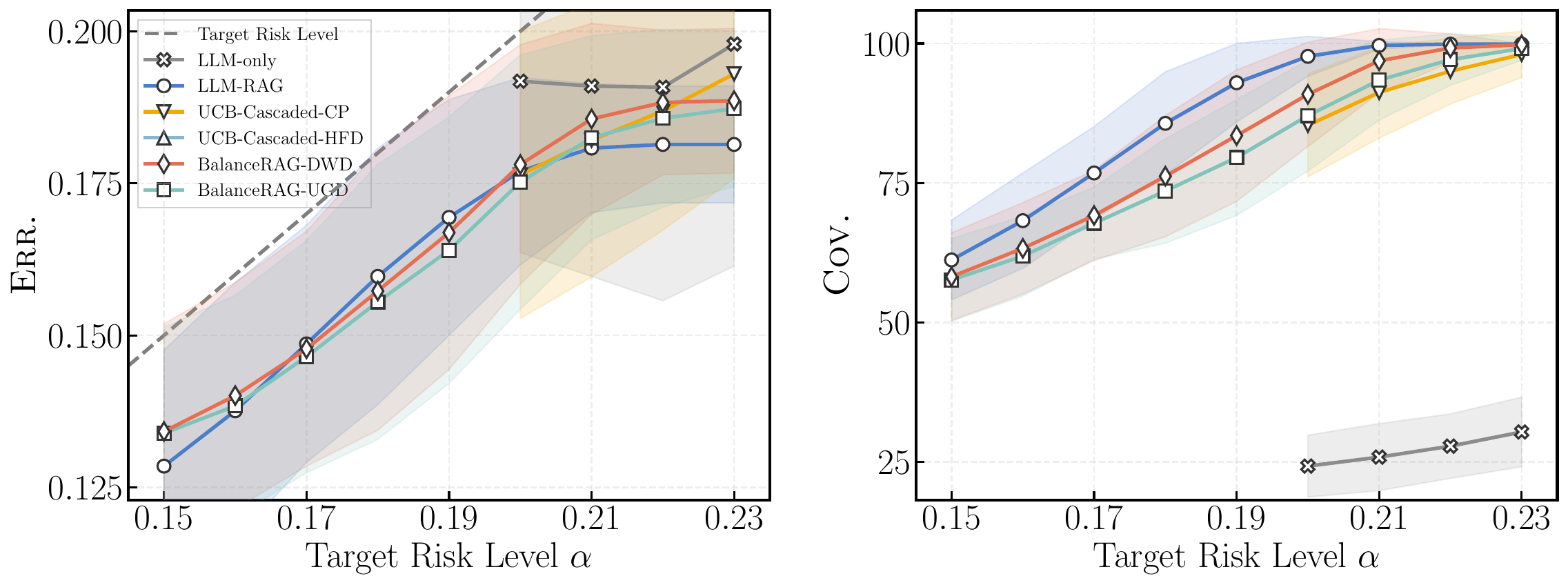}
    \caption{Qwen2.5-3B}
  \end{subfigure}%
  \begin{subfigure}[t]{0.5\textwidth}
    \centering
    \includegraphics[width=\linewidth]{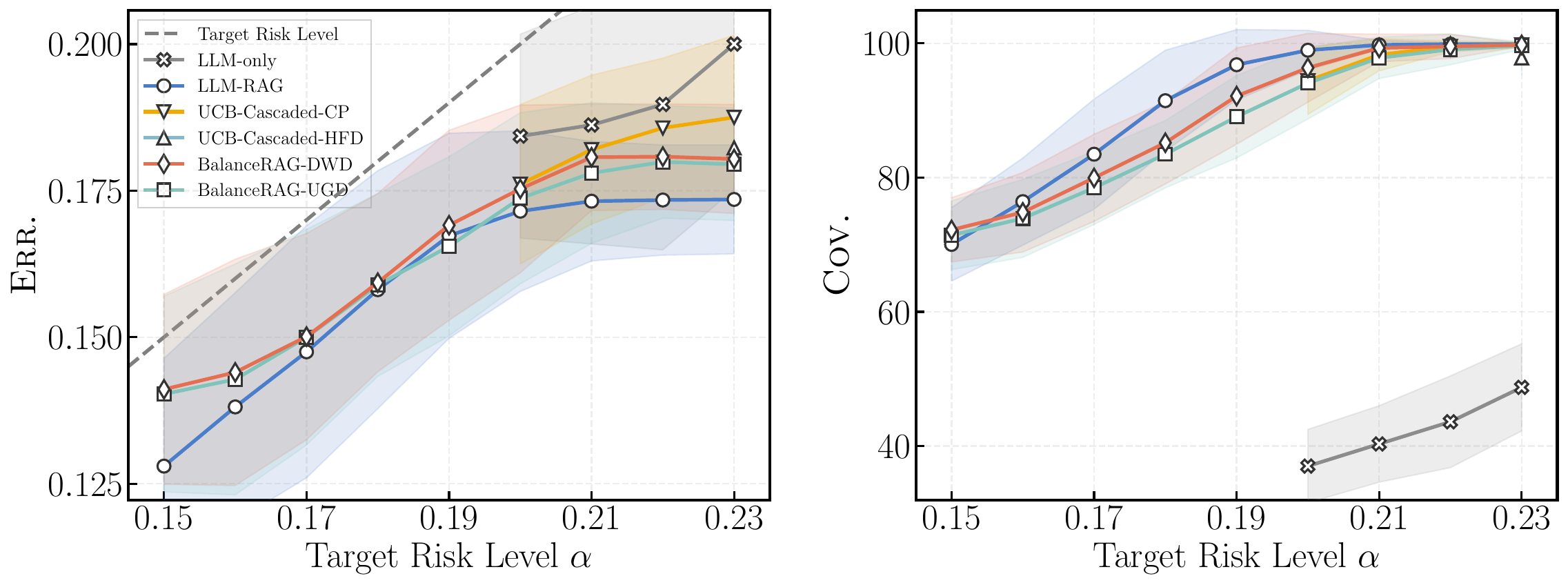}
    \caption{Qwen2.5-7B}
  \end{subfigure}

  \vspace{-0.1em}

  \begin{subfigure}[t]{0.5\textwidth}
    \centering
    \includegraphics[width=\linewidth]{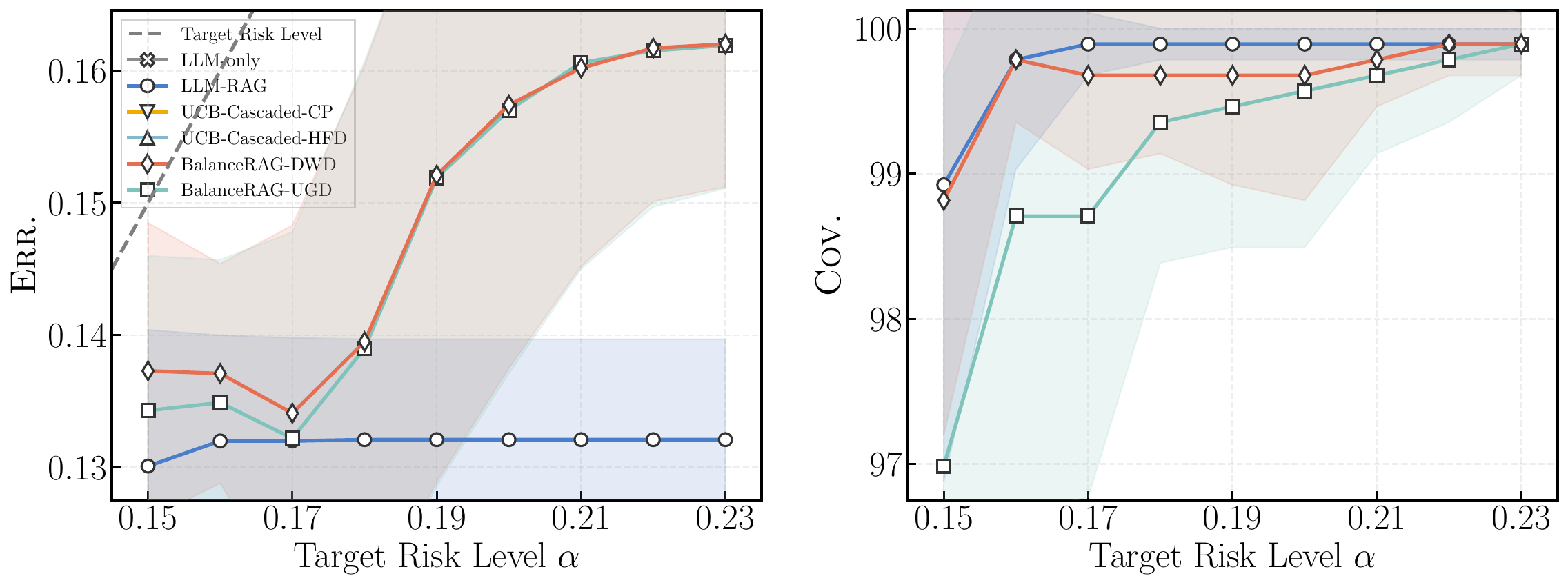}
    \caption{Qwen3-4B}
  \end{subfigure}%
  \begin{subfigure}[t]{0.5\textwidth}
    \centering
    \includegraphics[width=\linewidth]{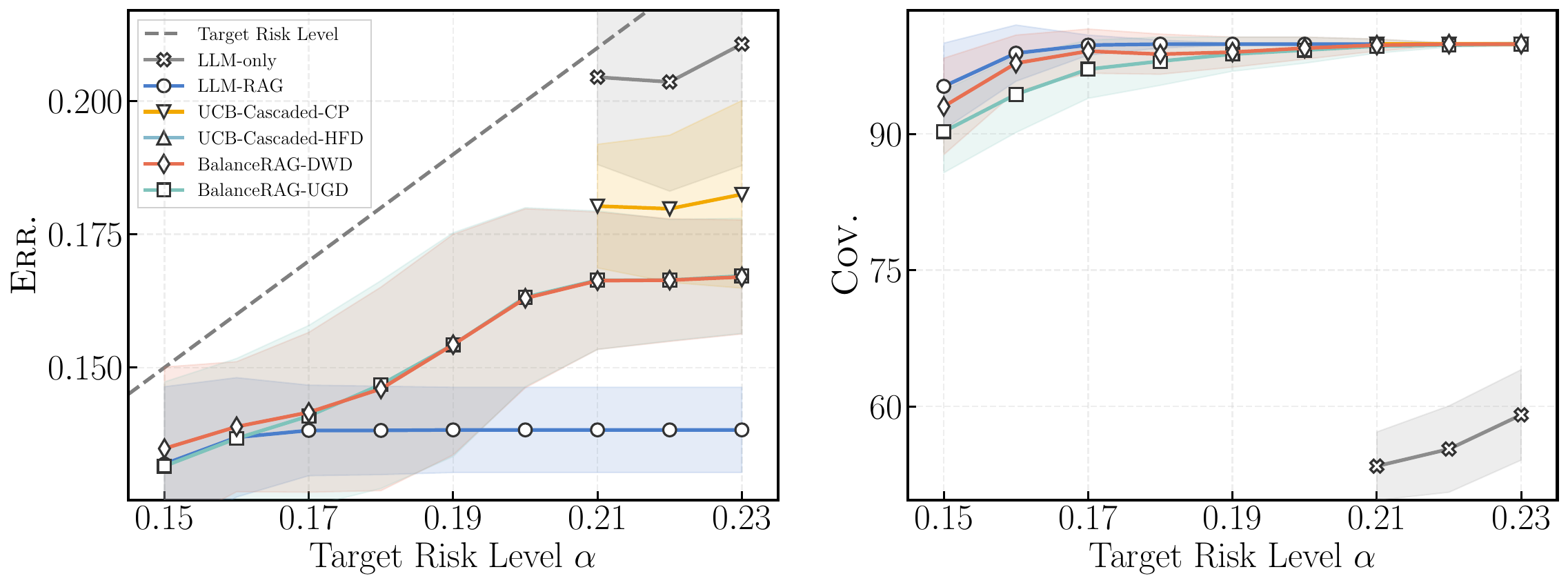}
    \caption{Qwen3-14B}
  \end{subfigure}

  \vspace{-0.1em}

  \begin{subfigure}[t]{0.5\textwidth}
    \centering
    \includegraphics[width=\linewidth]{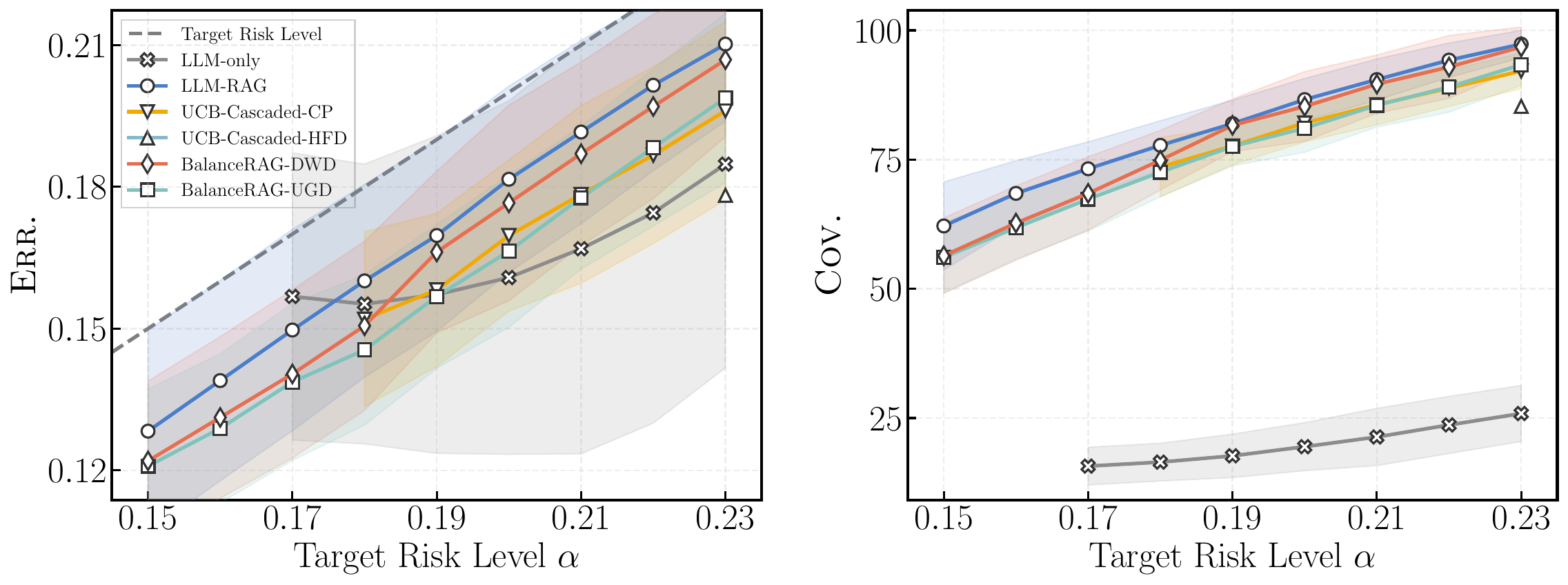}
    \caption{Llama3.2-3B}
  \end{subfigure}%
  \begin{subfigure}[t]{0.5\textwidth}
    \centering
    \includegraphics[width=\linewidth]{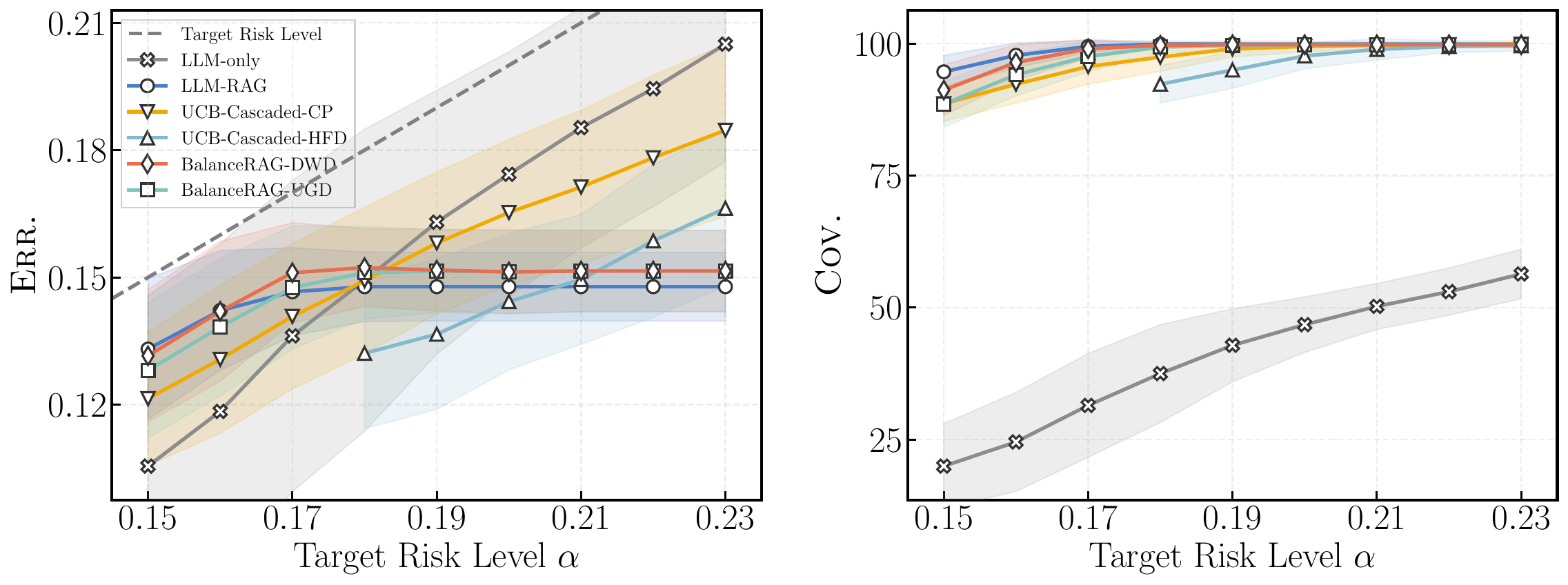}
    \caption{Llama3.1-8B}
  \end{subfigure}

  \vspace{-0.1em}

  \begin{subfigure}[t]{0.5\textwidth}
    \centering
    \includegraphics[width=\linewidth]{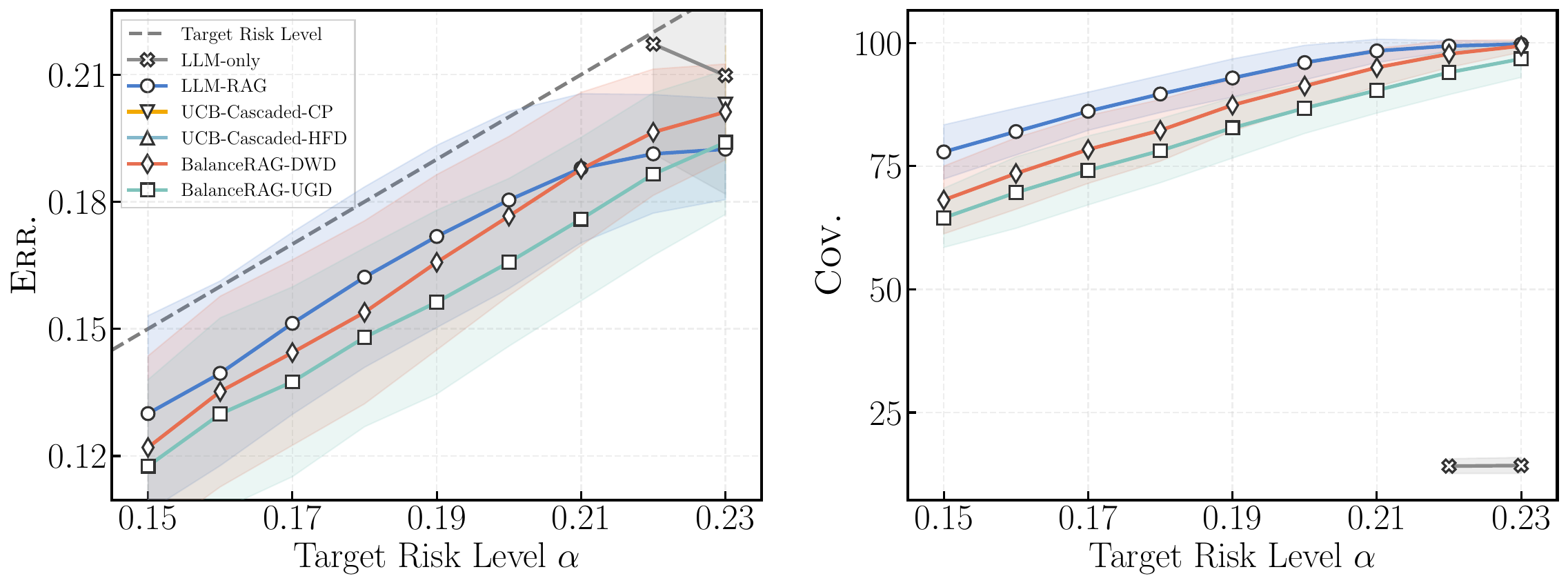}
    \caption{Vicuna-7B}
  \end{subfigure}%
  \begin{subfigure}[t]{0.5\textwidth}
    \centering
    \includegraphics[width=\linewidth]{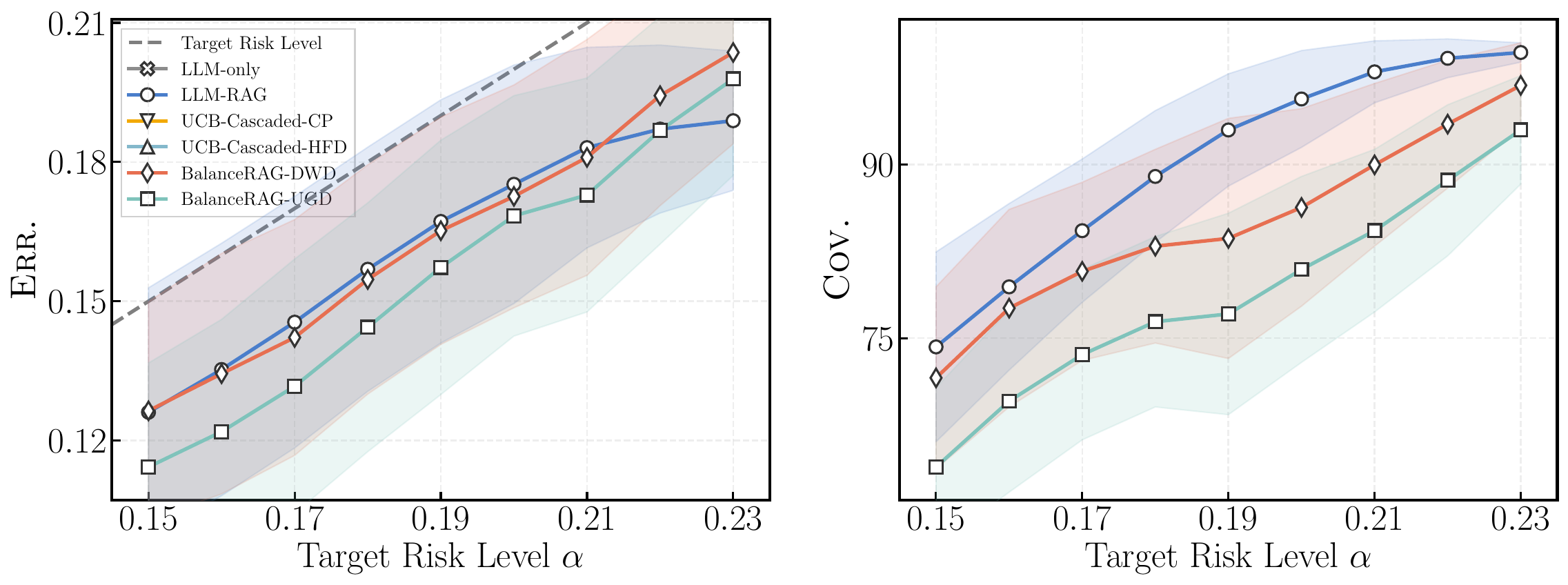}
    \caption{Vicuna-13B}
  \end{subfigure}

  \caption{Err. (left in each panel pair) and Cov. (right) under
different target risk levels $\alpha$ across eight LLMs spanning four model
families on NQ.}
  \label{fig:risk_level_nq_appendix}
  
\end{figure*}

\paragraph{Correct Acceptances Remain Competitive Under Risk Control}
Tables~\ref{tab:corr_comparison_triviaqa},~\ref{tab:corr_comparison_squad}, and~\ref{tab:corr_comparison_nq} compare the number of correctly accepted examples under the same target risk levels. The results show that the risk control of BalanceRAG is not obtained by simply rejecting a large number of examples. Instead, BalanceRAG preserves high utility while satisfying the risk constraint. On TriviaQA, BalanceRAG is often close to or better than RAG, and clearly outperforms UCB-Cascaded. On SQuAD~v2 and NQ, RAG usually retains more correctly accepted examples, but BalanceRAG still maintains comparable utility and is more likely to remain feasible than the UCB-Cascaded baselines under strict values of $\alpha$. These results show that BalanceRAG maintains utility close to RAG while providing feasible risk control.
\begin{table}[!htbp]
\centering
\small
\setlength{\tabcolsep}{5.5pt}
\caption{
Effect of uncertainty quantification methods on BalanceRAG on TriviaQA.
AUROC results reflect the performance of different uncertainty quantification methods on the LLM-only and RAG branches.
Coverage is reported under different target risk levels $\alpha$ after joint risk calibration.
}
\label{tab:uncertainty_ablation}
\begin{tabular}{lcccc}
\toprule
Uncertainty & AUROC Only/RAG $\uparrow$ 
& Cov. @ $\alpha=0.10$ $\uparrow$ 
& Cov. @ $\alpha=0.14$ $\uparrow$ 
& Cov. @ $\alpha=0.18$ $\uparrow$ \\
\midrule
Deg  & 0.7476 / 0.6978 & 81.97 & 99.95 & 99.96 \\
SELF & 0.5839 / 0.6800 & 70.25 & 99.43 & 99.93 \\
SE  & 0.7248 / 0.6544 & 84.11 & 99.96 & 99.96 \\
EigV & 0.7488 / 0.6987 & 82.87 & 99.95 & 99.95 \\
Ecc  & 0.6900 / 0.6706 & 72.89 & 99.49 & 99.95 \\
\bottomrule
\end{tabular}
\vspace{-2mm}
\end{table}

\paragraph{Risk Control Is Stable Across Uncertainty Estimators}
Figure~\ref{fig:uq_methods_triviaqa} and Table~\ref{tab:uncertainty_ablation} examine whether BalanceRAG depends on a particular uncertainty estimator. Deg, SELF, SE, EigV, and Ecc provide different forms of uncertainty signals, and therefore affect sample ranking, calibrated thresholds, and final coverage. The results show that, although different estimators lead to different utility levels, BalanceRAG still maintains risk control near the target levels across these settings. This finding suggests that the reliability of the method mainly comes from the calibration procedure itself, rather than from any specific uncertainty score. In other words, the uncertainty estimator affects utility and acceptance size, but does not change the basic ability of BalanceRAG to achieve risk control.

\begin{figure}[!htbp]
  \centering
  \captionsetup[subfigure]{skip=0pt}
  \setlength{\abovecaptionskip}{2pt}
  \setlength{\belowcaptionskip}{0pt}

  \begin{subfigure}[t]{0.5\linewidth}
    \centering
    \includegraphics[width=\linewidth]{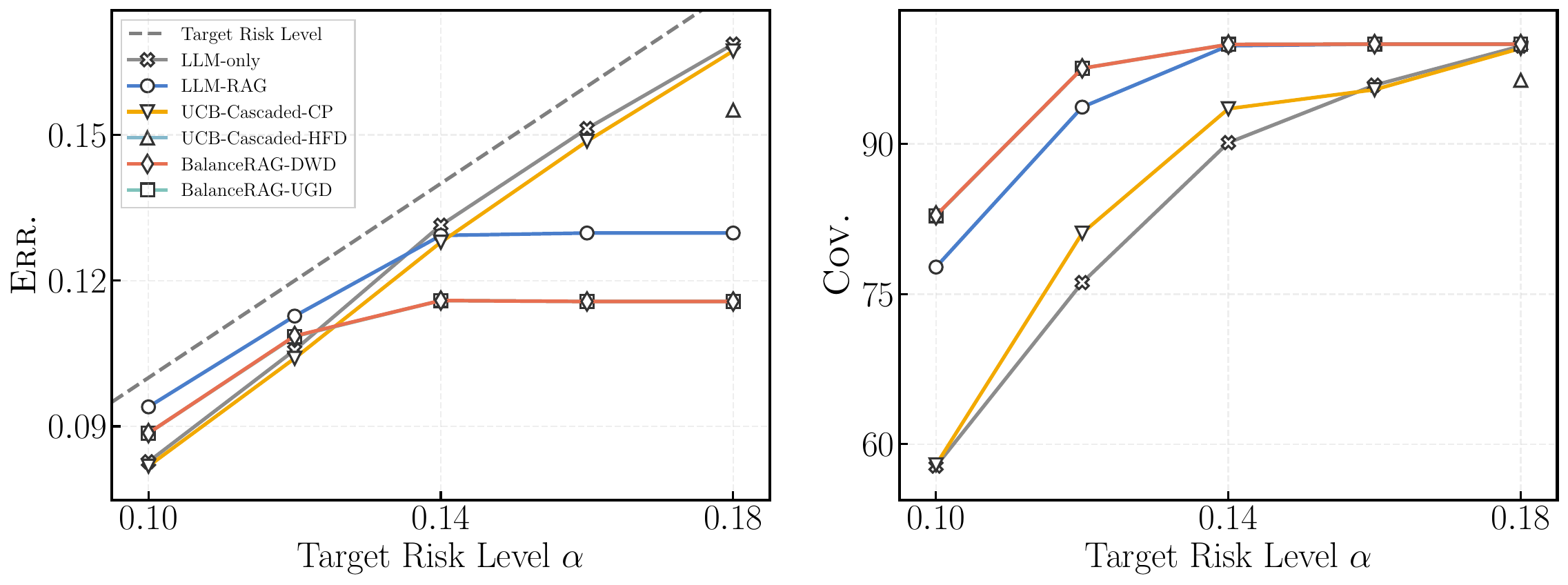}
    \caption{EigV}
  \end{subfigure}%
  \begin{subfigure}[t]{0.5\linewidth}
    \centering
    \includegraphics[width=\linewidth]{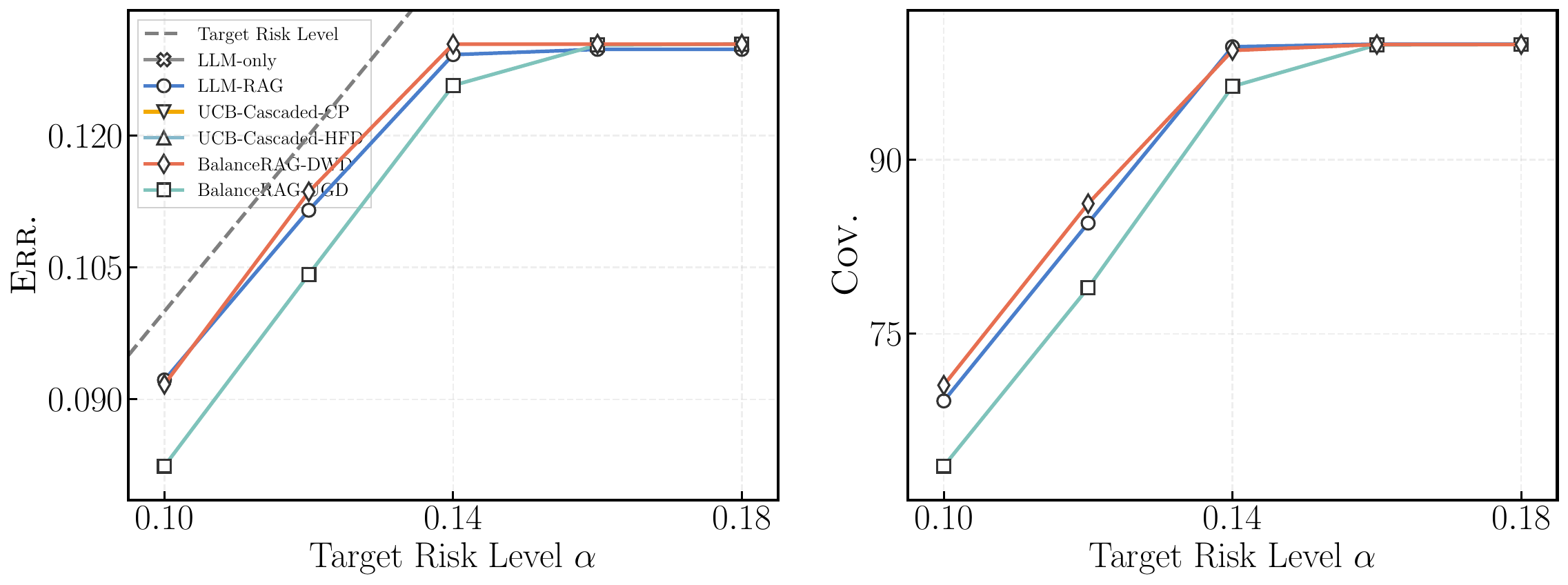}
    \caption{SELF}
  \end{subfigure}

  \par\vspace{-0.1em}

  \begin{subfigure}[t]{0.5\linewidth}
    \centering
    \includegraphics[width=\linewidth]{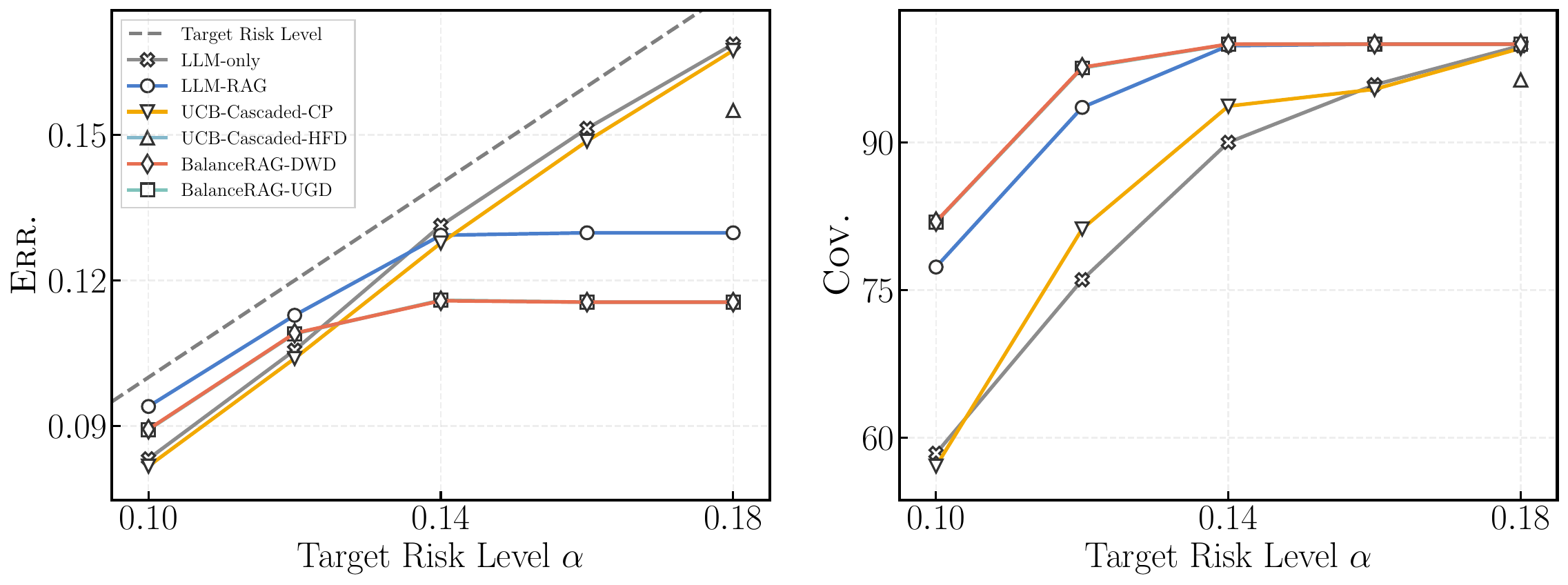}
    \caption{Deg}
  \end{subfigure}%
  \begin{subfigure}[t]{0.5\linewidth}
    \centering
    \includegraphics[width=\linewidth]{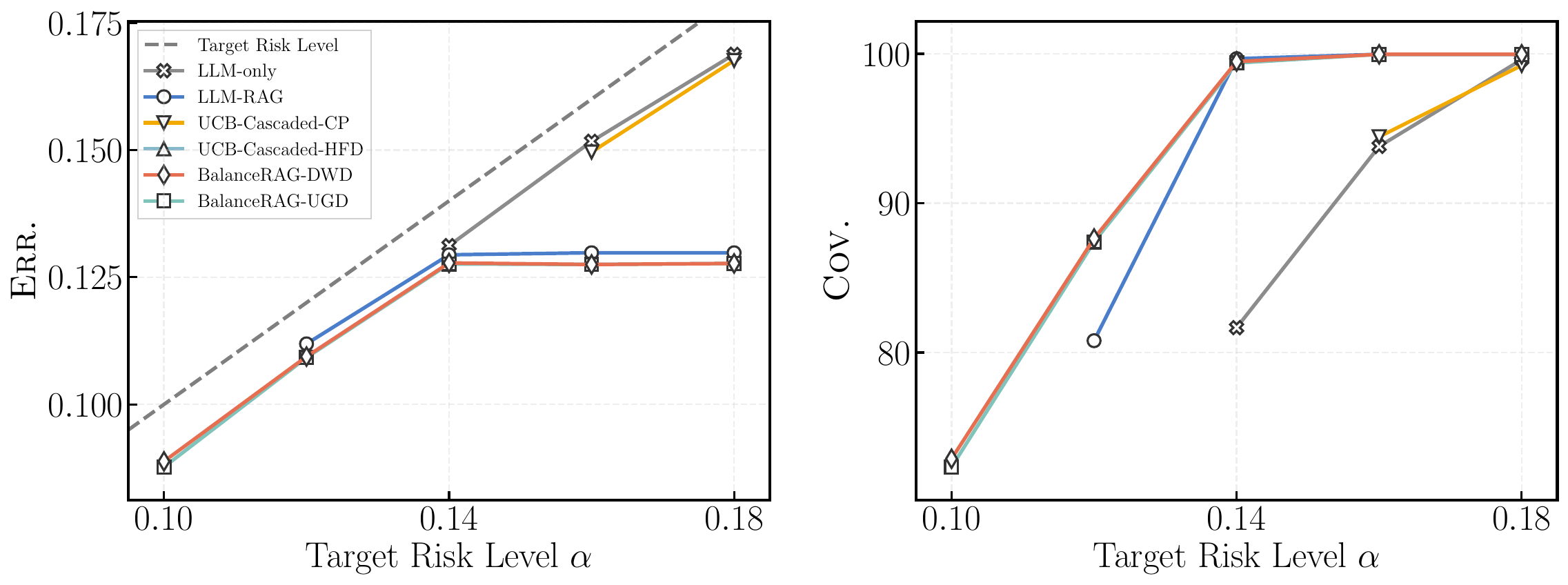}
    \caption{Ecc}
  \end{subfigure}
  \caption{Err. and Cov. on TriviaQA with Qwen2.5-7B using different UQ methods.}
  \label{fig:uq_methods_triviaqa}
\end{figure}

\begin{table}[htbp]
\centering
\small
\setlength{\tabcolsep}{4pt}
\renewcommand{\arraystretch}{1.08}
\caption{Branch-level error rates under different correctness evaluation methods on TriviaQA.}
\label{tab:triviaqa_eval_branch_error}
\begin{tabular}{lccccc@{\hspace{1.2em}}lccccc}
\toprule

\multirow{2}{*}{\raisebox{-0.6ex}{Model}}
& \multirow{2}{*}{\raisebox{-0.6ex}{$N$}}
& \multicolumn{2}{c}{LLM-as-a-Judge}
& \multicolumn{2}{c}{Entailment}
&
\multirow{2}{*}{\raisebox{-0.6ex}{Model}}
& \multirow{2}{*}{\raisebox{-0.6ex}{$N$}}
& \multicolumn{2}{c}{LLM-as-a-Judge}
& \multicolumn{2}{c}{Entailment} \\
\cmidrule(lr){3-4}
\cmidrule(lr){5-6}
\cmidrule(lr){9-10}
\cmidrule(lr){11-12}

& & Only $\downarrow$ & RAG $\downarrow$ & Only $\downarrow$ & RAG $\downarrow$
& & & Only $\downarrow$ & RAG $\downarrow$ & Only $\downarrow$ & RAG $\downarrow$ \\
\midrule

Llama3.2-3B
& 6248
& 0.2623 & 0.2308
& 0.3782 & 0.2591
&
Qwen2.5-3B
& 5929
& 0.2889 & 0.2663
& 0.3719 & 0.1881 \\

Llama3.1-8B
& 6365
& 0.1277 & 0.2687
& 0.2353 & 0.1983
&
Qwen2.5-7B
& 6196
& 0.1921 & 0.2515
& 0.2352 & 0.1451 \\

\bottomrule
\end{tabular}
\vspace{-2mm}
\end{table}

\paragraph{Risk Control Is Stable Across Correctness Criteria}
Figures~\ref{fig:judge_triviaqa} and~\ref{fig:entailment_triviaqa} evaluate BalanceRAG under LLM-as-a-Judge and entailment criteria. These two correctness criteria change which examples are treated as correct, and therefore affect the error signals observed during calibration and the final accepted set. The results show that BalanceRAG still maintains risk control under different correctness criteria, indicating that it does not rely on a single correctness definition. More importantly, BalanceRAG does not assume that one branch is always more reliable. Instead, it adjusts the accepted set and routing decisions according to the calibration results under the current evaluation criterion. Therefore, when the correctness criterion changes, the system can adapt through recalibration rather than being fixed toward either LLM-only or RAG.
\begin{figure}[!htbp]
  \centering
  \captionsetup[subfigure]{skip=0pt}
  \setlength{\abovecaptionskip}{2pt}
  \setlength{\belowcaptionskip}{0pt}

  \begin{subfigure}[t]{0.5\linewidth}
    \centering
    \includegraphics[width=\linewidth]{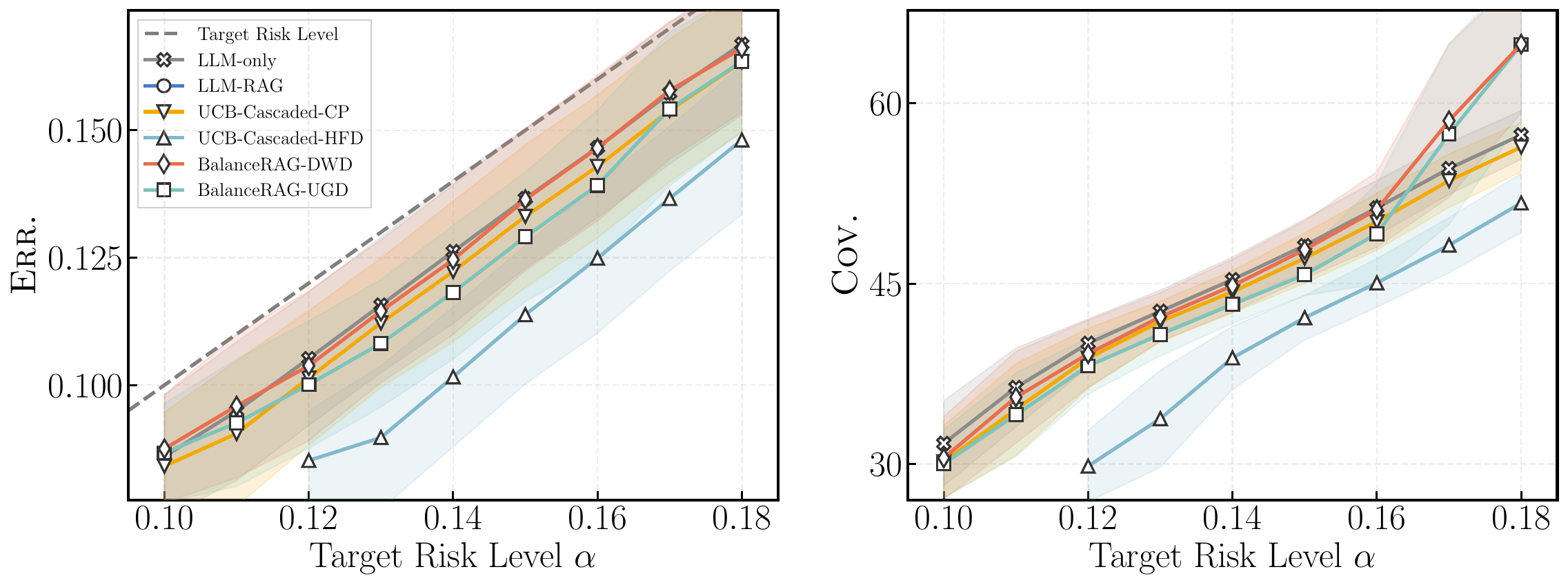}
    \caption{Qwen2.5-3B}
  \end{subfigure}%
  \begin{subfigure}[t]{0.5\linewidth}
    \centering
    \includegraphics[width=\linewidth]{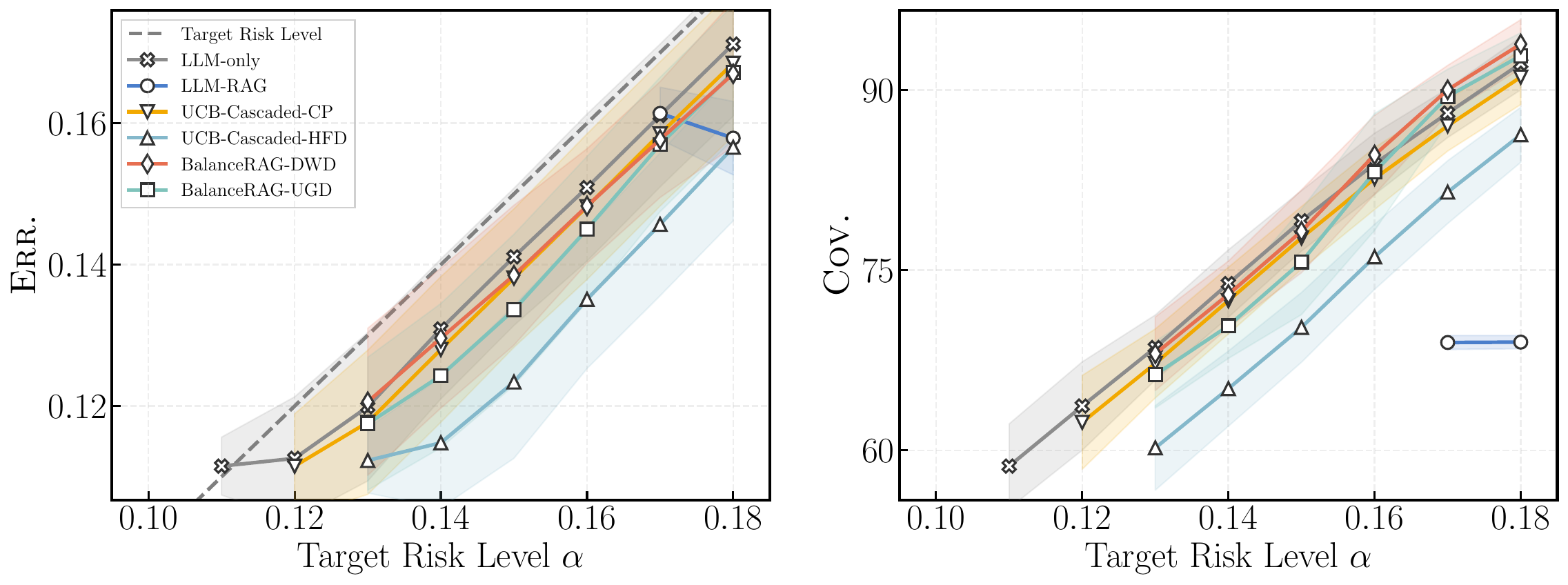}
    \caption{Qwen2.5-7B}
  \end{subfigure}

  \par\vspace{-0.1em}

  \begin{subfigure}[t]{0.5\linewidth}
    \centering
    \includegraphics[width=\linewidth]{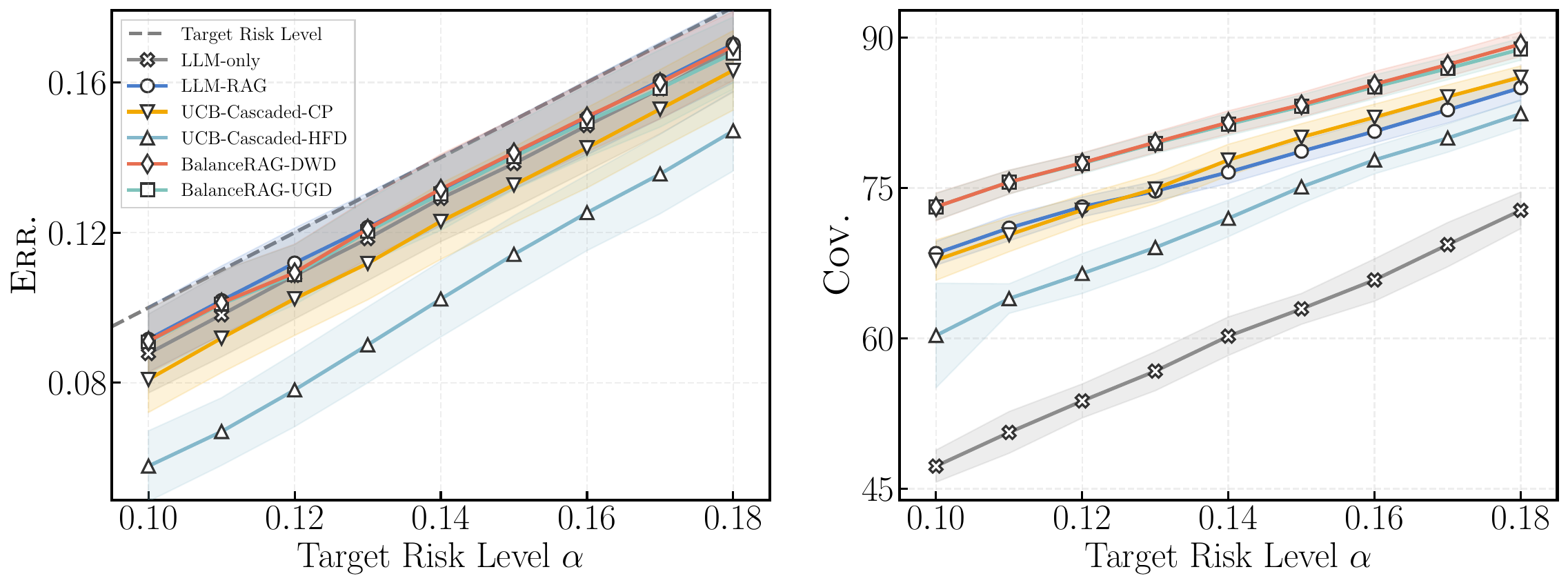}
    \caption{Llama3.2-3B}
  \end{subfigure}%
  \begin{subfigure}[t]{0.5\linewidth}
    \centering
    \includegraphics[width=\linewidth]{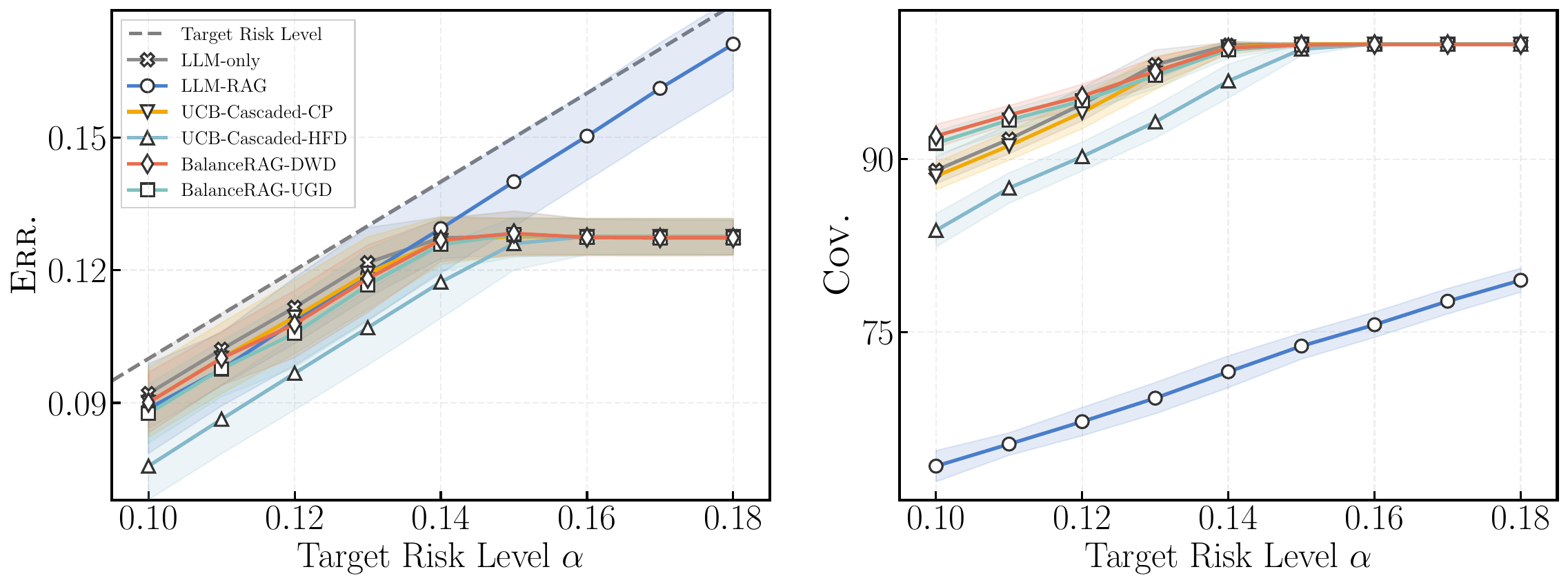}
    \caption{Llama3.1-8B}
  \end{subfigure}
  \caption{Err. and Cov. on TriviaQA with LLM-as-a-Judge for correctness evaluation.}
  \label{fig:judge_triviaqa}
  \vspace{-5mm}
\end{figure}

\paragraph{Calibration Size Mainly Affects Stability}
Figures~\ref{fig:split_ratio_all} analyze the effect of the calibration-test split ratio. With fewer calibration examples, threshold estimation is more sensitive to random splits, and the results may fluctuate more. With more calibration examples, the estimates are generally more stable. In the experiments, the empirical risk of BalanceRAG consistently satisfies risk control as the split ratio changes. Moreover, under some settings of $\alpha$, the variance of Err. decreases as the calibration size increases. These results show that the calibration size mainly affects stability, rather than changing the basic mechanism by which the method achieves risk control.

\paragraph{Routing Allocation Reflects Branch Complementarity}
Table~\ref{tab:balancerag_allocation} shows how accepted examples are allocated between the LLM-only and RAG branches. The allocation ratios vary across models and datasets, indicating that BalanceRAG neither always selects LLM-RAG nor always preserves LLM-only. In some settings, stronger models allow more examples to remain in the LLM-only branch, thereby reducing retrieval calls. In other settings, more examples are routed to RAG to use the additional information provided by retrieval. This variation shows that the two branches are complementary, and that BalanceRAG uses calibrated risk signals to decide when to trust each branch.
\begin{figure}[!htbp]
  \centering
  \captionsetup[subfigure]{skip=0pt}
  \setlength{\abovecaptionskip}{2pt}
  \setlength{\belowcaptionskip}{0pt}

  \begin{subfigure}[t]{0.5\linewidth}
    \centering
    \includegraphics[width=\linewidth]{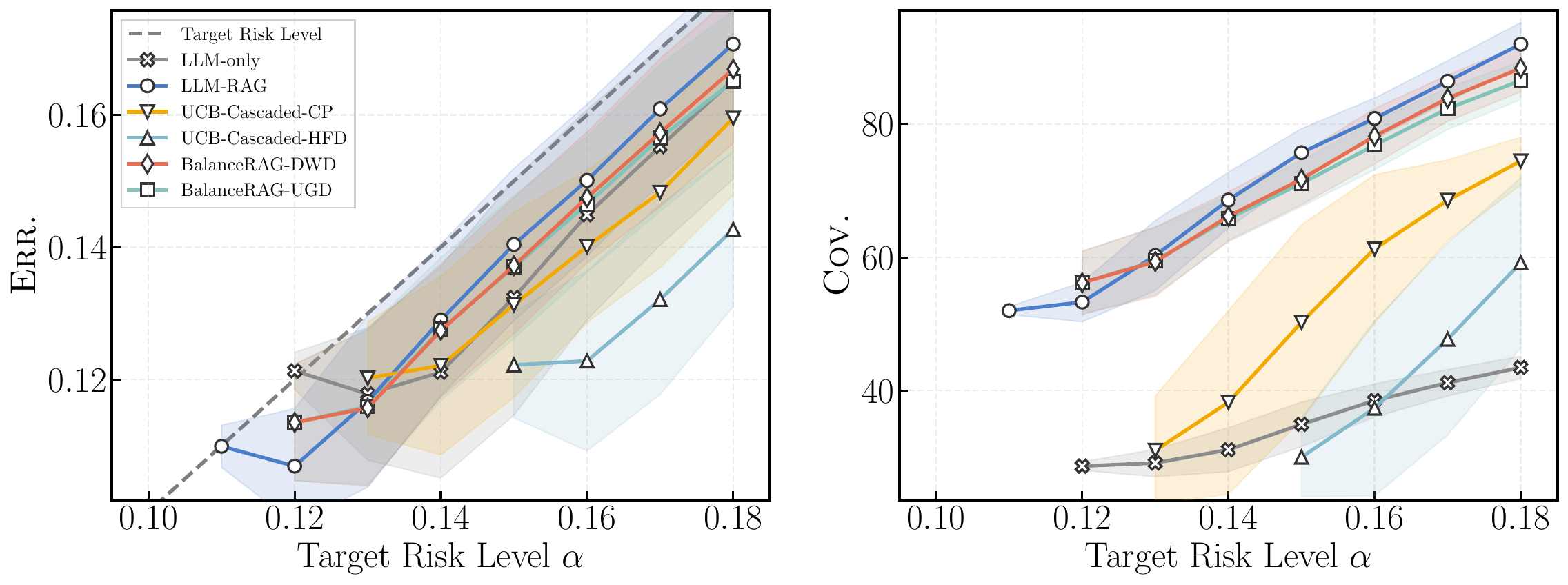}
    \caption{Qwen2.5-3B}
  \end{subfigure}%
  \begin{subfigure}[t]{0.5\linewidth}
    \centering
    \includegraphics[width=\linewidth]{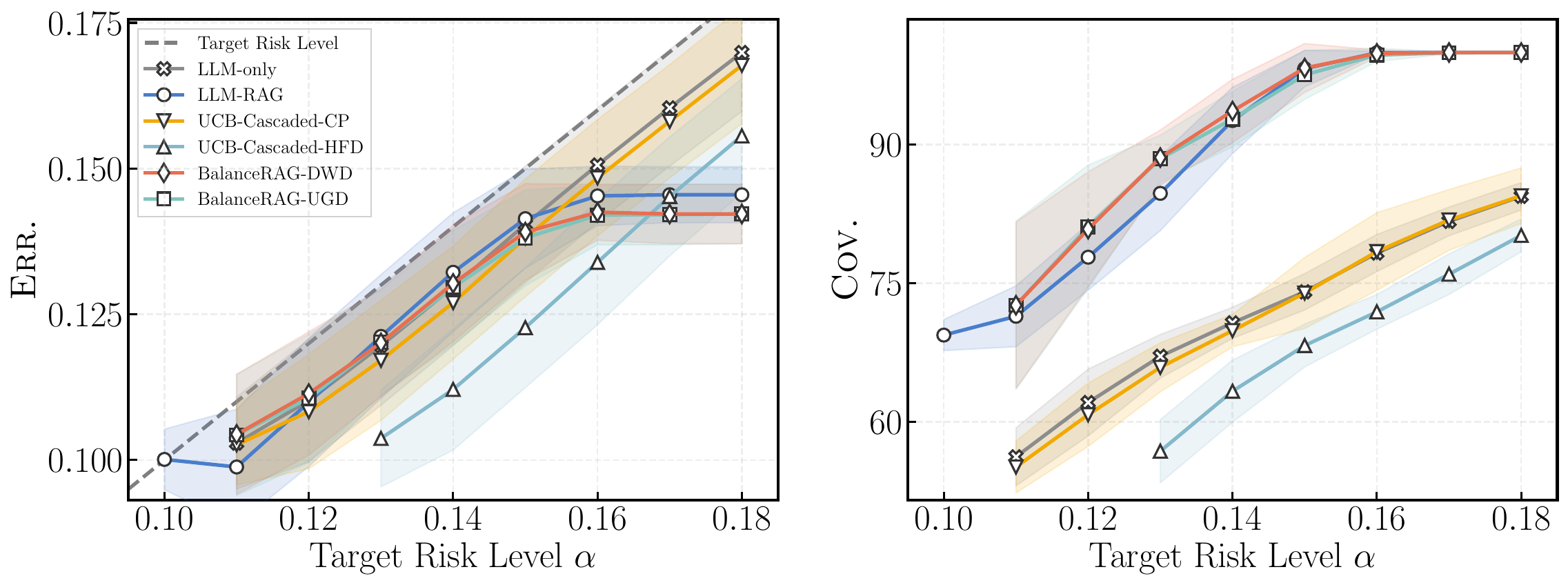}
    \caption{Qwen2.5-7B}
  \end{subfigure}

  \par\vspace{-0.1em}

  \begin{subfigure}[t]{0.5\linewidth}
    \centering
    \includegraphics[width=\linewidth]{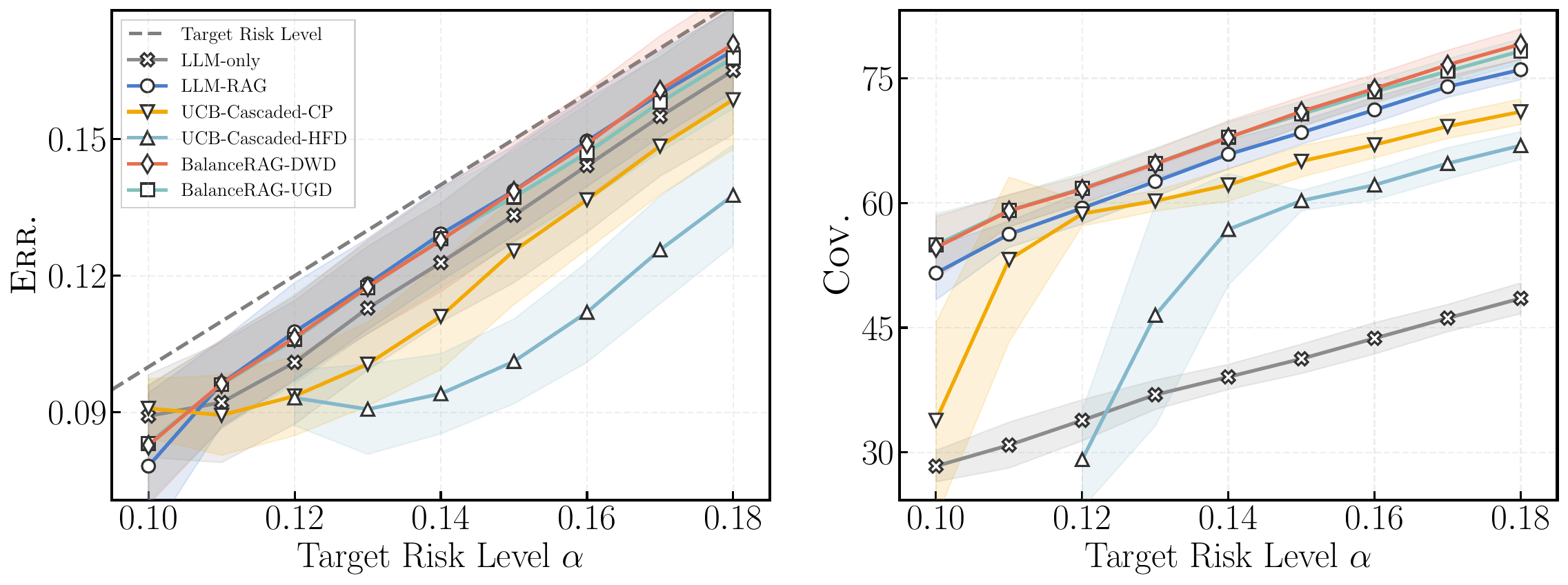}
    \caption{Llama3.2-3B}
  \end{subfigure}%
  \begin{subfigure}[t]{0.5\linewidth}
    \centering
    \includegraphics[width=\linewidth]{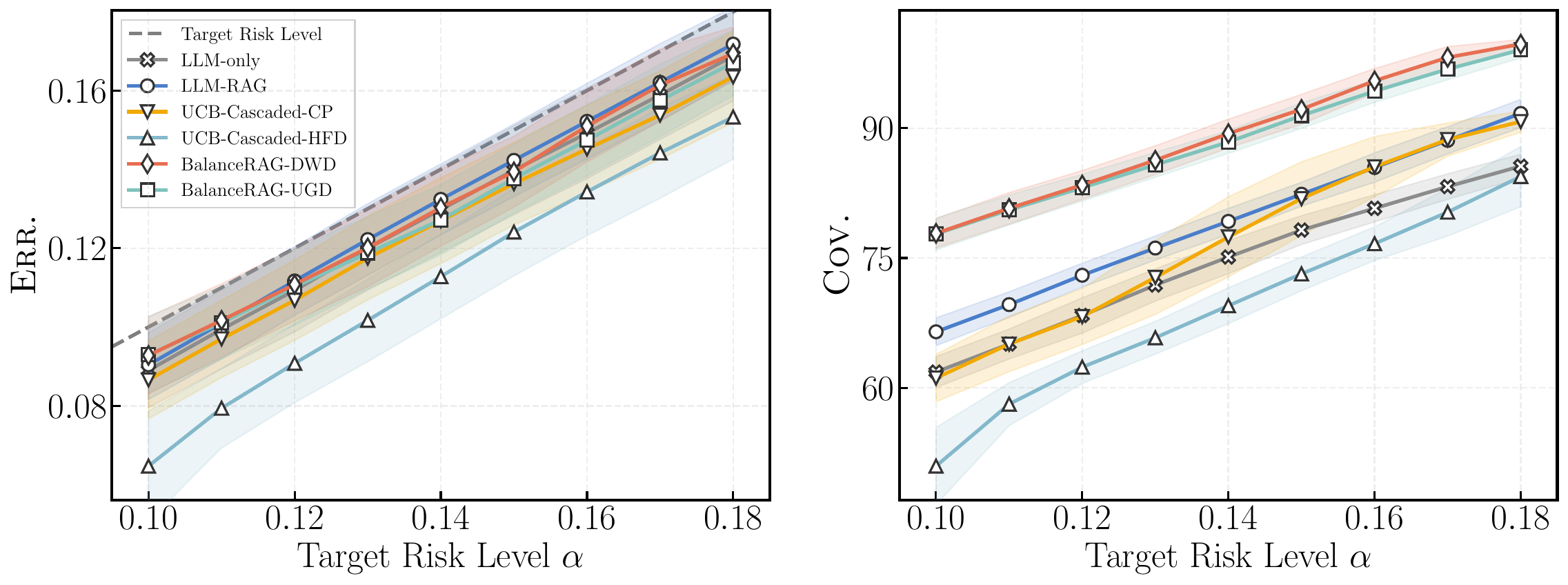}
    \caption{Llama3.1-8B}
  \end{subfigure}

  \caption{Err. and Cov. on TriviaQA with entailment for correctness evaluation.}
  \label{fig:entailment_triviaqa}
  \vspace{-5mm}
\end{figure}

\paragraph{Multi-Risk Calibration Provides a Cost Control Knob}
Table~\ref{tab:balancerag_multirisk_allocation} evaluates the multi-risk variant, where the second risk constraint limits the fallback invocation rate of RAG. Compared with single-risk routing, the multi-risk setting can explicitly control the frequency of RAG calls by adjusting $\alpha_2$. When this constraint becomes stricter, the system reduces retrieval usage and may also lower overall coverage. This result should be understood as a controllable cost-coverage trade-off, rather than a cost-free performance gain. In practical deployment, this setting allows users to adjust system behavior according to the retrieval budget while maintaining calibrated answer reliability.

\begin{table}[htbp]
\centering
\small
\setlength{\tabcolsep}{3.5pt}
\renewcommand{\arraystretch}{1.10}
\caption{
Reliability--coverage comparison on Llama3.1-8B with $\delta=0.1$ over 500 data splits.
We report mean accepted-answer risk, coverage, and guarantee success rate
under target risk levels $\alpha \in \{0.10, 0.11, 0.12\}$.
Success rate denotes the fraction of calibration/test splits where the
empirical test risk is no larger than $\alpha$.
Adaptive-RAG and Self-Route are uncalibrated adaptive routing
baselines that always commit to an answer.
}
\label{tab:llama31_8b_risk_control_comparison}
\begin{tabular}{lccccccccc}
\toprule
\multirow{2}{*}{\raisebox{-0.6ex}{Method}}
& \multicolumn{3}{c}{$\alpha = 0.10$}
& \multicolumn{3}{c}{$\alpha = 0.11$}
& \multicolumn{3}{c}{$\alpha = 0.12$} \\
\cmidrule(lr){2-4}
\cmidrule(lr){5-7}
\cmidrule(lr){8-10}
& Err. $\downarrow$ & Cov. $\uparrow$ & Succ. $\uparrow$
& Err. $\downarrow$ & Cov. $\uparrow$ & Succ. $\uparrow$
& Err. $\downarrow$ & Cov. $\uparrow$ & Succ. $\uparrow$ \\
\midrule

Adaptive-RAG (uncal.)
& 0.1516 & 100.0 & 0.0
& 0.1516 & 100.0 & 0.0
& 0.1516 & 100.0 & 0.0 \\

Self-Route (uncal.)
& 0.1389 & 100.0 & 0.0
& 0.1389 & 100.0 & 0.0
& 0.1389 & 100.0 & 0.0 \\

Empirical Risk $< \alpha$
& 0.1017 & 81.2 & 40.2
& 0.1118 & 87.9 & 41.6
& 0.1227 & 93.5 & 34.8 \\

\midrule

Bonferroni
& 0.0783 & 51.9 & 100.0
& 0.0780 & 57.6 & 100.0
& 0.0884 & 72.0 & 100.0 \\

\rowcolor{GainBlue}
BalanceRAG
& 0.0894 & 71.8 & 94.2
& 0.0997 & 79.7 & 95.4
& 0.1093 & 86.0 & 95.6 \\

\bottomrule
\end{tabular}
\vspace{-2mm}
\end{table}
\paragraph{Additional Baselines}
Table~\ref{tab:llama31_8b_risk_control_comparison} compares BalanceRAG with Adaptive RAG, the Empirical Risk $<\alpha$ baseline, and Bonferroni correction. Adaptive RAG and the Empirical Risk $<\alpha$ baseline have mean risks above the target $\alpha$ and lower guarantee success rates, showing that simple adaptive retrieval or empirical-risk selection does not provide the desired risk control. Bonferroni correction is more reliable, but its coverage is much lower because the correction is more conservative. In contrast, BalanceRAG achieves mean risk below the target level, obtains a high guarantee success rate, and preserves higher coverage than Bonferroni correction. These results show that BalanceRAG strikes a better balance between high-probability reliability and coverage.

\begin{figure}[t]
\vspace{-5mm}
    \centering
    \setlength{\abovecaptionskip}{2pt}
    \setlength{\belowcaptionskip}{0pt}

    \includegraphics[width=0.98\linewidth]{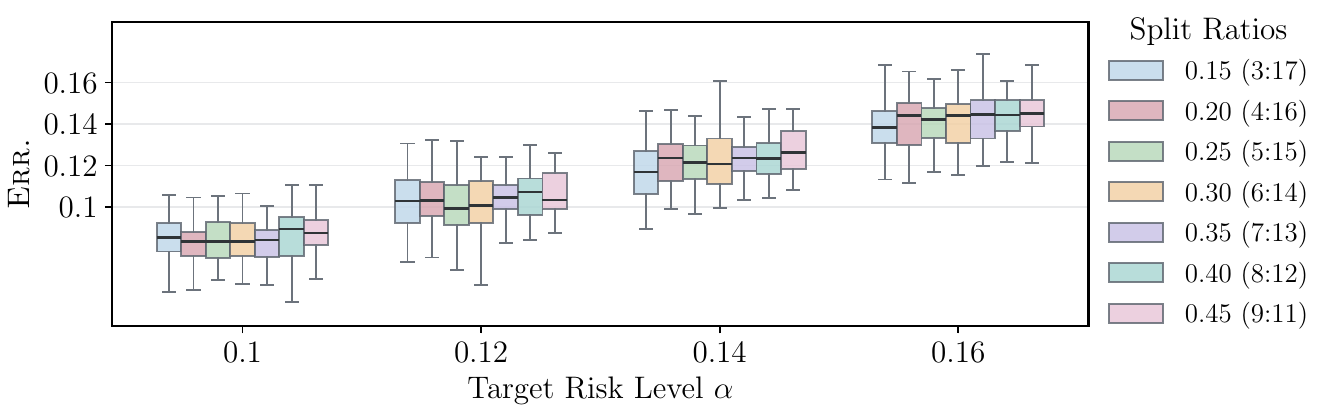}

    \vspace{-0.25em}
    \includegraphics[width=0.98\linewidth]{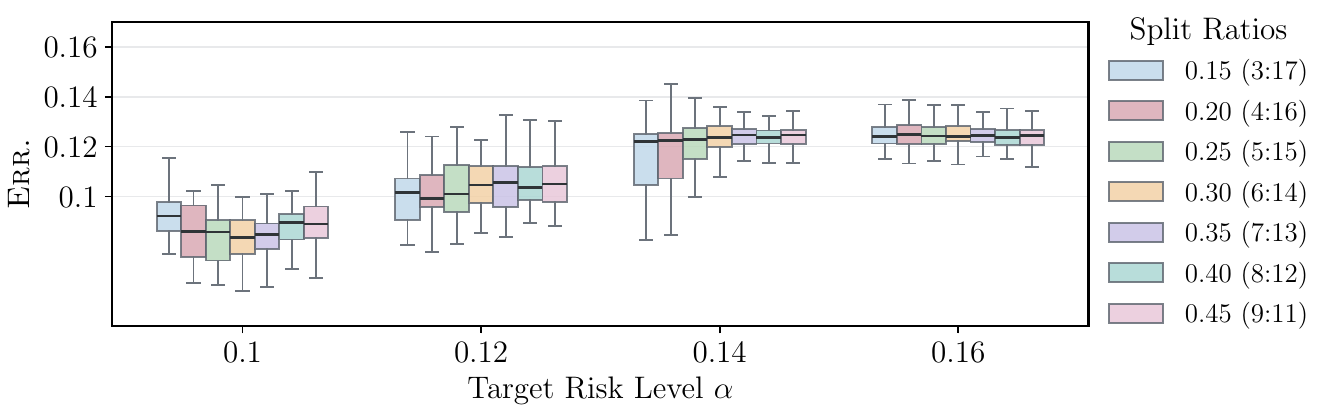}

    \vspace{-0.25em}

    \includegraphics[width=0.98\linewidth]{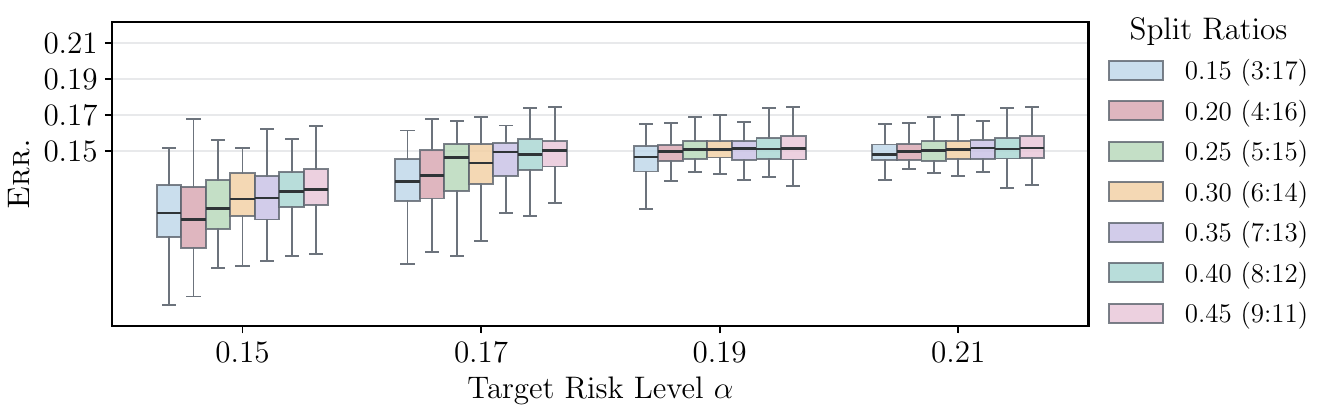}

    \caption{Risk control across various calibration-test split ratios. 
    From top to bottom, the results are reported on TriviaQA with Vicuna-v1.5-7B, 
    SQuAD~v2 with Qwen2.5-7B, and NQ with LLaMA-3.1-8B.}
    \label{fig:split_ratio_all}
    \vspace{-5mm}
\end{figure}

\paragraph{Qualitative Cases Explain Adaptive Routing}
Qualitative examples illustrate the routing behavior behind the aggregate results. In some cases, the LLM-only branch already produces the correct answer, while the retrieved content introduces misleading information. In such cases, BalanceRAG preserves the cheaper direct answer. In other cases, the LLM-only branch makes an error, while LLM-RAG corrects the answer by using external evidence, making routing to retrieval more appropriate. There are also cases where both branches are correct, suggesting that always invoking retrieval is not necessary. These cases do not serve as standalone statistical evidence, but they provide an intuitive explanation of the routing policy. The goal of BalanceRAG is not to favor a fixed branch, but to decide which output can be safely accepted under calibrated risk.

\clearpage
  \begin{table}[htbp]
  \centering
  \small
  \setlength{\tabcolsep}{5.2pt}
  \renewcommand{\arraystretch}{1.08}
  \caption{\textbf{Correctly accepted samples comparison on TriviaQA.}
  We report the number of correctly accepted samples (Corr.) under
  different target risk levels $\alpha$. Bold numbers indicate the best
  Corr. under the same LLM and risk level.}
  \label{tab:corr_comparison_triviaqa}
  \begin{tabular}{llccccccccc}
  \toprule
  \textbf{LLMs} & \textbf{Methods / $\alpha$}
  & \textbf{0.1} & \textbf{0.11} & \textbf{0.12} & \textbf{0.13}
  & \textbf{0.14} & \textbf{0.15} & \textbf{0.16} & \textbf{0.17} &
  \textbf{0.18} \\
  \midrule

  \multirow{5}{*}{Qwen2.5-3B}
  & LLM-only & -- & 798 & 885 & 971 & 1049 & 1137 & 1227 & 1316 & 1397 \\
  & LLM-RAG & \underline{1473} & \underline{1680} & \underline{1907} &
  \underline{2105} & \textbf{2280} & \underline{2400} & \underline{2475} &
  \underline{2500} & \underline{2501} \\
  & UCB-Cascaded-CP & -- & 1291 & 1633 & 1834 & 1952 & 2049 & 2143 & 2237 & 2304
  \\
  & UCB-Cascaded-HFD & -- & -- & -- & -- & 1555 & 1799 & 1955 & 2051 & 2152 \\
  \rowcolor{gray!12}
  & BalanceRAG & \textbf{1569} & \textbf{1733} & \textbf{1907} &
  \textbf{2110} & \underline{2279} & \textbf{2421} & \textbf{2498} &
  \textbf{2512} & \textbf{2514} \\

  \midrule

  \multirow{5}{*}{Qwen2.5-7B}
  & LLM-only & 1856 & 1988 & 2148 & 2255 & 2347 & 2422 & 2473 & 2522 & 2564 \\
  & LLM-RAG & \underline{2118} & \underline{2320} & \underline{2490} &
  \underline{2633} & \underline{2691} & \underline{2694} & \underline{2694} &
  \underline{2694} & \underline{2694} \\
  & UCB-Cascaded-CP & 1820 & 1992 & 2284 & 2377 & 2442 & 2499 & 2531 & 2534 &
  2557 \\
  & UCB-Cascaded-HFD & -- & 1716 & 1880 & 2035 & 2205 & 2301 & 2392 & 2455 &
  2502 \\
  \rowcolor{gray!12}
  & BalanceRAG & \textbf{2366} & \textbf{2573} & \textbf{2702} &
  \textbf{2737} & \textbf{2738} & \textbf{2738} & \textbf{2738} & \textbf{2738}
  & \textbf{2738} \\

  \midrule

  \multirow{5}{*}{Llama3.2-3B}
  & LLM-only & 870 & 988 & 1116 & 1214 & 1283 & 1371 & 1443 & 1513 & 1589 \\
  & LLM-RAG & \textbf{1367} & \underline{1390} & \underline{1532} & 1646 &
  1760 & 1876 & 1949 & 2036 & 2117 \\
  & UCB-Cascaded-CP & 832 & 1028 & 1488 & \underline{1749} & \underline{1850} &
  \underline{1917} & \underline{1991} & \underline{2065} & \underline{2142} \\
  & UCB-Cascaded-HFD & -- & -- & 820 & 964 & 1328 & 1661 & 1799 & 1899 & 1990 \\
  \rowcolor{gray!12}
  & BalanceRAG & \underline{1347} & \textbf{1643} & \textbf{1774} &
  \textbf{1873} & \textbf{1973} & \textbf{2073} & \textbf{2159} & \textbf{2229}
  & \textbf{2290} \\

  \midrule

  \multirow{5}{*}{Llama3.1-8B}
  & LLM-only & 1587 & 1780 & 1928 & 2070 & 2165 & 2218 & 2277 & 2345 & 2402 \\
  & LLM-RAG & \underline{1856} & \underline{2036} & \underline{2198} &
  \underline{2382} & \underline{2529} & \underline{2652} & \underline{2711} &
  \underline{2713} & \underline{2713} \\
  & UCB-Cascaded-CP & 1553 & 1849 & 2084 & 2264 & 2360 & 2434 & 2486 & 2505 &
  2482 \\
  & UCB-Cascaded-HFD & -- & 1395 & 1607 & 1810 & 1992 & 2148 & 2248 & 2351 &
  2400 \\
  \rowcolor{gray!12}
  & BalanceRAG & \textbf{2073} & \textbf{2288} & \textbf{2430} &
  \textbf{2584} & \textbf{2691} & \textbf{2727} & \textbf{2734} & \textbf{2736}
  & \textbf{2737} \\

  \midrule

  \multirow{5}{*}{Vicuna-7B-v1.5}
  & LLM-only & 792 & 868 & 935 & 992 & 1049 & 1115 & 1189 & 1250 & 1294 \\
  & LLM-RAG & \underline{1256} & \underline{1330} & \underline{1401} &
  \underline{1455} & \underline{1512} & \underline{1564} & \underline{1611} &
  \underline{1648} & \underline{1675} \\
  & UCB-Cascaded-CP & 981 & 1102 & 1236 & 1306 & 1354 & 1395 & 1426 & 1469 &
  1504 \\
  & UCB-Cascaded-HFD & 542 & 603 & 776 & 876 & 1011 & 1147 & 1266 & 1333 & 1355
  \\
  \rowcolor{gray!12}
  & BalanceRAG & \textbf{1297} & \textbf{1367} & \textbf{1430} &
  \textbf{1487} & \textbf{1533} & \textbf{1580} & \textbf{1629} & \textbf{1667}
  & \textbf{1694} \\

  \midrule

  \multirow{5}{*}{Vicuna-13B-v1.5}
  & LLM-only & 775 & 861 & 919 & 961 & 1002 & 1043 & 1081 & 1110 & 1138 \\
  & LLM-RAG & \underline{875} & \underline{939} & \underline{987} &
  \underline{1041} & \underline{1101} & \underline{1163} & \underline{1207} &
  \underline{1250} & \underline{1270} \\
  & UCB-Cascaded-CP & 750 & 868 & 925 & 962 & 998 & 1031 & 1070 & 1100 & 1127 \\
  & UCB-Cascaded-HFD & -- & 669 & 694 & 806 & 888 & 946 & 989 & 1035 & 1078 \\
  \rowcolor{gray!12}
  & BalanceRAG & \textbf{983} & \textbf{1041} & \textbf{1091} &
  \textbf{1124} & \textbf{1168} & \textbf{1203} & \textbf{1237} & \textbf{1262}
  & \textbf{1274} \\

  \midrule

  \multirow{5}{*}{Qwen3-4B}
  & LLM-only & -- & 1764 & 1871 & 1995 & 2093 & 2159 & 2233 & 2301 & 2357 \\
  & LLM-RAG & \textbf{2505} & \textbf{2627} & \textbf{2759} &
  \textbf{2835} & \textbf{2849} & \textbf{2849} & \textbf{2849} & \textbf{2849}
  & \textbf{2849} \\
  & UCB-Cascaded-CP & -- & 1746 & 1834 & 2016 & 2126 & 2240 & 2463 & 2533 & 2587
  \\
  & UCB-Cascaded-HFD & -- & -- & 1726 & 1781 & 1902 & 2048 & 2133 & 2224 & 2338
  \\
  \rowcolor{gray!12}
  & BalanceRAG & -- & \underline{2549} & \underline{2718} &
  \underline{2807} & \underline{2839} & \underline{2843} & \underline{2843} &
  \underline{2843} & \underline{2843} \\

  \midrule

  \multirow{5}{*}{Qwen3-14B}
  & LLM-only & -- & -- & 1836 & 1964 & 2073 & 2176 & 2259 & 2314 & 2363 \\
  & LLM-RAG & \textbf{2521} & \textbf{2660} & \textbf{2706} &
  \textbf{2707} & \textbf{2707} & \textbf{2707} & \textbf{2707} & \textbf{2707}
  & \textbf{2707} \\
  & UCB-Cascaded-CP & -- & -- & 2310 & 2362 & 2467 & 2538 & 2533 & 2525 & 2497 \\
  & UCB-Cascaded-HFD & -- & -- & -- & -- & 2058 & 2215 & 2359 & 2410 & 2420 \\
  \rowcolor{gray!12}
  & BalanceRAG & \underline{2499} & \underline{2634} & \underline{2629}
  & \underline{2629} & \underline{2660} & \underline{2661} & \underline{2661} &
  \underline{2661} & \underline{2661} \\

  \bottomrule
  \end{tabular}
  \end{table}
  \begin{table}[htbp]
  \centering
  \small
  \setlength{\tabcolsep}{5.2pt}
  \renewcommand{\arraystretch}{1.08}
  \caption{\textbf{Correctly accepted samples comparison on SQuAD.}
  We report the number of correctly accepted samples (Corr.) under
  different target risk levels $\alpha$. Bold numbers indicate the best
  Corr. under the same LLM and risk level; underlined numbers indicate
  the second best.}
  \label{tab:corr_comparison_squad}
  \begin{tabular}{llccccccccc}
  \toprule
  \textbf{LLMs} & \textbf{Methods / $\alpha$}
  & \textbf{0.1} & \textbf{0.11} & \textbf{0.12} & \textbf{0.13}
  & \textbf{0.14} & \textbf{0.15} & \textbf{0.16} & \textbf{0.17} &
  \textbf{0.18} \\
  \midrule

  \multirow{5}{*}{Qwen2.5-3B}
  & LLM-only & -- & -- & -- & -- & -- & 333 & 362 & 402 & 452 \\
  & LLM-RAG & \textbf{1450} & \textbf{1653} & \textbf{1785} &
  \textbf{1873} & \textbf{1929} & \textbf{1942} & \textbf{1942} & \textbf{1942}
  & \textbf{1942} \\
  & UCB-Cascaded-CP & -- & -- & -- & -- & -- & 1916 & 1927 & 1928 & 1925 \\
  & UCB-Cascaded-HFD & -- & -- & -- & -- & -- & -- & -- & -- & -- \\
  \rowcolor{gray!12}
  & BalanceRAG & \underline{1357} & \underline{1464} & \underline{1649}
  & \underline{1798} & \underline{1887} & \underline{1922} & \underline{1929} &
  \underline{1930} & \underline{1930} \\

  \midrule

  \multirow{5}{*}{Qwen2.5-7B}
  & LLM-only & -- & -- & -- & -- & 443 & 468 & 551 & 682 & 878 \\
  & LLM-RAG & \textbf{1723} & \textbf{1836} & \textbf{1896} &
  \textbf{1924} & \textbf{1929} & \textbf{1929} & \textbf{1929} & \textbf{1929}
  & \textbf{1929} \\
  & UCB-Cascaded-CP & -- & -- & -- & -- & 1913 & 1915 & 1913 & 1906 & 1890 \\
  & UCB-Cascaded-HFD & -- & -- & -- & -- & -- & -- & -- & 1914 & 1914 \\
  \rowcolor{gray!12}
  & BalanceRAG & \underline{1611} & \underline{1698} & \underline{1821}
  & \underline{1895} & \underline{1915} & \underline{1916} & \underline{1916} &
  \underline{1916} & \underline{1916} \\

  \midrule

  \multirow{5}{*}{Llama3.2-3B}
  & LLM-only & -- & -- & -- & -- & -- & 137 & 144 & 164 & 211 \\
  & LLM-RAG & \textbf{1438} & \textbf{1546} & \textbf{1617} &
  \textbf{1666} & \textbf{1714} & \textbf{1755} & \textbf{1784} & \textbf{1809}
  & \textbf{1829} \\
  & UCB-Cascaded-CP & -- & -- & -- & -- & -- & 1744 & 1775 & 1798 & 1818 \\
  & UCB-Cascaded-HFD & -- & -- & -- & -- & -- & -- & -- & -- & -- \\
  \rowcolor{gray!12}
  & BalanceRAG & \underline{1377} & \underline{1484} & \underline{1591}
  & \underline{1654} & \underline{1699} & \underline{1745} & \underline{1781} &
  \underline{1805} & \underline{1828} \\

  \midrule

  \multirow{5}{*}{Llama3.1-8B}
  & LLM-only & -- & -- & -- & -- & 194 & 211 & 227 & 257 & 288 \\
  & LLM-RAG & \textbf{1763} & \textbf{1803} & \textbf{1831} &
  \textbf{1850} & \textbf{1853} & \textbf{1853} & \textbf{1853} & \textbf{1853}
  & \textbf{1853} \\
  & UCB-Cascaded-CP & -- & -- & -- & -- & 1835 & 1837 & 1835 & 1834 & 1830 \\
  & UCB-Cascaded-HFD & -- & -- & -- & -- & -- & -- & -- & -- & -- \\
  \rowcolor{gray!12}
  & BalanceRAG & \underline{1723} & \underline{1784} & \underline{1819}
  & \underline{1836} & \underline{1842} & \underline{1843} & \underline{1843} &
  \underline{1843} & \underline{1843} \\

  \midrule

  \multirow{5}{*}{Vicuna-7B-v1.5}
  & LLM-only & -- & -- & -- & -- & -- & -- & 130 & 138 & 150 \\
  & LLM-RAG & \textbf{1330} & \textbf{1376} & \textbf{1411} &
  \textbf{1443} & \textbf{1464} & \textbf{1479} & \textbf{1484} & \textbf{1484}
  & \textbf{1484} \\
  & UCB-Cascaded-CP & -- & -- & -- & -- & -- & -- & 1467 & 1473 & 1473 \\
  & UCB-Cascaded-HFD & -- & -- & -- & -- & -- & -- & -- & -- & -- \\
  \rowcolor{gray!12}
  & BalanceRAG & \underline{1258} & \underline{1327} & \underline{1374}
  & \underline{1415} & \underline{1447} & \underline{1468} & \underline{1479} &
  \underline{1482} & \underline{1482} \\

  \midrule

  \multirow{5}{*}{Vicuna-13B-v1.5}
  & LLM-only & -- & -- & -- & -- & -- & -- & 128 & 136 & 152 \\
  & LLM-RAG & \textbf{1027} & \textbf{1094} & \textbf{1143} &
  \textbf{1183} & \textbf{1224} & \textbf{1256} & \textbf{1280} & \textbf{1291}
  & \textbf{1292} \\
  & UCB-Cascaded-CP & -- & -- & -- & -- & -- & -- & 1245 & 1266 & 1279 \\
  & UCB-Cascaded-HFD & -- & -- & -- & -- & -- & -- & -- & -- & -- \\
  \rowcolor{gray!12}
  & BalanceRAG & \underline{950} & \underline{1018} & \underline{1097}
  & \underline{1149} & \underline{1189} & \underline{1224} & \underline{1256} &
  \underline{1276} & \underline{1281} \\

  \midrule

  \multirow{5}{*}{Qwen3-4B}
  & LLM-only & -- & -- & -- & -- & -- & 781 & 838 & 948 & 1120 \\
  & LLM-RAG & \textbf{2069} & \textbf{2075} & \textbf{2075} &
  \textbf{2075} & \textbf{2075} & \textbf{2075} & \textbf{2075} & \textbf{2075}
  & \textbf{2075} \\
  & UCB-Cascaded-CP & -- & -- & -- & -- & -- & 2019 & 2018 & 2008 & 1983 \\
  & UCB-Cascaded-HFD & -- & -- & -- & -- & -- & -- & -- & 2019 & 2015 \\
  \rowcolor{gray!12}
  & BalanceRAG & \underline{2053} & \underline{2050} & \underline{2031}
  & \underline{2032} & \underline{2032} & \underline{2032} & \underline{2032} &
  \underline{2032} & \underline{2032} \\

  \midrule

  \multirow{5}{*}{Qwen3-14B}
  & LLM-only & -- & -- & -- & -- & -- & 974 & 1094 & 1251 & 1386 \\
  & LLM-RAG & \textbf{1953} & \textbf{1972} & \textbf{1973} &
  \textbf{1973} & \textbf{1973} & \textbf{1973} & \textbf{1973} & \textbf{1973}
  & \textbf{1973} \\
  & UCB-Cascaded-CP & -- & -- & -- & -- & -- & 1912 & 1903 & 1882 & 1853 \\
  & UCB-Cascaded-HFD & -- & -- & -- & -- & -- & -- & -- & 1910 & 1895 \\
  \rowcolor{gray!12}
  & BalanceRAG & \underline{1942} & \underline{1948} & \underline{1905}
  & \underline{1914} & \underline{1923} & \underline{1924} & \underline{1924} &
  \underline{1924} & \underline{1924} \\

  \bottomrule
  \end{tabular}
  \end{table}
  \begin{table}[htbp]
  \centering
  \small
  \setlength{\tabcolsep}{5.2pt}
  \renewcommand{\arraystretch}{1.08}
  \caption{\textbf{Correctly accepted samples comparison on NQ.}
  We report the number of correctly accepted samples (Corr.) under
  different target risk levels $\alpha$. Bold numbers indicate the best
  Corr. under the same LLM and risk level; underlined numbers indicate
  the second best.}
  \label{tab:corr_comparison_nq}
  \begin{tabular}{llccccccccc}
  \toprule
  \textbf{LLMs} & \textbf{Methods / $\alpha$}
  & \textbf{0.15} & \textbf{0.16} & \textbf{0.17} & \textbf{0.18}
  & \textbf{0.19} & \textbf{0.20} & \textbf{0.21} & \textbf{0.22} &
  \textbf{0.23} \\
  \midrule

  \multirow{5}{*}{Qwen2.5-3B}
  & LLM-only & -- & -- & -- & -- & -- & 168 & 180 & 193 & 209 \\
  & LLM-RAG & \textbf{461} & \textbf{508} & \textbf{565} & \textbf{622} &
  \textbf{667} & \textbf{696} & \textbf{707} & \textbf{708} & \textbf{708} \\
  & UCB-Cascaded-CP & -- & -- & -- & -- & -- & 607 & 645 & 669 & 685 \\
  & UCB-Cascaded-HFD & -- & -- & -- & -- & -- & -- & -- & -- & -- \\
  \rowcolor{gray!12}
  & BalanceRAG & \underline{435} & \underline{470} & \underline{510} &
  \underline{554} & \underline{600} & \underline{646} & \underline{683} &
  \underline{697} & \underline{701} \\

  \midrule

  \multirow{5}{*}{Qwen2.5-7B}
  & LLM-only & -- & -- & -- & -- & -- & 264 & 287 & 309 & 341 \\
  & LLM-RAG & \underline{535} & \textbf{577} & \textbf{623} & \textbf{674}
  & \textbf{707} & \textbf{720} & \textbf{724} & \textbf{725} & \textbf{725} \\
  & UCB-Cascaded-CP & -- & -- & -- & -- & -- & 682 & 706 & 712 & 712 \\
  & UCB-Cascaded-HFD & -- & -- & -- & -- & -- & -- & -- & -- & 703 \\
  \rowcolor{gray!12}
  & BalanceRAG & \textbf{544} & \underline{562} & \underline{595} &
  \underline{628} & \underline{672} & \underline{697} & \underline{715} &
  \underline{716} & \underline{718} \\

  \midrule

  \multirow{5}{*}{Llama3.2-3B}
  & LLM-only & -- & -- & 119 & 125 & 134 & 146 & 159 & 174 & 189 \\
  & LLM-RAG & \textbf{489} & \textbf{532} & \textbf{562} & \textbf{590} &
  \textbf{615} & \textbf{640} & \textbf{661} & \textbf{680} & \textbf{694} \\
  & UCB-Cascaded-CP & -- & -- & -- & 563 & 590 & 615 & 635 & 653 & 669 \\
  & UCB-Cascaded-HFD & -- & -- & -- & -- & -- & -- & -- & -- & 633 \\
  \rowcolor{gray!12}
  & BalanceRAG & \underline{447} & \underline{492} & \underline{531} &
  \underline{575} & \underline{615} & \underline{634} & \underline{657} &
  \underline{674} & \underline{694} \\

  \midrule

  \multirow{5}{*}{Llama3.1-8B}
  & LLM-only & 158 & 191 & 240 & 283 & 320 & 345 & 365 & 382 & 401 \\
  & LLM-RAG & \textbf{734} & \textbf{752} & \textbf{760} & \textbf{763} &
  \textbf{763} & \textbf{763} & \textbf{763} & \textbf{763} & \textbf{763} \\
  & UCB-Cascaded-CP & 698 & 718 & 736 & 743 & 747 & 744 & 740 & 735 & 730 \\
  & UCB-Cascaded-HFD & -- & -- & -- & 717 & 735 & 749 & 753 & 750 & 744 \\
  \rowcolor{gray!12}
  & BalanceRAG & \underline{709} & \underline{741} & \underline{753} &
  \underline{757} & \underline{759} & \underline{759} & \underline{759} &
  \underline{759} & \underline{759} \\

  \midrule

  \multirow{5}{*}{Vicuna-7B-v1.5}
  & LLM-only & -- & -- & -- & -- & -- & -- & -- & 88 & 90 \\
  & LLM-RAG & \textbf{541} & \textbf{564} & \textbf{584} & \textbf{600} &
  \textbf{615} & \textbf{629} & \textbf{639} & \textbf{643} & \textbf{644} \\
  & UCB-Cascaded-CP & -- & -- & -- & -- & -- & -- & -- & -- & 629 \\
  & UCB-Cascaded-HFD & -- & -- & -- & -- & -- & -- & -- & -- & -- \\
  \rowcolor{gray!12}
  & BalanceRAG & \underline{478} & \underline{508} & \underline{536} &
  \underline{556} & \underline{582} & \underline{601} & \underline{617} &
  \underline{628} & \underline{635} \\

  \midrule

  \multirow{5}{*}{Vicuna-13B-v1.5}
  & LLM-only & -- & -- & -- & -- & -- & -- & -- & -- & -- \\
  & LLM-RAG & \textbf{387} & \textbf{410} & \textbf{430} & \textbf{448} &
  \textbf{462} & \textbf{471} & \textbf{478} & \textbf{482} & \textbf{483} \\
  & UCB-Cascaded-CP & -- & -- & -- & -- & -- & -- & -- & -- & -- \\
  & UCB-Cascaded-HFD & -- & -- & -- & -- & -- & -- & -- & -- & -- \\
  \rowcolor{gray!12}
  & BalanceRAG & \underline{373} & \underline{400} & \underline{413} &
  \underline{418} & \underline{416} & \underline{427} & \underline{439} &
  \underline{449} & \underline{460} \\

  \midrule

  \multirow{5}{*}{Qwen3-4B}
  & LLM-only & -- & -- & -- & -- & -- & -- & -- & -- & -- \\
  & LLM-RAG & \textbf{799} & \textbf{804} & \textbf{806} & \textbf{806} &
  \textbf{806} & \textbf{806} & \textbf{806} & \textbf{806} & \textbf{806} \\
  & UCB-Cascaded-CP & -- & -- & -- & -- & -- & -- & -- & -- & -- \\
  & UCB-Cascaded-HFD & -- & -- & -- & -- & -- & -- & -- & -- & -- \\
  \rowcolor{gray!12}
  & BalanceRAG & \underline{792} & \underline{799} & \underline{802} &
  \underline{797} & \underline{785} & \underline{780} & \underline{779} &
  \underline{778} & \underline{777} \\

  \midrule

  \multirow{5}{*}{Qwen3-14B}
  & LLM-only & -- & -- & -- & -- & -- & -- & 383 & 397 & 421 \\
  & LLM-RAG & \textbf{747} & \textbf{772} & \textbf{777} & \textbf{778} &
  \textbf{778} & \textbf{778} & \textbf{778} & \textbf{778} & \textbf{778} \\
  & UCB-Cascaded-CP & -- & -- & -- & -- & -- & -- & 740 & 741 & 738 \\
  & UCB-Cascaded-HFD & -- & -- & -- & -- & -- & -- & -- & -- & -- \\
  \rowcolor{gray!12}
  & BalanceRAG & \underline{727} & \underline{761} & \underline{769} &
  \underline{763} & \underline{757} & \underline{752} & \underline{752} &
  \underline{753} & \underline{752} \\

  \bottomrule
  \end{tabular}
  \end{table}
\begin{table}[htbp]
\centering
\scriptsize
\setlength{\tabcolsep}{4.2pt}
\renewcommand{\arraystretch}{1.20}
\arrayrulecolor{black!45}

\caption{
\textbf{Route allocation (\%) of BalanceRAG.}
We report the allocation ratio of accepted samples routed to the
LLM-only branch (Only) and the RAG branch.
Cov. denotes the total accepted ratio, i.e., Only+RAG    .
}
\label{tab:balancerag_allocation}

\resizebox{\textwidth}{!}{
\begin{tabular}{llcccccccccccc}
\toprule

\rowcolor{blue!18}
\textbf{Dataset}
& \textbf{LLMs}
& \multicolumn{3}{c}{$\alpha = 0.12$}
& \multicolumn{3}{c}{$\alpha = 0.14$}
& \multicolumn{3}{c}{$\alpha = 0.16$}
& \multicolumn{3}{c}{$\alpha = 0.18$} \\

\rowcolor{blue!10}
& & \textbf{Only} & \textbf{RAG} & \textbf{Cov.}
& \textbf{Only} & \textbf{RAG} & \textbf{Cov.}
& \textbf{Only} & \textbf{RAG} & \textbf{Cov.}
& \textbf{Only} & \textbf{RAG} & \textbf{Cov.} \\

\midrule

\rowcolor{blue!3}
& Qwen2.5-3B      & 27.7 & 44.5 & 72.2 & 29.9 & 58.3 & 88.2 & 30.7 & 68.4 & 99.1 & 31.1 & 68.9 & 100.0 \\
\rowcolor{blue!3}
& Qwen2.5-7B      & 62.2 & 35.8 & 98.0 & 64.1 & 35.8 & 99.9 & 64.1 & 35.8 & 99.9 & 64.1 & 35.8 & 99.9 \\
\rowcolor{blue!3}
& LLaMA-3.2-3B    & 37.1 & 26.7 & 63.8 & 41.1 & 31.4 & 72.5 & 41.6 & 39.7 & 81.3 & 41.5 & 46.7 & 88.2 \\
\rowcolor{blue!3}
\textbf{TriviaQA} & Vicuna-7B-v1.5  & 33.5 & 43.2 & 76.7 & 33.2 & 50.6 & 83.8 & 35.0 & 56.2 & 91.2 & 35.0 & 62.2 & 97.2 \\
\rowcolor{blue!3}
& Vicuna-13B-v1.5 & 51.9 & 28.2 & 80.1 & 54.5 & 33.2 & 87.7 & 54.7 & 40.1 & 94.8 & 54.7 & 45.0 & 99.7 \\
\rowcolor{blue!3}
& LLaMA-3.1-8B    & 51.2 & 34.5 & 85.7 & 53.3 & 44.1 & 97.4 & 62.1 & 37.8 & 99.9 & 63.1 & 36.8 & 99.9 \\
\rowcolor{blue!3}
& Qwen3-4B        & 32.4 & 61.7 & 94.1 & 58.7 & 41.2 & 99.9 & 58.8 & 41.2 & 100.0 & 58.8 & 41.2 & 100.0 \\
\rowcolor{blue!3}
& Qwen3-14B       & 29.8 & 67.7 & 97.5 & 64.0 & 35.9 & 99.9 & 64.1 & 35.8 & 99.9 & 64.1 & 35.8 & 99.9 \\

\addlinespace[3pt]
\midrule

\rowcolor{green!18}
\textbf{Dataset}
& \textbf{LLMs}
& \multicolumn{3}{c}{$\alpha = 0.12$}
& \multicolumn{3}{c}{$\alpha = 0.14$}
& \multicolumn{3}{c}{$\alpha = 0.16$}
& \multicolumn{3}{c}{$\alpha = 0.18$} \\

\rowcolor{green!10}
& & \textbf{Only} & \textbf{RAG} & \textbf{Cov.}
& \textbf{Only} & \textbf{RAG} & \textbf{Cov.}
& \textbf{Only} & \textbf{RAG} & \textbf{Cov.}
& \textbf{Only} & \textbf{RAG} & \textbf{Cov.} \\

\midrule

\rowcolor{green!3}
& Qwen2.5-3B      & 15.8 & 66.9 & 82.7 & 17.3 & 79.5 & 96.8 & 18.3 & 81.6 & 99.9 & 18.8 & 81.1 & 99.9 \\
\rowcolor{green!3}
& Qwen2.5-7B      & 21.7 & 71.3 & 93.0 & 24.0 & 75.9 & 99.9 & 25.1 & 74.8 & 99.9 & 25.1 & 74.8 & 99.9 \\
\rowcolor{green!3}
& LLaMA-3.2-3B    & 7.4 & 72.1 & 79.5 & 8.4 & 78.5 & 86.9 & 7.8 & 85.3 & 93.1 & 8.0 & 90.1 & 98.1 \\
\rowcolor{green!3}
\textbf{SQuAD} & Vicuna-7B-v1.5  & 6.8 & 81.9 & 88.7 & 6.4 & 89.0 & 95.4 & 6.4 & 93.2 & 99.6 & 6.5 & 93.4 & 99.9 \\
\rowcolor{green!3}
& Vicuna-13B-v1.5 & 9.7 & 71.0 & 80.7 & 9.7 & 79.9 & 89.6 & 10.2 & 86.6 & 96.8 & 10.0 & 89.9 & 99.9 \\
\rowcolor{green!3}
& LLaMA-3.1-8B    & 7.7 & 89.6 & 97.3 & 8.0 & 91.9 & 99.9 & 8.4 & 91.6 & 100.0 & 8.4 & 91.6 & 100.0 \\
\rowcolor{green!3}
& Qwen3-4B        & 29.2 & 70.2 & 99.4 & 38.6 & 61.4 & 100.0 & 38.6 & 61.4 & 100.0 & 38.6 & 61.4 & 100.0 \\
\rowcolor{green!3}
& Qwen3-14B       & 28.3 & 69.2 & 97.5 & 45.6 & 54.3 & 99.9 & 45.7 & 54.3 & 100.0 & 45.7 & 54.3 & 100.0 \\

\addlinespace[3pt]
\midrule

\rowcolor{orange!20}
\textbf{Dataset}
& \textbf{LLMs}
& \multicolumn{3}{c}{$\alpha = 0.17$}
& \multicolumn{3}{c}{$\alpha = 0.19$}
& \multicolumn{3}{c}{$\alpha = 0.21$}
& \multicolumn{3}{c}{$\alpha = 0.23$} \\

\rowcolor{orange!10}
& & \textbf{Only} & \textbf{RAG} & \textbf{Cov.}
& \textbf{Only} & \textbf{RAG} & \textbf{Cov.}
& \textbf{Only} & \textbf{RAG} & \textbf{Cov.}
& \textbf{Only} & \textbf{RAG} & \textbf{Cov.} \\

\midrule

\rowcolor{orange!3}
& Qwen2.5-3B      & 13.8 & 55.4 & 69.2 & 20.0 & 63.5 & 83.5 & 20.9 & 76.0 & 96.9 & 20.7 & 79.1 & 99.8 \\
\rowcolor{orange!3}
& Qwen2.5-7B      & 17.8 & 62.2 & 80.0 & 25.0 & 67.2 & 92.2 & 34.9 & 64.4 & 99.3 & 36.6 & 63.2 & 99.8 \\
\rowcolor{orange!3}
& LLaMA-3.2-3B    & 15.1 & 53.4 & 68.5 & 14.3 & 67.4 & 81.7 & 17.5 & 72.0 & 89.5 & 18.0 & 78.8 & 96.8 \\
\rowcolor{orange!3}
\textbf{NQ} & Vicuna-7B-v1.5  & 12.4 & 66.0 & 78.4 & 12.2 & 75.2 & 87.4 & 12.4 & 82.6 & 95.0 & 11.9 & 87.5 & 99.4 \\
\rowcolor{orange!3}
& Vicuna-13B-v1.5 & 1.9 & 78.8 & 80.7 & 13.2 & 70.3 & 83.5 & 19.6 & 70.3 & 89.9 & 23.5 & 73.3 & 96.8 \\
\rowcolor{orange!3}
& LLaMA-3.1-8B    & 22.4 & 76.6 & 99.0 & 24.5 & 75.3 & 99.8 & 25.1 & 74.7 & 99.8 & 25.1 & 74.7 & 99.8 \\
\rowcolor{orange!3}
& Qwen3-4B        & 1.9 & 97.7 & 99.6 & 26.4 & 73.3 & 99.7 & 39.8 & 60.0 & 99.8 & 43.2 & 56.7 & 99.9 \\
\rowcolor{orange!3}
& Qwen3-14B       & 5.1 & 93.9 & 99.0 & 25.0 & 74.1 & 99.1 & 46.1 & 53.7 & 99.8 & 48.7 & 51.1 & 99.8 \\

\bottomrule
\end{tabular}
}

\arrayrulecolor{black}
\end{table}
\begin{table}[htbp]
\centering
\scriptsize
\setlength{\tabcolsep}{3.4pt}
\renewcommand{\arraystretch}{1.18}
\arrayrulecolor{black!45}

\caption{
\textbf{Route allocation (\%) of multi-risk BalanceRAG.}
We report the allocation ratio of accepted samples routed to the
LLM-only branch (Only) and the RAG branch.
$\alpha_2$ denotes the fallback invocation rate cap, chosen as the nearest
lower multiple of 10\% below the corresponding risk-only RAG allocation.
Cov. denotes the total accepted ratio.
}
\label{tab:balancerag_multirisk_allocation}

\resizebox{\textwidth}{!}{
\begin{tabular}{llcccccccccccccccc}
\toprule

\rowcolor{blue!18}
\textbf{Dataset}
& \textbf{LLMs}
& \multicolumn{4}{c}{$\alpha = 0.12$}
& \multicolumn{4}{c}{$\alpha = 0.14$}
& \multicolumn{4}{c}{$\alpha = 0.16$}
& \multicolumn{4}{c}{$\alpha = 0.18$} \\

\rowcolor{blue!10}
& & $\boldsymbol{\alpha_2}$ & \textbf{Only} & \textbf{RAG} & \textbf{Cov.}
& $\boldsymbol{\alpha_2}$ & \textbf{Only} & \textbf{RAG} & \textbf{Cov.}
& $\boldsymbol{\alpha_2}$ & \textbf{Only} & \textbf{RAG} & \textbf{Cov.}
& $\boldsymbol{\alpha_2}$ & \textbf{Only} & \textbf{RAG} & \textbf{Cov.} \\

\midrule

\rowcolor{blue!3}
& Qwen2.5-3B & 40 & 34.2 & \textbf{34.0} & 68.2 & 50 & 35.3 & \textbf{46.8} & 82.2 & 60 & 37.1 & \textbf{57.7} & 94.8 & 60 & 42.7 & \textbf{57.1} & 99.7 \\
\rowcolor{blue!3}
& Qwen2.5-7B & 30 & 70.1 & \textbf{25.0} & 95.0 & 30 & 74.8 & \textbf{25.1} & 99.9 & 30 & 75.7 & \textbf{24.2} & 100.0 & 30 & 75.8 & \textbf{24.1} & 100.0 \\
\rowcolor{blue!3}
& LLaMA-3.2-3B & 20 & 38.0 & \textbf{0.4} & 38.4 & 30 & 45.3 & \textbf{25.5} & 70.8 & 30 & 51.9 & \textbf{25.7} & 77.6 & 40 & 51.1 & \textbf{35.1} & 86.2 \\
\rowcolor{blue!3}
\textbf{TriviaQA} & Vicuna-7B-v1.5 & 40 & 40.6 & \textbf{34.2} & 74.8 & 50 & 37.9 & \textbf{45.2} & 83.1 & 50 & 42.9 & \textbf{45.9} & 88.8 & 60 & 40.5 & \textbf{55.8} & 96.3 \\
\rowcolor{blue!3}
& Vicuna-13B-v1.5 & 20 & 62.9 & \textbf{13.1} & 76.0 & 30 & 63.8 & \textbf{21.3} & 85.0 & 40 & 63.3 & \textbf{30.5} & 93.8 & 40 & 65.4 & \textbf{33.8} & 99.2 \\
\rowcolor{blue!3}
& LLaMA-3.1-8B & 30 & 59.4 & \textbf{24.1} & 83.5 & 40 & 61.7 & \textbf{33.4} & 95.1 & 30 & 74.0 & \textbf{25.8} & 99.8 & 30 & 74.6 & \textbf{25.4} & 100.0 \\
\rowcolor{blue!3}
& Qwen3-4B & 60 & 59.1 & \textbf{31.2} & 90.3 & 40 & 62.4 & \textbf{36.7} & 99.2 & 40 & 63.7 & \textbf{36.3} & 100.0 & 40 & 63.8 & \textbf{36.2} & 100.0 \\
\rowcolor{blue!3}
& Qwen3-14B & 60 & 63.9 & \textbf{30.0} & 93.9 & 30 & 71.8 & \textbf{27.2} & 99.0 & 30 & 73.0 & \textbf{27.0} & 100.0 & 30 & 73.2 & \textbf{26.8} & 100.0 \\

\addlinespace[3pt]
\midrule

\rowcolor{green!18}
\textbf{Dataset}
& \textbf{LLMs}
& \multicolumn{4}{c}{$\alpha = 0.12$}
& \multicolumn{4}{c}{$\alpha = 0.14$}
& \multicolumn{4}{c}{$\alpha = 0.16$}
& \multicolumn{4}{c}{$\alpha = 0.18$} \\

\rowcolor{green!10}
& & $\boldsymbol{\alpha_2}$ & \textbf{Only} & \textbf{RAG} & \textbf{Cov.}
& $\boldsymbol{\alpha_2}$ & \textbf{Only} & \textbf{RAG} & \textbf{Cov.}
& $\boldsymbol{\alpha_2}$ & \textbf{Only} & \textbf{RAG} & \textbf{Cov.}
& $\boldsymbol{\alpha_2}$ & \textbf{Only} & \textbf{RAG} & \textbf{Cov.} \\

\midrule

\rowcolor{green!3}
& Qwen2.5-3B & 60 & 19.4 & \textbf{55.5} & 74.9 & 70 & 25.1 & \textbf{66.4} & 91.5 & 80 & 24.9 & \textbf{74.9} & 99.8 & 80 & 25.4 & \textbf{74.5} & 99.9 \\
\rowcolor{green!3}
& Qwen2.5-7B & 70 & 25.6 & \textbf{65.3} & 90.9 & 70 & 33.1 & \textbf{65.9} & 99.0 & 70 & 35.1 & \textbf{64.8} & 99.9 & 70 & 35.6 & \textbf{64.4} & 100.0 \\
\rowcolor{green!3}
& LLaMA-3.2-3B & 70 & 10.4 & \textbf{65.9} & 76.3 & 70 & 18.8 & \textbf{64.6} & 83.5 & 80 & 17.4 & \textbf{72.6} & 90.0 & 90 & 10.3 & \textbf{86.1} & 96.4 \\
\rowcolor{green!3}
\textbf{SQuAD} & Vicuna-7B-v1.5 & 80 & 9.3 & \textbf{77.1} & 86.4 & 80 & 12.9 & \textbf{77.2} & 90.1 & 90 & 9.9 & \textbf{88.2} & 98.1 & 90 & 12.7 & \textbf{87.2} & 99.9 \\
\rowcolor{green!3}
& Vicuna-13B-v1.5 & 70 & 11.1 & \textbf{64.9} & 76.0 & 70 & 18.0 & \textbf{65.4} & 83.4 & 80 & 17.2 & \textbf{75.9} & 93.0 & 80 & 21.6 & \textbf{76.0} & 97.5 \\
\rowcolor{green!3}
& LLaMA-3.1-8B & 80 & 16.4 & \textbf{76.7} & 93.0 & 90 & 12.3 & \textbf{87.4} & 99.7 & 90 & 13.9 & \textbf{86.0} & 99.9 & 90 & 14.1 & \textbf{85.8} & 99.9 \\
\rowcolor{green!3}
& Qwen3-4B & 70 & 38.3 & \textbf{60.2} & 98.5 & 60 & 44.1 & \textbf{55.8} & 99.9 & 60 & 44.5 & \textbf{55.4} & 100.0 & 60 & 44.7 & \textbf{55.2} & 100.0 \\
\rowcolor{green!3}
& Qwen3-14B & 60 & 44.4 & \textbf{50.5} & 94.9 & 50 & 52.8 & \textbf{46.6} & 99.4 & 50 & 54.5 & \textbf{45.4} & 99.9 & 50 & 54.7 & \textbf{45.2} & 100.0 \\

\addlinespace[3pt]
\midrule

\rowcolor{orange!20}
\textbf{Dataset}
& \textbf{LLMs}
& \multicolumn{4}{c}{$\alpha = 0.17$}
& \multicolumn{4}{c}{$\alpha = 0.19$}
& \multicolumn{4}{c}{$\alpha = 0.21$}
& \multicolumn{4}{c}{$\alpha = 0.23$} \\

\rowcolor{orange!10}
& & $\boldsymbol{\alpha_2}$ & \textbf{Only} & \textbf{RAG} & \textbf{Cov.}
& $\boldsymbol{\alpha_2}$ & \textbf{Only} & \textbf{RAG} & \textbf{Cov.}
& $\boldsymbol{\alpha_2}$ & \textbf{Only} & \textbf{RAG} & \textbf{Cov.}
& $\boldsymbol{\alpha_2}$ & \textbf{Only} & \textbf{RAG} & \textbf{Cov.} \\

\midrule

\rowcolor{orange!3}
& Qwen2.5-3B & 50 & 21.3 & \textbf{39.7} & 61.0 & 60 & 24.2 & \textbf{50.8} & 75.0 & 70 & 26.0 & \textbf{63.8} & 89.8 & 70 & 32.3 & \textbf{64.2} & 96.5 \\
\rowcolor{orange!3}
& Qwen2.5-7B & 60 & 35.0 & \textbf{42.9} & 77.9 & 60 & 37.3 & \textbf{49.6} & 87.0 & 60 & 43.1 & \textbf{52.4} & 95.5 & 60 & 48.7 & \textbf{50.5} & 99.2 \\
\rowcolor{orange!3}
& LLaMA-3.2-3B & 50 & 20.1 & \textbf{41.6} & 61.6 & 60 & 20.5 & \textbf{53.7} & 74.2 & 70 & 21.9 & \textbf{63.0} & 84.9 & 70 & 27.1 & \textbf{63.3} & 90.4 \\
\rowcolor{orange!3}
\textbf{NQ} & Vicuna-7B-v1.5 & 60 & 15.7 & \textbf{53.9} & 69.5 & 70 & 16.4 & \textbf{64.4} & 80.8 & 80 & 15.3 & \textbf{75.7} & 91.0 & 80 & 18.5 & \textbf{76.3} & 94.8 \\
\rowcolor{orange!3}
& Vicuna-13B-v1.5 & 70 & 3.3 & \textbf{63.2} & 66.6 & 70 & 15.9 & \textbf{57.7} & 73.6 & 70 & 23.4 & \textbf{60.9} & 84.4 & 70 & 29.1 & \textbf{62.3} & 91.4 \\
\rowcolor{orange!3}
& LLaMA-3.1-8B & 70 & 33.0 & \textbf{61.6} & 94.5 & 70 & 37.3 & \textbf{61.9} & 99.2 & 70 & 38.8 & \textbf{61.0} & 99.8 & 70 & 38.9 & \textbf{60.9} & 99.8 \\
\rowcolor{orange!3}
& Qwen3-4B & 90 & 9.2 & \textbf{80.1} & 89.3 & 70 & 43.3 & \textbf{55.4} & 98.7 & 50 & 52.5 & \textbf{44.5} & 97.0 & 50 & 55.2 & \textbf{44.1} & 99.3 \\
\rowcolor{orange!3}
& Qwen3-14B & 90 & 11.2 & \textbf{78.0} & 89.2 & 70 & 50.0 & \textbf{47.4} & 97.5 & 50 & 54.9 & \textbf{44.0} & 98.9 & 50 & 56.1 & \textbf{43.6} & 99.7 \\

\bottomrule
\end{tabular}
}

\arrayrulecolor{black}
\end{table}
\newpage


\begin{figure}[H]
\centering
\begin{minipage}{0.94\linewidth}

\begin{tcolorbox}[onlyhelpbox]
\exampletitle{Case 1: LLM-only is correct while RAG is misled}

\noindent
\textbf{ID:} sfq\_8482\_5657
\hfill
$\alpha = 0.10$
\hfill
$t_1 = 0.286$
\hfill
$t_2 = 0.748$

\vspace{2pt}
\noindent
$u_{Only} = 0.286$
\hfill
$u_{RAG} = 2.466$
\hfill
\textbf{Route:} Only \cmark

\examplesep

\noindent\textbf{Question} \\
Charleston is the capital of which US state?

\examplesep

\noindent\textbf{Gold Answer} \\
West Virginia

\examplesep

\noindent\textbf{LLM-only Answer \cmark} \\
Charleston is the capital of West Virginia.

\examplesep

\noindent\textbf{RAG Answer} \\
Charleston is not the capital of any US state.
\end{tcolorbox}

\begin{tcolorbox}[onlyhelpbox]
\exampletitle{Case 2: LLM-only gives the exact answer while RAG rejects it}

\noindent
\textbf{ID:} qz\_453\_1207
\hfill
$\alpha = 0.10$
\hfill
$t_1 = 0.748$
\hfill
$t_2 = 0.469$

\vspace{2pt}
\noindent
$u_{Only} = 0.469$
\hfill
$u_{RAG} = 2.020$
\hfill
\textbf{Route:}Only \cmark

\examplesep

\noindent\textbf{Question} \\
For what novel did J. K. Rowling win the 1999 Whitbread Children's book of the year award?

\examplesep

\noindent\textbf{Gold Answer} \\
Harry Potter and the Prisoner of Azkaban

\examplesep

\noindent\textbf{LLM-only Answer \cmark} \\
Harry Potter and the Prisoner of Azkaban.

\examplesep

\noindent\textbf{RAG Answer} \\
She did not win the 1999 Whitbread Children's book of the year award.
\end{tcolorbox}

\end{minipage}

\caption{
Qualitative examples where BalanceRAG preserves the LLM-only branch.
In both examples, the LLM-only answer is correct and passes the first-stage uncertainty threshold,
while retrieval introduces misleading evidence and produces an incorrect RAG answer.
}
\label{fig:qualitative_only_correct_rag_wrong}
\end{figure}


\begin{figure}[htbp]
\centering
\begin{minipage}{0.94\linewidth}

\begin{tcolorbox}[raghelpbox]
\exampletitle{Case 3: RAG corrects a wrong entity generated by LLM-only}

\noindent
\textbf{ID:} bt\_4365\_238
\hfill
$\alpha = 0.10$
\hfill
$t_1 = 0.286$
\hfill
$t_2 = 0.748$

\vspace{2pt}
\noindent
$u_{Only} = 4.222$
\hfill
$u_{RAG} = 0.286$
\hfill
\textbf{Route:} RAG \cmark

\examplesep

\noindent\textbf{Question} \\
Who in verse sailed to the land where the bong tree grows?

\examplesep

\noindent\textbf{Gold Answer} \\
The owl and the pussycat

\examplesep

\noindent\textbf{LLM-only Answer} \\
Alfred Noyes in his poem ``The Highwayman'' sailed to the land where the bong tree grows.

\examplesep

\noindent\textbf{RAG Answer \cmark} \\
The Owl and the Pussy Cat sailed to the land where the bong tree grows.
\end{tcolorbox}

\begin{tcolorbox}[raghelpbox]
\exampletitle{Case 4: RAG fixes an incorrect surname answer}

\noindent
\textbf{ID:} dpql\_4593\_2651
\hfill
$\alpha = 0.10$
\hfill
$t_1 = 0.286$
\hfill
$t_2 = 0.748$

\vspace{2pt}
\noindent
$u_{Only} = 3.304$
\hfill
$u_{RAG} = 0.286$
\hfill
\textbf{Route:} RAG \cmark

\examplesep

\noindent\textbf{Question} \\
What is the surname of the title character of Anne of Green Gables?

\examplesep

\noindent\textbf{Gold Answer} \\
Shirley

\examplesep

\noindent\textbf{LLM-only Answer} \\
Cuthbert is not the surname of the title character of Anne of Green Gables.

\examplesep

\noindent\textbf{RAG Answer \cmark} \\
The surname of the title character of Anne of Green Gables is Shirley.
\end{tcolorbox}

\end{minipage}

\caption{
Qualitative examples where BalanceRAG routes to the RAG branch.
The LLM-only branch has high uncertainty and gives an incorrect answer, while the RAG branch
has lower uncertainty and produces the correct answer.
}
\label{fig:qualitative_only_wrong_rag_correct}
\end{figure}


\begin{figure}[htbp]
\centering
\begin{minipage}{0.94\linewidth}

\begin{tcolorbox}[bothcorrectbox]
\exampletitle{Case 5: Both branches are correct, but LLM-only is selected earlier}

\noindent
\textbf{ID:} dpql\_6316\_2767
\hfill
$\alpha = 0.10$
\hfill
$t_1 = 0.286$
\hfill
$t_2 = 0.748$

\vspace{2pt}
\noindent
$u_{Only} = 0.286$
\hfill
$u_{RAG} = 2.139$
\hfill
\textbf{Route:} Only \cmark

\examplesep

\noindent\textbf{Question} \\
Who co-starred with George Clooney in the 2013 movie Gravity?

\examplesep

\noindent\textbf{Gold Answer} \\
Sandra Bullock

\examplesep

\noindent\textbf{LLM-only Answer \cmark} \\
Sandra Bullock co-starred with George Clooney in the 2013 movie Gravity.

\examplesep

\noindent\textbf{RAG Answer} \\
Sandra Bullock co-starred with George Clooney in the 2013 movie Gravity.
\end{tcolorbox}

\begin{tcolorbox}[bothcorrectbox]
\exampletitle{Case 6: Both branches answer correctly, so the cheaper LLM-only route is sufficient}

\noindent
\textbf{ID:} qf\_3204\_522
\hfill
$\alpha = 0.10$
\hfill
$t_1 = 0.286$
\hfill
$t_2 = 0.748$

\vspace{2pt}
\noindent
$u_{Only} = 0.286$
\hfill
$u_{RAG} = 2.020$
\hfill
\textbf{Route:} Only \cmark

\examplesep

\noindent\textbf{Question} \\
If you were eating a Knickerbocker Glory, what would you be eating?

\examplesep

\noindent\textbf{Gold Answer} \\
An ice-cream sundae

\examplesep

\noindent\textbf{LLM-only Answer \cmark} \\
A Knickerbocker Glory is a type of ice cream sundae.

\examplesep

\noindent\textbf{RAG Answer} \\
A Knickerbocker Glory is a type of ice cream sundae.
\end{tcolorbox}

\end{minipage}

\caption{
Qualitative examples where both branches produce correct answers, but BalanceRAG
selects the cheaper LLM-only route. These cases illustrate that always invoking RAG is
unnecessary when the LLM-only branch is already sufficiently confident.
}
\label{fig:qualitative_both_correct_choose_only}
\end{figure}

\end{document}